%% file: main.tex
\documentclass{article}
\input{preamble}
\hypersetup{
  pdftitle={Optimistic Actor-Critic with Parametric Policies for Linear Markov Decision Processes},
  pdfauthor={Max Qiushi Lin, Reza Asad, Kevin Tan, Haque Ishfaq, Csaba Szepesv\'ari, Sharan Vaswani}}

\title{Optimistic Actor-Critic with Parametric Policies\\
  for Linear Markov Decision Processes}
\author{%
  \begin{tabular}{c}
    Max Qiushi Lin\thanks{Simon Fraser University. \texttt{maxqslin@gmail.com}}\hspace{2.5em}
    Reza Asad\thanks{Simon Fraser University. \texttt{rasad@sfu.ca}}\hspace{2.5em}
    Kevin Tan\thanks{University of Pennsylvania. \texttt{kevtan@umich.edu}}\hspace{2.5em}
    Haque Ishfaq\thanks{Mila and McGill University. \texttt{haque.ishfaq@mail.mcgill.ca}} \\[0.5em]
    Csaba Szepesv{\'a}ri\thanks{Google DeepMind and University of Alberta. \texttt{szepi@google.com}}\hspace{2.5em}
    Sharan Vaswani\thanks{Simon Fraser University. \texttt{vaswani.sharan@gmail.com}}
  \end{tabular}
}
\date{}

\begin{document}

\maketitle

\begin{abstract}
Although actor-critic methods have been successful in practice, their theoretical analyses have several limitations. Specifically, existing theoretical work either sidesteps the exploration problem by making strong assumptions or analyzes impractical methods with complicated algorithmic modifications. Moreover, the actor-critic methods analyzed for linear MDPs often employ natural policy gradient and define ``implicit'' policies through the accumulated critic estimates, rather than learning an explicitly parameterized actor. Such policies can be computationally expensive to sample from, making the environment interactions inefficient. To that end, we focus on the finite-horizon linear MDPs and propose an optimistic actor-critic framework that uses parametric log-linear policies. In particular, we introduce a tractable $\textit{logit-matching}$ regression objective for the actor. For the critic, we use approximate Thompson sampling via Langevin Monte Carlo to obtain optimistic value estimates. We prove that, when the policy features are sufficiently expressive, the resulting algorithm achieves $\widetilde{\mathcal{O}}(\epsilon^{-4})$ and $\widetilde{\mathcal{O}}(\epsilon^{-2})$ sample complexity in the on-policy and off-policy settings, respectively. Our results match the state-of-the-art dependence on $\epsilon$ in prior theoretical work, while our algorithm is more aligned with practice.
\end{abstract}

\section{Introduction}
Reinforcement learning (RL) is a general framework for sequential decision-making under uncertainty and has been successful in various real-world applications, such as robotics~\citep{kober2013reinforcement} and large language models~\citep{uc2023survey}. Policy Gradient (PG) methods~\citep{williams1992simple,sutton1999policy,kakade2001natural,schulman2017equivalence} are an important class of algorithms that assume a differentiable parameterization of the policy, and directly optimize the policy parameters using the return from interacting with the environment. PG methods are widely used in practice as they can easily handle function approximation or structured state-action spaces. However, since the environment is typically stochastic in practice, the estimated returns usually have high variance, resulting in poor sample efficiency~\citep{dulac2019challenges}.

Actor-critic (AC) methods~\citep{konda1999actor,peters2005natural,bhatnagar2009natural} alleviate this issue by using value-based approaches in conjunction with PG methods. In particular, they utilize a critic that estimates the policy's value and an actor that performs PG to improve the policy towards obtaining higher returns. These AC methods have been proven to be empirically successful in both on-policy~\citep{schulman2015trust,schulman2017proximal} and off-policy~\citep{lillicrap2015continuous,fujimoto2018addressing,haarnoja2018soft} settings.

Subsequently, there have been many attempts to provide a theoretical understanding of actor-critic methods, especially in the presence of function approximation~\citep{cai2020provably,zhong2023theoretical,liu2023optimistic}. However, there are two prevalent issues that result in mismatches between theory and practice: the studied methods either (i) do not consider strategic exploration in a systematic manner or (ii) analyze complicated and impractical variants of the algorithm. In particular, much of the literature makes unrealistic assumptions to avoid dealing with exploration, a central challenge in RL. For instance, existing works on PG methods~\citep{agarwal2021theory,yuan2023linear,alfanolinear,asad2025fast} obtain convergence rates that involve a mismatch ratio between the optimal policy and the initial state distribution. These results are only meaningful if the mismatch ratio is bounded. However, a bounded mismatch ratio indicates that the initial state distribution already provides a good coverage over the state space, thereby sidestepping the exploration problem. Within actor-critic methods, some early analyses make assumptions on the reachability of the state-action space or the coverage of collected data~\citep{abbasi2019politex,neu2017unified,bhandari2024global,
agarwal2021theory,cen2022fast,gaur2024closing}, which again imply that the state-action space is already relatively easy to explore. Follow-up works~\citep{hong2023two,fu2021singletimescale,xu2020improving,cayci2024finitetime} assume a bounded mismatch ratio, while others~\citep{khodadadian2022finite,gaur2023global} require mixing assumptions on the induced Markov chain.

\begin{table}[tbh]
\centering
\caption{Comparison with the state-of-the-art algorithms for episodic finite-horizon linear MDPs (with feature dimension $d_\phi$ and horizon $H$). The sample complexity refers to the number of trajectories required for outputting an $\epsilon$-optimal policy, whereas the cost of policy inference refers to the per-episode cost for interacting with the environment for $H$ steps. For~\citet{sherman2023rate,cassel2024warm}, this cost follows from the closed-form representations of their policies, which aggregate the $Q$-function estimates through a small number of rarely updated covariance matrices. Our proposed algorithm matches the best-known dependence on $\epsilon$ in both the on- and off-policy settings. Furthermore, in contrast to existing works, which define policies through the accumulated critic estimates, our algorithm learns an explicitly parameterized log-linear actor of dimension $d_x$. $^\dagger$Under a sufficiently small approximation error of the policy features. $^\ddagger$With data reuse across $\lceil H/\epsilon \rceil$ policy updates.}
\label{tab:comparison}
\resizebox{\linewidth}{!}{
\begin{tabular}{lccccc}
\toprule[1.2pt]
\textbf{Algorithm} &
  \textbf{\begin{tabular}[c]{@{}c@{}}Sample Complexity\\ (On-Policy)\end{tabular}} &
  \textbf{\begin{tabular}[c]{@{}c@{}}Sample Complexity\\ (Off-Policy)\end{tabular}} &
  \textbf{\begin{tabular}[c]{@{}c@{}}Policy\\ Param.\end{tabular}} &
  \textbf{\begin{tabular}[c]{@{}c@{}}Clipping\\ Q-Function\end{tabular}} &
  \textbf{\begin{tabular}[c]{@{}c@{}}Computational Cost\\ for Policy Inference\end{tabular}} \\ \midrule \addlinespace
\citet{liu2023optimistic} &
  $\widetilde{\mathcal{O}}(\epsilon^{-4})$ &
  $\widetilde{\mathcal{O}}(\epsilon^{-3})^\ddagger$ &
  implicit &
  yes &
  $\mathcal{O}( d_\phi^2 H \lvert \mathcal{A} \rvert \epsilon^{-2})$ \\ \addlinespace
\citet{sherman2023rate} &
  $\xmark$ &
  $\widetilde{\mathcal{O}}(\epsilon^{-2})$ &
  implicit &
  no &
  $\widetilde{\mathcal{O}}( d_\phi^3 H \lvert \mathcal{A} \rvert)$ \\ \addlinespace
\citet{cassel2024warm} &
  $\xmark$ &
  $\widetilde{\mathcal{O}}(\epsilon^{-2})$ &
  implicit &
  no &
  $\mathcal{O}( d_\phi^2 H \lvert \mathcal{A} \rvert)$ \\ \addlinespace
Ours &
  $\widetilde{\mathcal{O}}(\epsilon^{-4})^\dagger$ &
  $\widetilde{\mathcal{O}}(\epsilon^{-2})^\dagger$ &
  explicit &
  yes &
  $\mathcal{O}( d_x H \lvert \mathcal{A} \rvert)$ \\ \addlinespace \bottomrule[1.2pt]
\end{tabular}
}
\end{table}

On the other hand, recent works~\citep{cai2020provably,jin2021bellman,zanette2021cautiously,zhong2023theoretical,agarwal2023vo,liu2023optimistic,sherman2023rate,cassel2024warm,tan2025actor} tackle the exploration issue directly. However, the algorithms analyzed are significantly different from those implemented in practice. Much of this body of work studies AC methods that use the natural policy gradient ($\NPG$) update for policy optimization. However, the canonical implementation of the $\NPG$ update does not consider an explicit policy parameterization. Instead, the update involves constructing ``implicit'' policies \textit{on the fly} using all previously stored $Q$-functions. This makes it computationally expensive to sample from these policies and use them to interact with the environment. In contrast, the algorithms used in practice typically employ explicitly parameterized complex models as learnable policies and optimize them with gradient descent-based methods. Therefore, we ask the following question:

\begin{center}
\textit{Can we design a provably efficient actor-critic algorithm with parametric policies for linear MDPs?}
\end{center}

We answer the above question affirmatively in both the on- and off-policy settings, and make the following contributions.

\paragraph{Contribution 1}
In~\cref{sec:optimistic_actor_critic}, we propose a general optimistic actor-critic framework that utilizes an explicitly parameterized policy. We apply this framework in the setting of linear function approximation for both the environment (i.e., linear MDP~\citep{jin2020provably}) and the policy (i.e., log-linear policy class). In~\cref{sec:actor_npg}, we propose an actor algorithm, Projected $\NPG$ ($\PNPG$), that learns a log-linear policy by solving a specific regression problem at each iteration. This allows us to directly control the error between the explicitly parametrized policy and the implicit policy induced by $\NPG$. Using this error bound in conjunction with the well-established theoretical results of $\NPG$~\citep{hazan2016introduction,szepesvari2022algorithms} enables us to analyze the performance of the parameterized actor. The proposed algorithm substantially improves the memory complexity while retaining similar theoretical guarantees.

\paragraph{Contribution 2}
In~\cref{sec:critic_lmc}, instead of constructing $\UCB$ bonuses~\citep{jin2020provably}, which are ubiquitous within prior works~\citep{zhong2023theoretical,liu2023optimistic,sherman2023rate,cassel2024warm}, we adopt a more practical approach. We employ Langevin Monte Carlo ($\LMC$)~\citep{welling2011bayesian} to update the critic parameters at each episode. Unlike UCB-based approaches that require computing confidence sets at every episode, $\LMC$ simply perturbs (by Gaussian noise) the gradient descent update on the critic loss. This gradient descent-based approach is both easier to implement~\citep{ishfaq2025langevin} and to extend to general function approximation~\citep{ishfaq2024more}. Furthermore, the $\LMC$ algorithm directly leads to an optimistic estimate of the $Q$-function that has similar guarantees as $\UCB$ bonuses. Nevertheless, previous works have only successfully designed provably efficient algorithms for solving multi-armed bandits~\citep{mazumdar2020approximate}, contextual bandits~\citep{xu2022langevin}, and linear MDPs via value-based methods~\citep{ishfaq2024provable}. To the best of our knowledge, this is the first work to establish regret and sample complexity guarantees for policy optimization with an LMC-based optimistic critic in linear MDPs.

\paragraph{Contribution 3}
In~\cref{sec:sample_efficiency_analysis}, we analyze the proposed actor-critic framework in both the on-policy and off-policy settings without making any assumptions on the mismatch ratio or data coverage. In particular, when the policy features can approximate the estimated $Q$-functions sufficiently well (\cref{cor:projection_error}), we prove that our method requires $\widetilde{\mathcal{O}}(\epsilon^{-4})$ trajectories to learn an $\epsilon$-optimal policy in the on-policy setting. This matches the dependence on $\epsilon$ in~\citet{liu2023optimistic}, which uses an implicit $\NPG$ policy in conjunction with $\UCB$ bonuses. On the other hand, under the same condition, we also prove that our framework can attain a sample complexity of $\widetilde{\mathcal{O}}(\epsilon^{-2})$ in the off-policy setting. This matches the dependence on $\epsilon$ in~\citet{sherman2023rate,cassel2024warm}, but with a less complicated algorithm design.

Furthermore, in~\cref{sec:experiments}, we empirically evaluate our algorithm in linear MDPs and extend it to Atari games~\citep{mnih2013playing}, where it achieves comparable performance to \texttt{PPO}~\citep{schulman2017proximal} and \texttt{DQN}~\citep{mnih2015human}.

We thus demonstrate that our proposed algorithm is both sample-efficient and aligned with practice.

\section{Preliminaries}
\label{sec:preliminaries}
In this section, we introduce the finite-horizon linear MDP setting and the log-linear policy class.
\paragraph{Finite-Horizon Linear MDP} A finite-horizon MDP is a tuple $\mathcal{M} = (\mathcal{S}, \mathcal{A}, \mathbb{P}, r, H)$ where $\mathcal{S}$ denotes the state space, $\mathcal{A}$ is the action set, and $H \in \mathbb{Z}_+$ is the length of the horizon. $\mathbb{P}=\{\mathbb{P}_h\}_{h\in[H]}$ is a set of time-dependent transition kernels, and $r = \{r_h\}_{h \in [H]}$ denotes a sequence of reward functions. We assume that the state space $\mathcal{S}$ is a (possibly infinite) measurable space, whereas $\mathcal{A}$ is a finite set with cardinality $\lvert \mathcal{A} \rvert$. $\mathbb{P}_h(\cdot \mid s, a) \in \Delta(\mathcal{S})$ is the distribution over states when taking action $a \in \mathcal{A}$ in state $s \in \mathcal{S}$ at step $h \in [H]$, and $r_h(s, a) \in [0, 1]$ is the corresponding reward. Additionally, for any given function $V: \mathcal{S} \to \mathbb{R}$, we define that $[\mathbb{P}_h V_{h+1}] (s, a) \coloneq \mathbb{E}_{s^\prime \sim \mathbb{P}_h(\cdot \mid s, a)} V_{h+1}(s^\prime)$.

The agent interacts with the environment by starting at an initial state (w.l.o.g., fixed to be $s_1 \in \mathcal{S}$). At step $h$, the agent first observes the current state $s_h \in \mathcal{S}$, then takes an action $a_h \in \mathcal{A}$ and receives the reward $r_h(s_h, a_h)$. After that, the agent transitions to $s_{h+1} \sim \mathbb{P}_h(\cdot \mid s_h, a_h) $. The agent follows a given policy $\pi: [H] \times \mathcal{S} \mapsto \Delta(\mathcal{A})$ in which $\pi_h (\cdot \mid s) \in \Delta (\mathcal{A})$ is the probability distribution over $\mathcal{A}$ in state $s$ at step $h$. To quantify the performance of any policy $\pi$, we define the value function as $V^\pi_h(s) \coloneq \mathbb{E}_{\pi, \mathbb{P}} [\sum_{\tau=h}^H r_{\tau}(s_{\tau}, a_{\tau}) \mid s_h = s]$, and the corresponding state-action value function is defined as $Q^\pi_h(s, a) \coloneq \mathbb{E}_{\pi, \mathbb{P}} [\sum_{\tau=h}^H r_{\tau}(s_{\tau}, a_{\tau}) \mid s_h = s, a_h = a]$, where the expectation is with respect to the randomness in the stochastic policy and the transition dynamics. The value function (resp. $Q$-function) corresponds to the expected cumulative rewards when starting in state $s$ (resp. state-action $(s, a)$) at step $h$, and subsequently following the policy $\pi$ until reaching step $H$.

We assume that both $\mathbb{P}$ and $r$ are unknown to the agent. In order to efficiently learn these quantities, we consider the linear MDP assumption~\citep{jin2020provably} where both the transition kernel and the reward function are assumed to be linear functions of given features.

\begin{definition}[Linear MDP]
    \label{def:linear_mdp}
    A finite-horizon MDP $\mathcal{M} = (\mathcal{S}, \mathcal{A}, \mathbb{P}, r, H)$ is a linear MDP with a feature map $\bm{\phi}: \mathcal{S} \times \mathcal{A} \mapsto \mathbb{R}^{d_\phi}$ if the following holds. There exist $H$ signed measures $\bm{\psi}_h: \mathcal{S} \to \mathbb{R}^{d_\phi}$ and $\bm{\upsilon}_h \in \mathbb{R}^{d_\phi}$ such that $\mathbb{P}_h(s^\prime\mid s,a) = \bm{\phi}(s,a)^\top \bm{\psi}_h(s^\prime)$ and $r_h(s,a) = \bm{\phi}(s,a)^\top \bm{\upsilon}_h$ for all $h$, $s$, and $a$. It should also satisfy the following constraints: $\left\| \bm{\phi}(s, a) \right\|_2 \leq 1$ and $\left\| \bm{\upsilon}_h \right\|_2 \leq \sqrt{d_\phi}$ for all $h$, $s$, and $a$. Additionally, for any measurable function $V: \mathcal{S} \to [0, 1]$, $\left\| \int_{s\in\mathcal{S}} V(s) \bm{\psi}_h(s) \, \mathrm{d} s \right\|_2 \leq \sqrt{d_\phi}$.
\end{definition}

According to~\citet[Proposition 2.3]{jin2020provably}, for a linear MDP and any policy $\pi$, $Q_h^\pi$ is a linear function of the features:  for all $(h, s, a)$, there exists a $\bm{w}_h \in \mathbb{R}^{d_\phi}$ such that $Q_h^\pi(s, a) = \bm{\phi}(s, a)^\top \bm{w}_h$.

\paragraph{Learning Objective} For this linear MDP setting, we assume that only $\bm{\phi}$ is given to the learner whereas $\bm{\psi}$ and $\bm{\upsilon}$ are not. The agent sequentially interacts with the environment for $K$ episodes (in the on-policy setting, each episode consists of $N$ trajectories; see~\cref{sec:optimistic_actor_critic}) and aims to minimize the \textit{cumulative regret} defined as $\mathrm{Reg}_K \coloneq \sum_{k = 1}^{K} \left[ V_1^\star(s_1) - V_1^{\pi_k}(s_1) \right]$, where $V_1^\star \coloneq V_1^{\pi^\star}\coloneq \sup_{\pi} V_1^{\pi}$ is the value function of the optimal policy $\pi^\star \coloneq \arg\sup_{\pi} V_1^{\pi}(s_1)$. Equivalently, if $\overline{\pi}_{K}$ denotes the mixture policy that picks a policy among $\{\pi_{1}, \ldots,\pi_{K}\}$ uniformly randomly, we aim to learn an $\epsilon$-optimal $\overline{\pi}_{K}$, i.e., its \textit{suboptimality gap} is bounded such that
\begin{align}
    \SubOpt(\overline{\pi}_K) \coloneq \mathbb{E} \left[ V_1^\star(s_1) - V_1^{\overline{\pi}_{K}}(s_1) \right] = \frac{\mathrm{Reg}_K}{K} \leq \widetilde{\mathcal{O}}(\epsilon),
\end{align}
where the expectation is taken with respect to the randomness of the mixture policy.

\paragraph{Log-Linear Policy} We consider a restricted policy class $\Pi_{\text{lin}}$ consisting of \textit{log-linear policies}. Log-linear policies are represented using the softmax function with linear function approximation. In particular, a log-linear policy is defined as follows: for all $h \in [H]$,
\begin{align}
    \pi_{h}(a \mid s, \bm{\theta}) = \frac{\exp(z_h(s, a \mid \bm{\theta}_h))}{\sum_{a^\prime \in \mathcal{A}} \exp(z_h(s, a^\prime \mid \bm{\theta}_h))}, \label{eq:log_linear_policy}
\end{align}
where $z_h(s, a \mid \bm{\theta}_h) = \bm{x}(s, a)^\top \bm{\theta}_h$ represents the logits parameterized by $\bm{\theta}_h$, and $\bm{x}: \mathcal{S} \times \mathcal{A} \mapsto \mathbb{R}^{d_x}$ are fixed policy features given to the learner. Without loss of generality, we assume that $\left\| \bm{x}(s, a) \right\| \leq 1$ for all $s$ and $a$. For convenience, we use the shorthand $\pi(\bm{\theta}): [H] \times \mathcal{S} \rightarrow \Delta(\mathcal{A})$ to refer to the log-linear policy corresponding to the parameters $\bm{\theta}$.

\section{A General Actor-Critic Framework with Parametric Policies}
\label{sec:optimistic_actor_critic}

In this section, we start by introducing our general optimistic actor-critic framework as shown in~\cref{alg:optimistic_actor_critic}.

\begin{algorithm2e}[!ht]
\caption{Actor-Critic with Parametric Policies}\label{alg:optimistic_actor_critic}
\begin{algorithmic}[1]
\State \textbf{Input}:
number of episodes $K$, data collection batch size $N$ (only for on-policy)
\State set $\mathcal{D}_{0} \gets \emptyset$, $\bm{w}_{0} \gets \bm{0}$, and $\bm{\theta}_{1} \gets \bm{0}$ (i.e., $\pi_{1,h}(\cdot \mid s) = \mathrm{Unif}(\mathcal{A})$ for all $(h,s)$)
\For{$k=1,\ldots, K$}
    \State $\mathcal{D}_{k} \gets
    \begin{cases}
        \text{On-Policy:}& \left\{ N \text{ fresh traj. }\stackrel{\text{i.i.d.}}{\sim} \pi_{k} \right\} \\
        \text{Off-Policy:}& \mathcal{D}_{k-1} \cup \left\{ 1 \text{ traj. } \sim \pi_{k} \right\} \\
    \end{cases}$\label{line:collect} \Comment{Collect Data}
    \State $(\widehat{Q}_{k}, \bm{w}_{k}) \gets \texttt{CRITIC\_UPDATE} (\mathcal{D}_{k}, \pi_{k}, \bm{w}_{k-1})$\label{line:critic} \Comment{Update Critic}
    \State $\bm{\theta}_{k+1} \gets \texttt{ACTOR\_UPDATE} (\widehat{Q}_{k}, \bm{\theta}_{k})$\label{line:actor} \Comment{Update Actor}
    \State $\pi_{k+1} = \pi(\bm{\theta}_{k+1})$\label{line:policy} \Comment{Instantiate Parametric Policy}
\EndFor
\State \textbf{Return}: mixture policy $\overline{\pi}_{K}$ that picks a policy among $\{\pi_1, \ldots, \pi_K\}$ uniformly at random
\end{algorithmic}
\end{algorithm2e}

Starting with a uniform policy $\pi_{1}$, at the beginning of every learning episode $k \in [K]$, the agent interacts with the environment using policy $\pi_{k}$ (\cref{line:collect}). Our framework allows for collecting data from the environment in either an on-policy or off-policy fashion. In the on-policy setting, at episode $k$, the agent collects $N$ fresh trajectories $\mathcal{D}_{k}$ by interacting with the environment using the current policy $\pi_{k}$. On the other hand, in the off-policy setting, at episode $k$, the agent collects only $1$ trajectory from the environment using $\pi_{k}$. However, the agent stores all the historical data collected by the previous policies, and hence, $\mathcal{D}_{k}$ consists of $k$ trajectories, each collected by $\pi_{1}, \ldots,\pi_{k}$ respectively.

The \textit{critic} uses the collected data and estimates an (optimistic) $Q$-function $\widehat{Q}_k = \{\widehat{Q}_{k,h}\}_{h \in [H]}$ for the current policy $\pi_k$ via learning the critic parameters $\bm{w}_{k}$ (\cref{line:critic}). The \textit{actor} then uses the estimated $Q$-function, and updates the parameters $\bm{\theta}_{k}$ of the log-linear policy to $\bm{\theta}_{k+1}$ (\cref{line:actor}). The updated log-linear policy is denoted by $\pi_{k+1}$ (\cref{line:policy}), and is used to collect data in the next episode.

Given this general framework, we will next instantiate the actor in~\cref{sec:actor_npg} and the critic in~\cref{sec:critic_lmc}.

\section{Instantiating the Actor: Projected Natural Policy Gradient}
\label{sec:actor_npg}

In this section, we instantiate the actor using natural policy gradient ($\NPG$) with parametric policies and analyze its behavior. In particular, in~\cref{subsec:projected_npg}, we devise an algorithm that projects the standard $\NPG$ update onto the class of realizable policies. In~\cref{subsec:bounding_projection_error}, we analyze and control the errors induced by the projection step. Finally, in~\cref{subsec:npg_llp}, we put everything together and instantiate the complete actor algorithm for the log-linear policy class.

\subsection{Projected Natural Policy Gradient}
\label{subsec:projected_npg}
At episode $k \in [K]$, given $\widehat{Q}_{k,h}$, the estimated $Q$-function, $\NPG$ updates the policy as: for every $(h,s) \in [H] \times \mathcal{S}$,
\begin{align}
    \pi_{k+1,h}(\cdot \mid s) &\propto \pi_{k,h}(\cdot \mid s) \exp (\eta \widehat{Q}_{k,h}(s, \cdot))
\label{eq:npg}
\end{align}
with the corresponding normalization across $\mathcal{A}$. Existing work on policy optimization in linear MDPs~\citep{liu2023optimistic,sherman2023rate,cassel2024warm} uses $\NPG$ to update the actor because of its favorable theoretical properties. Importantly, \textit{these results do not consider any explicit parameterization for the actor}. Directly implementing the update in~\cref{eq:npg} requires $\mathcal{O}(H \lvert \mathcal{S} \rvert \lvert \mathcal{A} \rvert)$ memory, and is therefore impractical with large state-action spaces. Consequently, existing work uses the following equivalent form of the $\NPG$ update: for every $(h,s) \in [H] \times \mathcal{S}$,
\begin{align}
    \textstyle
    \pi_{k+1,h}(\cdot \mid s) \propto \pi_{1,h}(\cdot \mid s) \exp \left( \eta \sum_{i=1}^k \widehat{Q}_{i,h}(s, \cdot) \right),
\end{align}
which characterizes the policy implicitly. In particular, at episode $k$, for any $(h, s, a)$, we can compute the policy \textit{on the fly} if we have access to the sum of all the parameterized $Q$-functions up to episode $k$. However, in~\citet{liu2023optimistic}, the $Q$-function estimates are clipped, so their sum cannot be stored in a succinct manner, and \textit{all} the parameterized $Q$-functions need to be stored. Consequently, interacting with the environment using such an implicit policy incurs a per-episode computational cost of $\mathcal{O}( d_\phi^2 H \lvert \mathcal{A} \rvert K)$. \citet{sherman2023rate,cassel2024warm} avoid clipping and update the covariance matrices that define their bonuses only rarely, so that the sum collapses into a small number of parameters. However, this requires dedicated algorithmic modifications (see~\cref{sec:sample_efficiency_analysis}), and evaluating the resulting policies still requires computing weighted norms of the features for every action. In all cases, the resulting algorithm is far from practice that typically uses an explicit (and often sophisticated) actor parameterization.

To alleviate these issues, we aim to compute a policy that is (i) realizable by the explicit actor parameterization and (ii) provably approximates the policy induced by the $\NPG$ update in~\cref{eq:npg} (referred to as the \textit{implicit policy}). To this end, we use the $\PNPG$ update, which performs the ordinary $\NPG$ update and then projects the result onto a given policy class $\Pi$: for every $(h,s) \in [H] \times \mathcal{S}$,
\begin{align}
    \pi_{k+1/2,h}(a \mid s) = \frac{\pi_{k,h}(a \mid s) \exp (\eta \widehat{Q}_{k,h}(s, a))}{\sum_{a^\prime \in \mathcal{A}} \pi_{k,h}(a^\prime \mid s) \exp (\eta \widehat{Q}_{k,h}(s, a^\prime))}, \qquad \pi_{k+1} = \Proj_{\Pi} (\pi_{k+1/2}), \label{eq:pnpg}
\end{align}
where $\Proj_\Pi$ is a projection operator, which will be instantiated for log-linear policies in~\cref{subsec:bounding_projection_error}.

When theoretically analyzing policy optimization methods, an important intermediate result is the bound on the regret for a specific online linear optimization problem. For the standard $\NPG$ update in~\cref{eq:npg}, this regret can be bounded by $\widetilde{\mathcal{O}}(\sqrt{K})$~\citep{hazan2016introduction,szepesvari2022algorithms}. In the following theorem, we consider this online linear optimization problem in a generic form and analyze the effect of an additional projection step on the regret.
\begin{theorem}
\label{lem:generalized_omd_regret}
Let $\Delta_n$ denote the probability simplex in $\mathbb{R}^n$, and let $\{g_k\}_{k=1}^K$ be a sequence of vectors in $\mathbb{R}^n$ with $\left\| g_k \right\|_\infty \leq L$. Let $p_1 \in \Delta_n$ be the uniform distribution, and for every $k \in [K]$, let
\begin{align}
    p_{k+1/2} &= \operatorname*{arg min}_{p \in \Delta_n} \left\{ \left\langle p, -\eta g_{k} \right\rangle + \KL(p \| p_{k}) \right\}, \label{eq:standard_npg} \\
    p_{k+1} &= \Proj_k (p_{k+1/2}), \label{eq:standard_projection}
\end{align}
where $\eta > 0$ is the step size, and $\Proj_k \colon \Delta_n \to \Delta_n$ is a projection operator whose outputs have full support. For every $u \in \Delta_n$, let $\varepsilon_k(u) \coloneq \KL(u \| p_{k+1}) - \KL(u \| p_{k+1/2})$ be the projection error induced by~\cref{eq:standard_projection}. Then, for every $u \in \Delta_n$, it holds that
\begin{align}
    \sum_{k=1}^K \left\langle u - p_{k}, g_{k} \right\rangle \leq \frac{\log n + \sum_{k=1}^K \varepsilon_{k}(u)}{\eta} + \frac{\eta L^2 K}{2}.
\end{align}
\end{theorem}
We now instantiate~\cref{lem:generalized_omd_regret} for $\PNPG$. Consider any $(h, s) \in [H] \times \mathcal{S}$, and let $n = \lvert \mathcal{A} \rvert$, $L = H$, $g_k = \widehat{Q}_{k,h}(s, \cdot)$, and $p_k = \pi_{k,h}(\cdot \mid s)$ for every $k$, where $\left\| g_k \right\|_\infty \leq H$ holds whenever $\widehat{Q}_{k,h}$ takes values in $[0, H]$. The $\NPG$ step in~\cref{eq:pnpg} is an online mirror descent step with the negative entropy regularizer~\citep{xiao2022convergence}, i.e., $\pi_{k+1/2,h}(\cdot \mid s)$ is exactly $p_{k+1/2}$ in~\cref{eq:standard_npg}. Moreover, $\Proj_k$ maps $p_{k+1/2}$ to $p_{k+1} = \pi_{k+1,h}(\cdot \mid s)$, as determined by $\Proj_\Pi$ in~\cref{eq:pnpg}. For log-linear policies, $p_{k+1}$ has full support. Consequently, the projection error in~\cref{lem:generalized_omd_regret} becomes
\begin{align}
    \varepsilon^{\proj}_{k,h}(u, s) \coloneq \KL \left( u \| \pi_{k+1,h}(\cdot \mid s) \right) - \KL \left( u \| \pi_{k+1/2,h}(\cdot \mid s) \right). \label{eq:projection_error}
\end{align}
Since~\cref{lem:generalized_omd_regret} holds for every sequence $\{g_k\}_{k=1}^K$, it also applies to the data-dependent $Q$-function estimates generated by~\cref{alg:optimistic_actor_critic}. We measure the projection error uniformly by
\begin{align}
    \varepsilon_{\proj} \coloneq \frac{H}{\eta} \max_{k \in [K], h \in [H]} \varepsilon_{k,h}, \qquad \text{where} \quad \varepsilon_{k,h} \coloneq \sup_{s \in \mathcal{S}} \sup_{u \in \Delta(\mathcal{A})} \varepsilon^{\proj}_{k,h}(u, s). \label{eq:eps_proj}
\end{align}
Taking $u = \pi_{k+1/2,h}(\cdot \mid s)$, for which the second divergence in~\cref{eq:projection_error} vanishes, shows that $\varepsilon_{k,h} \geq 0$, and hence $\varepsilon_{\proj} \geq 0$. Applying~\cref{lem:generalized_omd_regret} with the above instantiation for each $(h, s)$ and the comparator $u = \pi^\star_h(\cdot \mid s)$, bounding each $\varepsilon^{\proj}_{k,h}(u, s)$ by $\varepsilon_{k,h}$, taking the expectation over the states visited by $\pi^\star$, summing over $h \in [H]$, and setting $\eta = \sqrt{2 \log \lvert \mathcal{A} \rvert / (K H^2)}$ gives the following regret bound\footnote{This generalized regret bound holds for any other mirror descent-based policy optimization method (e.g., \texttt{SPMA}~\citep{asad2025fast} in~\cref{sec:spma_actor}), but we discuss $\NPG$ within the main text for the ease of exposition.} for $\PNPG$:
\begin{align}
    \sum_{k=1}^{K} \sum_{h = 1}^H \mathbb{E}_{\pi^\star} \left[ \left\langle \pi_h^\star(\cdot \mid s_h) - \pi_{k,h}(\cdot \mid s_h), \widehat{Q}_{k,h} (s_h, \cdot) \right\rangle \right] \leq H^2 \sqrt{2 K \log \lvert \mathcal{A} \rvert} + K \varepsilon_{\proj}. \label{eq:npg-regret-bound}
\end{align}
For the $\NPG$ update in~\cref{eq:npg} without projection, $\varepsilon_{\proj} = 0$ and the above result recovers the standard regret bound for $\NPG$. In general, $\varepsilon_{\proj}$ is the additional regret per episode caused by the projection. Hence, by choosing the projection operator carefully and controlling the projection errors, we can bound the regret.

\subsection{Controlling the Projection Error for Log-Linear Policies}
\label{subsec:bounding_projection_error}
To control the projection error, one could choose $\Proj_\Pi$ to directly minimize $\varepsilon_{k,h}$ in~\cref{eq:eps_proj} over $\Pi$. However, $\varepsilon_{k,h}$ involves a supremum over all the states, and it is unclear how to minimize it using only a finite subset of state-action pairs. Consequently, we instead define $\Proj_\Pi$ through a regression in the logit space. Let $\pi_{k} = \pi(\bm{\theta}_{k})$ be a log-linear policy. Substituting it into~\cref{eq:pnpg}, the normalizing constant of $\pi_{k,h}(\cdot \mid s)$ cancels, and $\pi_{k+1/2,h}(\cdot \mid s)$ is the softmax policy with logits
\begin{align}
    z_{k,h}(s, a) \coloneq \bm{x}(s, a)^\top \bm{\theta}_{k,h} + \eta \widehat{Q}_{k,h}(s, a). \label{eq:target_logits}
\end{align}
Since $\widehat{Q}_{k,h}$ need not lie in the span of the policy features, $\pi_{k+1/2}$ need not be log-linear. We therefore project it onto the log-linear policy class by fitting its logits. A direct fit would control the logit error at every state-action pair through the \textit{global actor loss}
\begin{align}
    \ell^{\act}_{k,h}(\bm{\theta}) \coloneq \sup_{(s, a) \in \mathcal{S} \times \mathcal{A}} \left( \bm{x}(s, a)^\top \bm{\theta} - z_{k,h}(s, a) \right)^2. \label{eq:global_actor_loss}
\end{align}
However, minimizing~\cref{eq:global_actor_loss} requires evaluating the target logits at every state-action pair, which is impractical with large state-action spaces. Therefore, we propose to construct a good and preferably small subset $\mathcal{C} \subseteq \mathcal{S} \times \mathcal{A}$, referred to as a \textit{coreset}, along with a corresponding distribution $\rho_{\mathcal{C}} \in \Delta(\mathcal{C})$ that offers good coverage of the feature space. Given $\mathcal{C}$ and $\rho_{\mathcal{C}}$ (the construction of which will be detailed subsequently), we instantiate the \textit{coreset actor loss} in terms of a \textit{logit-matching} regression:
\begin{align}
    \widetilde{\ell}^{\act}_{k,h} (\bm{\theta}) &\coloneq \sum_{(s,a) \in \mathcal{C}} \rho_{\mathcal{C}}(s,a) \left( \bm{x}(s, a)^\top \bm{\theta} - z_{k,h}(s, a) \right)^2 \\
    &= \sum_{(s,a) \in \mathcal{C}} \rho_{\mathcal{C}}(s,a) \left[ \bm{x}(s, a)^\top \left( \bm{\theta} - \bm{\theta}_{k,h} \right) - \eta \widehat{Q}_{k,h}(s, a) \right]^2. \label{eq:actor_loss}
\end{align}
As a regression problem, the coreset actor loss can be easily optimized via gradient descent-based methods.

In order to show that optimizing the above actor loss can indeed bound the projection error, we require the following assumptions. We assume that the policy features $\bm{x}$ are expressive enough to control the bias when minimizing $\ell^{\act}_{k,h}(\bm{\theta})$.
\begin{assumption}[PNPG Approximation Error]
    \label{asm:bias}
    For every possible history and every $k \in [K]$, $h \in [H]$, the global actor loss in~\cref{eq:global_actor_loss} satisfies $\min_{\bm{\theta} \in \mathbb{R}^{d_x}} \ell^{\act}_{k,h}(\bm{\theta}) \leq \varepsilon_{\bias}$.
\end{assumption}
In practice, $\varepsilon_{\bias}$ can be reduced by choosing high-dimensional features (e.g., $d_x \gg d_\phi$) or a sufficiently expressive policy class (e.g., neural networks).

Next, we assume the coreset actor loss $\widetilde{\ell}^{\act}_{k,h}(\bm{\theta})$ is sufficiently minimized.
\begin{assumption}[PNPG Optimization Error]
    \label{asm:opt_error}
    For every possible history and every $k \in [K]$, $h \in [H]$, the updated actor parameters $\bm{\theta}_{k+1,h}$ satisfy $\widetilde{\ell}^{\act}_{k,h}(\bm{\theta}_{k+1,h}) - \min_{\bm{\theta} \in \mathbb{R}^{d_x}} \widetilde{\ell}^{\act}_{k,h}(\bm{\theta}) \leq \varepsilon_{\opt}$.
\end{assumption}
In practice, minimizing $\widetilde{\ell}^{\act}_{k,h}(\bm{\theta})$ by $T_{k}$ steps of gradient descent ensures that $\varepsilon_\opt \leq \mathcal{O}(\exp(-T_{k} / \kappa_G))$, where $\kappa_G$ is the condition number of the coreset Gram matrix $\bm{G}$ in~\cref{lem:projection_error}.
Given these approximation and optimization assumptions, we can then proceed to bound the projection error $\varepsilon_{k,h}$ defined in~\cref{eq:eps_proj} as follows.
\begin{lemma}[PNPG Projection Error]
    \label{lem:projection_error}
    Let $\bm{G} \coloneq \sum_{(s, a) \in \mathcal{C}} \rho_{\mathcal{C}}(s, a) \bm{x}(s, a) \bm{x}(s, a)^\top$ be the Gram matrix of a fixed coreset with $\bm{G} \succ 0$, and let $\chi_G \coloneq \sup_{(s, a) \in \mathcal{S} \times \mathcal{A}} \left\| \bm{x}(s, a) \right\|_{\bm{G}^{-1}}$ measure how well the coreset covers the policy features. Under~\cref{asm:bias,asm:opt_error}, for every $k \in [K]$ and $h \in [H]$, it holds that
    \begin{align}
        \varepsilon_{k,h} \leq 2 (1 + \chi_G) \sqrt{\varepsilon_{\bias}} + 2 \chi_G \sqrt{\varepsilon_{\opt}}.
    \end{align}
\end{lemma}
The above lemma relies on the logits being linear in the parameters, which allows us to transfer the regression error on $\mathcal{C}$ to every state-action pair. It is true for any $\mathcal{C}$ and $\rho_{\mathcal{C}}$, and suggests that if we can control $\chi_G$, the projection error can be bounded.
\paragraph{Constructing the Coreset} Therefore, we would like to construct a suitable $\mathcal{C}$ and $\rho_{\mathcal{C}}$ to bound $\left\| \bm{x}(s, a) \right\|_{\bm{G}^{-1}}$ for any $(s, a) \in \mathcal{S} \times \mathcal{A}$ and solve the following optimization problem:
\begin{align}
    \MoveEqLeft
    \textstyle \inf_{ \substack{\mathcal{C} \subseteq \mathcal{S} \times \mathcal{A} \\ \rho_{\mathcal{C}} \in \Delta(\mathcal{C}) } } \sup_{(s, a) \in \mathcal{S} \times \mathcal{A}} \left\| \bm{x}(s, a) \right\|_{\bm{G}^{-1}} \\
    &\textstyle \text{s.t.} \quad \bm{G} = \sum_{(s, a) \in \mathcal{C}} \rho_{\mathcal{C}}(s, a) \bm{x}(s, a) \bm{x}(s, a)^\top,
\end{align}
which fits the form of experimental design. Ideally, we would also like $\left\lvert \mathcal{C} \right\rvert$ to be relatively small so that the actor parameters can be updated efficiently.

There are standard techniques to solve this problem. The most common approach is the $G$-optimal design, which involves constructing a coreset (i.e., $\mathcal{C}$ and $\rho_{\mathcal{C}}$) and bounds $\left\| \bm{x}(s, a) \right\|_{\bm{G}^{-1}}$. In particular, if the set of policy features is compact and spans $\mathbb{R}^{d_x}$, the Kiefer–Wolfowitz theorem~\citep{kiefer1960equivalence} guarantees that there exists a coreset such that $\chi_G = \sqrt{d_x}$ and $\left\lvert \mathcal{C} \right\rvert \leq d_x (d_x + 1) / 2$, whose matrix $\bm{G}$ is invertible, i.e., $\lambda_G \coloneq \lambda_{\min}(\bm{G}) > 0$. In practice, such a coreset can be approximately constructed by the Frank–Wolfe algorithm~\citep{frank1956algorithm,todd2016minimum}, which returns a coreset with $\chi_G \leq \sqrt{2 d_x}$ and $\left\lvert \mathcal{C} \right\rvert = \mathcal{O}(d_x \log \log d_x)$~\citep[Chapter 21]{lattimore2020bandit}. We remark that this method only uses the given policy features $\bm{x}$, and does not involve the linear MDP features $\bm{\phi}$. Furthermore, the required coreset can be constructed offline, even before the learning procedure or without any knowledge of the environment (see~\cref{subsec:kiefer_wolfowitz} for details). Given access to such a coreset, \cref{lem:projection_error} guarantees that $\varepsilon_{k,h} \leq \mathcal{O}(\sqrt{d_x} (\sqrt{\varepsilon_{\bias}} + \sqrt{\varepsilon_{\opt}}))$, and the coreset actor loss only involves $\widetilde{\mathcal{O}}(d_x)$ state-action pairs.

Rather than forming a coreset, alternative approaches assume $\bm{x} = \bm{\phi}$, and use some limited interaction with the environment to construct $\mathcal{C}$. In particular, under some standard assumptions (e.g.,~\citealp[Assumption 1]{wagenmaker2022instance}), we can apply methods such as \texttt{CoverTraj}~\citep{wagenmaker2022reward} and \texttt{OptCov}~\citep{wagenmaker2022instance} that construct $\mathcal{C}$ and $\rho_{\mathcal{C}}$ and can subsequently ensure that $\chi_G$ is bounded. The sample complexity for such procedures is $\widetilde{\mathcal{O}}(d_\phi^4 H^3 \varepsilon_{\mathcal{M}}^{-1})$ where $\varepsilon_{\mathcal{M}}$ is a problem-dependent constant. We defer all the details to~\cref{sec:experimental_design}.

\begin{remark}
Having access to $\mathcal{C}$ and $\rho_{\mathcal{C}}$ does not obviate the necessity of exploration. Specifically, we do not assume random access, i.e., the agent cannot visit all the state-actions in $\mathcal{C}$. Hence, we cannot effectively calculate $Q^{\pi}(s, a)$ for any $(s, a) \in \mathcal{C}$.
\end{remark}

\subsection{Putting Everything Together}
\label{subsec:npg_llp}
In~\cref{alg:actor_npg}, we instantiate the complete actor algorithm, which uses the $\PNPG$ update for log-linear policies. Unlike the standard $\NPG$ update,~\cref{alg:actor_npg} alleviates the necessity of storing past $Q$-functions, reducing the memory of the current policy to $\mathcal{O}(H d_x)$, while enjoying similar theoretical guarantees. We note that the guarantees in~\cref{sec:sample_efficiency_analysis} are for the mixture policy $\overline{\pi}_K$, which requires storing all the actor parameters $\{\bm{\theta}_k\}_{k \in [K]}$. Alternatively, one can draw $k$ uniformly from $[K]$ before training and only keep $\bm{\theta}_k$, and the expected suboptimality of $\pi_k$ over this draw is $\SubOpt(\overline{\pi}_K)$. Furthermore, the actor parameters are updated by using gradient descent on a properly defined surrogate loss, rendering it closer to the practical implementation of common algorithms (e.g., \texttt{PPO}~\citep{schulman2017proximal}).

\begin{algorithm2e}[!ht]
\caption{Actor: Projected Natural Policy Gradient ($\PNPG$)}\label{alg:actor_npg}
\begin{algorithmic}[1]
\State \textbf{Input}: estimated $Q$-function $\widehat{Q}_{k}$, actor parameters $\bm{\theta}_{k}$, policy optimization learning rate $\eta$, number of actor updates $T_{k}$, actor learning rate $\eta^{\act}$, coreset $\mathcal{C}$ and distribution $\rho_{\mathcal{C}}$
\For{$h=1, 2, \ldots, H$}
    \State Define the coreset actor loss $\widetilde{\ell}^{\act}_{k,h} (\bm{\theta})$ using~\cref{eq:actor_loss}
    \State $\bm{\theta}_{k,h}^{(0)} \gets \bm{\theta}_{k,h}$
    \For{$t=1, \ldots, T_{k}$}
        \State $\bm{\theta}_{k,h}^{(t)} \gets \bm{\theta}_{k,h}^{(t-1)} - \eta^{\act} \nabla_{\bm{\theta}} \widetilde{\ell}^{\act}_{k,h}(\bm{\theta}_{k,h}^{(t-1)})$
    \EndFor
    \State $\bm{\theta}_{k+1,h} \gets \bm{\theta}_{k,h}^{(T_k)}$
\EndFor
\State \textbf{Return}: actor parameters $\bm{\theta}_{k+1} = \{\bm{\theta}_{k+1,h}\}_{h \in [H]}$
\end{algorithmic}
\end{algorithm2e}

We remark that although we focused on log-linear policies,~\cref{alg:optimistic_actor_critic} can also be used with more general policy classes, such as neural networks. Extending our guarantees to such classes requires analogous bounds on the projection error (\cref{lem:projection_error}) and, in the off-policy setting, on the covering number of the policy class (\cref{lem:covering_number_of_v}). We evaluate such an extension empirically on Atari games in~\cref{subsec:atari_experiments}. In the next section, we instantiate the critic in~\cref{alg:optimistic_actor_critic}.

\section{Instantiating the Critic: Langevin Monte Carlo}
\label{sec:critic_lmc}
In this section, we use Langevin Monte Carlo ($\LMC$) to instantiate the critic. We describe the resulting algorithm in~\cref{subsec:lmc_for_linear_mdps}, and analyze it in~\cref{subsec:optimism_guarantee_and_error_bound}.

The $\LMC$ approaches allow for sampling from a posterior distribution and have recently been used in sequential decision-making problems. For example,~\citet{mazumdar2020approximate} achieves optimal instance-dependent regret bounds for multi-armed bandits using Langevin dynamics for approximate Thompson sampling. On the other hand,~\citet{xu2022langevin} uses $\LMC$ for contextual bandits, achieving comparable theoretical results to Thompson sampling. More recently,~\citet{ishfaq2024provable} leverages $\LMC$ for linear MDPs by using it to sample the $Q$-function from its posterior distribution, achieving the optimal $\widetilde{\mathcal{O}}(\sqrt{K})$ regret.

Nevertheless, existing regret and sample complexity analyses of LMC-based approaches for MDPs, including those for general function approximation~\citep{ishfaq2024more,jorge2024isoperimetry}, focus on value-based algorithms. To the best of our knowledge, no such guarantees exist for policy optimization. Next, we incorporate the $\LMC$ algorithm into our actor-critic framework and provide the first such guarantee.

\subsection{LMC for Linear MDPs}
\label{subsec:lmc_for_linear_mdps}
At episode $k$, the critic (\cref{alg:lmc}) uses the collected dataset $\mathcal{D}_{k}$ to obtain an optimistic estimate of the $Q$-function of the current policy $\pi_k$. We denote the trajectories in $\mathcal{D}_{k}$ by $\tau_i = (s_{i,1}, a_{i,1}, \ldots, s_{i,H}, a_{i,H}, s_{i,H+1})$, write $\tau_i \in \mathcal{D}_k$, and use the shorthand $\bm{\phi}_{i,h} \coloneq \bm{\phi}(s_{i,h}, a_{i,h})$. We denote by $\lvert \mathcal{D}_{k} \rvert$ the number of trajectories in $\mathcal{D}_{k}$. The critic loss at episode $k$ and step $h$ uses the estimated value function at step $h+1$, and forms the following ridge regression problem:
\begin{align}
    \MoveEqLeft
    \ell^{\cri}_{k,h} (\bm{w}) \coloneq \frac{1}{2} \sum_{\tau_i \in \mathcal{D}_k} \left[ r_h(s_{i,h}, a_{i,h}) + \widehat{V}_{k,h+1} (s_{i,h+1}) - \bm{\phi}_{i,h}^\top \bm{w} \right]^2 + \frac{\lambda }{2} \left\| \bm{w} \right\|^2. \label{eq:critic_loss}
\end{align}

\begin{algorithm2e}[!ht]
\caption{Critic: $\LMC$}\label{alg:lmc}
\begin{algorithmic}[1]
\State \textbf{Input}: collected data $\mathcal{D}_{k}$, policy $\pi_{k}$, previous critic parameters $\bm{w}_{k-1}$, number of critic updates $J_k$, critic learning rate $\eta^{\cri}_{k,h}$, inverse temperature $\beta$, number of critic samples $M$
\State $\widehat{V}_{k,H+1} (s) \gets 0$, for every queried $s$
\For{$h=H, H-1, \ldots, 1$}
    \State Define the critic loss $\ell^{\cri}_{k,h} (\bm{w})$ using~\cref{eq:critic_loss}
    \For{$m=1, \ldots, M$}
        \State $\bm{w}_{k,h}^{(m,0)} \gets \bm{w}_{k-1,h}^{(m)}$
        \For{$j=1, \ldots, J_k$}\label{line:lmc_start}
            \State $\bm{\nu}_{k,h}^{(m,j)} \sim \mathcal{N}(\bm{0}, I)$
            \State $\bm{w}_{k,h}^{(m,j)} \gets \bm{w}_{k,h}^{(m,j-1)} - \eta^{\cri}_{k,h} \nabla_{\bm{w}} \ell^{\cri}_{k,h}( \bm{w}_{k,h}^{(m,j-1)}) + \sqrt{\eta^{\cri}_{k,h} / \beta} \bm{\nu}_{k,h}^{(m,j)}$\label{line:lmc_end}
        \EndFor
        \State $\bm{w}_{k,h}^{(m)} \gets \bm{w}_{k,h}^{(m,J_k)}$
    \EndFor
    \State $\displaystyle \widehat{Q}_{k,h} (s, a) \gets \Clip_{[0, H-h+1]} \left( \max_{m \in [M]} \bm{\phi}(s, a)^\top \bm{w}_{k,h}^{(m)} \right)$, for every queried $(s, a)$\label{line:optimistic_q}
    \State $\widehat{V}_{k,h} (s) \gets \mathbb{E}_{a \sim \pi_{k,h}(\cdot \mid s)} \widehat{Q}_{k,h} (s, a)$, for every queried $s$
\EndFor
\State \textbf{Return}: estimated $Q$-function $\widehat{Q}_k = \{\widehat{Q}_{k,h}\}_{h \in [H]}$ and critic parameters $\bm{w}_k = \{\bm{w}_{k,h}^{(m)}\}_{(m, h) \in [M] \times [H]}$
\end{algorithmic}
\end{algorithm2e}

For each step $h$, the $\LMC$ algorithm iteratively adds Gaussian noise to the gradient descent updates on $\ell^{\cri}_{k,h} (\bm{w})$, and aims to produce approximate samples of the critic parameters from its underlying posterior distribution (Lines~\ref{line:lmc_start}--\ref{line:lmc_end}). In particular, for an arbitrary loss $\ell$, the $j$-th $\LMC$ update can be written as:
\begin{align}
    \bm{w}^{(j)} = \bm{w}^{(j-1)} - \eta^{\cri} \nabla_{\bm{w}} \ell (\bm{w}^{(j-1)}) + \sqrt{\eta^{\cri} / \beta} \bm{\nu}^{(j)},
\end{align}
where $\eta^{\cri}$ is the learning rate, $\beta$ is the inverse temperature parameter, and $\bm{\nu}^{(j)} \sim \mathcal{N}(\bm{0}, I)$ is an isotropic Gaussian noise vector. After $J_k$ steps of the $\LMC$ update on the critic loss (Lines~\ref{line:lmc_start}--\ref{line:lmc_end} in~\cref{alg:lmc}), the resulting critic parameters are used to produce an optimistic sample of the $Q$-function (\cref{line:optimistic_q}). From a theoretical perspective, we note that it is important to clip $\widehat{Q}_{k,h}$ appropriately.\footnote{$\Clip_{[a, b]}(y) \coloneq \max\{\min\{y, b\}, a\}$ is the clipping function that bounds a scalar $y$ to the interval $[a, b]$.} In order to improve the optimism guarantees of the $\LMC$ algorithm, we follow the idea in~\citet{ishfaq2021randomized}, and repeat the $\LMC$ update $M$ times, taking the maximum over these samples (\cref{line:optimistic_q}). Iterating this procedure backwards from $h = H$ to $1$, we can obtain the desired critic parameters. We note that the critic is represented by its parameters, and $\widehat{Q}_{k,h}$ and $\widehat{V}_{k,h}$ are never stored as tables over $\mathcal{S} \times \mathcal{A}$ or $\mathcal{S}$. Instead, they are only evaluated at the queried state-action pairs and states, e.g., at the next states $s_{i,h+1}$ in $\mathcal{D}_k$ when forming the critic loss, and at the coreset $\mathcal{C}$ when forming the actor loss.

Note that compared to UCB-based approaches, $\LMC$ does not require computing confidence sets at every episode. Instead, it simply perturbs gradient descent by injecting Gaussian noise, allowing for a natural extension beyond the linear function approximation setting and rendering it easier to implement in practice.

\subsection{Optimism Guarantee and Error Bound}
\label{subsec:optimism_guarantee_and_error_bound}
In order to theoretically analyze~\cref{alg:lmc}, we first define the following Bellman error.
\begin{definition}[Bellman Error]
    \label{def:bellman_error}
    Given an estimated $Q$-function $\widehat{Q}_{k}$ and the corresponding estimated value function $\widehat{V}_{k}$, for all $(k,h, s, a)$, the Bellman error is defined as: $\xi_{k,h}(s, a) \coloneq r_h(s, a) + \mathbb{P}_h \widehat{V}_{k,h+1} (s, a) - \widehat{Q}_{k,h} (s, a)$.
\end{definition}
The theoretical analyses in existing work~\citep{jin2020provably,zhong2023theoretical,liu2023optimistic} that use $\UCB$ bonuses typically proceed by proving an upper bound of $0$ on $\xi_{k,h}$ (\textit{optimism}) and a lower bound that is proportional to the bonus $\left\| \bm{\phi}(s, a) \right\|_{\bm{\Lambda}_{k,h}^{-1}}$ (\textit{error bound}). The following lemma shows that $\LMC$ can offer similar guarantees.
\begin{lemma}
    \label{lem:informal_lmc_optimism_and_error_bound}
    Let $\bm{\Lambda}_{k,h} \coloneq \sum_{\tau_i \in \mathcal{D}_k} \bm{\phi}_{i,h} \bm{\phi}_{i,h}^\top + \lambda I$. With appropriate choices of $\lambda$, $\beta$, $J_k$, $\eta^{\cri}_{k,h}$, $M$ and for any $\delta \in (0, 1)$,~\cref{alg:optimistic_actor_critic} with the $\LMC$ critic in~\cref{alg:lmc} ensures that in both the on-policy and off-policy settings, for some constant $\Gamma_{\LMC} > 0$, with probability at least $1 - \delta$, for all $(k, h, s, a) \in [K] \times [H] \times \mathcal{S} \times \mathcal{A}$,
    \begin{align}
        - \Gamma_{\LMC} \left\| \bm{\phi}(s, a) \right\|_{\bm{\Lambda}_{k,h}^{-1}} \leq \xi_{k,h}(s, a) \leq 0.
    \end{align}
\end{lemma}
The upper bound is the optimism guarantee, and the lower bound is the error bound. The exact definition of $\Gamma_{\LMC}$ is given in~\cref{lem:lmc_optimism_and_error_bound}, the full version of this lemma, and it varies between the two settings. With $C_\delta$ chosen as in~\cref{lem:on_policy_self_normalized_concentration,lem:off_policy_self_normalized_concentration}, respectively, $\Gamma_{\LMC} = \widetilde{\mathcal{O}} (H d_\phi^{3/2})$ in the on-policy setting and $\Gamma_{\LMC} = \widetilde{\mathcal{O}} (H \sqrt{d_\phi^3 + d_\phi d_x})$ in the off-policy setting. For a fixed state-action pair, a single $\LMC$ sample is optimistic only with a constant probability. In order to ensure optimism at all the state-action pairs simultaneously, we take the maximum over $M = \widetilde{\mathcal{O}}(d_\phi)$ samples and use a covering argument over the directions in $\mathbb{R}^{d_\phi}$. Since the number of samples grows with $d_\phi$, $\Gamma_{\LMC}$ has a worse dependence on $d_\phi$ than the corresponding constant of $\UCB$ bonuses. In fact, the critic in~\cref{alg:optimistic_actor_critic} can also be instantiated with $\UCB$ bonuses, which offer the same guarantee as~\cref{lem:informal_lmc_optimism_and_error_bound} with $\Gamma_{\UCB} = \widetilde{\mathcal{O}}(H d_\phi)$ in the on-policy setting and $\Gamma_{\UCB} = \widetilde{\mathcal{O}}(H \sqrt{d_\phi^2 + d_x})$ in the off-policy setting (see~\cref{lem:ucb_optimism_and_error_bound}).

The error bound relies on a uniform concentration argument for the estimated value functions. In the on-policy setting, conditioned on the past, all the trajectories in $\mathcal{D}_{k}$ are collected independently using the same policy $\pi_{k}$. However, $\widehat{V}_{k,h+1}$ is computed from the same trajectories. Hence, we use the value-aware uniform concentration result from~\citet{jin2020provably} over a class of $Q$-functions, where the policy $\pi_k$ is fixed. In the off-policy setting, the trajectories in $\mathcal{D}_{k}$ are collected by different data-dependent policies, and the uniform concentration argument must also hold uniformly over the policies. This requires control over the covering number of the value function class including the policy class, which is deferred to~\cref{sec:sample_efficiency_analysis}. This additional covering is the reason why $\Gamma_{\LMC}$ depends on $d_x$ in the off-policy setting.

Therefore, we conclude that, compared to $\UCB$ bonuses, $\LMC$ offers significant practical advantages while still providing similar theoretical guarantees.

\section{Sample Complexity Analysis}
\label{sec:sample_efficiency_analysis}
In this section, we analyze the sample complexity of~\cref{alg:optimistic_actor_critic} with the $\PNPG$ actor from~\cref{alg:actor_npg} and the $\LMC$ critic from~\cref{alg:lmc}.~\cref{subsec:analysis_for_on_policy} focuses on the on-policy setting, while~\cref{subsec:analysis_for_off_policy} addresses the off-policy setting.

\subsection{On-Policy Setting}
\label{subsec:analysis_for_on_policy}
We now present the following theorem that shows that our proposed algorithm achieves a sample complexity of $\widetilde{\mathcal{O}}(\epsilon^{-4})$ in the on-policy setting.

\begin{theorem}
\label{thm:on_policy_sample_efficiency}
Under~\cref{asm:bias,asm:opt_error}, consider~\cref{alg:optimistic_actor_critic} in the on-policy setting with the $\PNPG$ actor (\cref{alg:actor_npg}) and the $\LMC$ critic (\cref{alg:lmc}). For an appropriate choice of the actor and critic parameters, including $\eta = \sqrt{2 \log \lvert \mathcal{A} \rvert / (K H^2)}$ and $N = \lceil d_\phi^4 K / \log \lvert \mathcal{A} \rvert \rceil$, and any $\delta \in (0,1)$, it holds that with probability at least $1 - \delta$,
\begin{align}
    \SubOpt(\overline{\pi}_K) \leq \widetilde{\mathcal{O}} \left( \frac{H^2}{\sqrt{K}} + \varepsilon_{\proj} \right),
\end{align}
where $\varepsilon_{\proj}$ is the projection error defined in~\cref{eq:eps_proj} and controlled by~\cref{lem:projection_error}. Hence, for any $\epsilon > 0$,~\cref{alg:optimistic_actor_critic} with $K = H^4 \log \lvert \mathcal{A} \rvert \epsilon^{-2}$ yields an $\widetilde{\mathcal{O}}(\epsilon + \varepsilon_{\proj})$-optimal mixture policy, and therefore requires $K \times N = \widetilde{\mathcal{O}}(d_\phi^4 H^8 \epsilon^{-4})$ trajectories.
\end{theorem}

\begin{proof}[Proof Sketch]
We first decompose the cumulative regret into three terms.

\paragraph{Regret Decomposition} In~\cref{lem:regret_decomposition}, we apply the extended value difference lemma (\cref{lem:extended_value_diff}) and obtain the following decomposition of the cumulative regret:
\begin{align}
    \mathrm{Reg}_K = \underbrace{\sum_{k,h} \mathbb{E}_{\pi^\star} [ \langle \pi^\star_h(\cdot \mid s_h) - \pi_{k,h}(\cdot \mid s_h), \widehat{Q}_{k,h}(s_h, \cdot) \rangle ]}_{\text{(I)}} + \underbrace{\sum_{k,h} \mathbb{E}_{\pi^\star} [ \xi_{k,h}(s_h, a_h) ]}_{\text{(II)}} + \underbrace{\sum_{k,h} \mathbb{E}_{\pi_k} [ -\xi_{k,h}(s_h, a_h) ]}_{\text{(III)}}, \label{eq:regret_decomposition}
\end{align}
where $\xi_{k,h}$ is the Bellman error in~\cref{def:bellman_error}. Term (I) corresponds to the regret of the policies $\pi_1, \ldots, \pi_K$ with respect to the $Q$-function estimates $\widehat{Q}_1, \ldots, \widehat{Q}_K$. Term (II) is the expected Bellman error under $\pi^\star$. Term (III) is the expected negative Bellman error under $\pi_k$. We bound Term (I) directly and Terms (II) and (III) using~\cref{lem:informal_lmc_optimism_and_error_bound}.

\paragraph{Bounding Term (I)} Since $\widehat{Q}_{k,h}$ takes values in $[0, H]$, we can apply~\cref{eq:npg-regret-bound} and obtain
\begin{align}
    \text{Term (I)} \leq H^2 \sqrt{2 K \log \lvert \mathcal{A} \rvert} + K \varepsilon_{\proj}.
\end{align}

\paragraph{Bounding Terms (II) and (III)} The optimism guarantee in~\cref{lem:informal_lmc_optimism_and_error_bound} implies that $\xi_{k,h} \leq 0$, and hence $\text{Term (II)} \leq 0$. The error bound in~\cref{lem:informal_lmc_optimism_and_error_bound} reduces Term (III) to the expected bonuses $\Gamma_{\LMC} \sum_{k=1}^K \sum_{h=1}^H \mathbb{E}_{\pi_{k}} \left\| \bm{\phi}(s_h, a_h) \right\|_{\bm{\Lambda}_{k,h}^{-1}}$. Since the $N$ trajectories in $\mathcal{D}_k$ are collected using $\pi_k$, an elliptical potential argument together with a martingale concentration argument shows that the expected bonuses at each episode are $\widetilde{\mathcal{O}}(H \sqrt{d_\phi / N})$. Together with $\Gamma_{\LMC} = \widetilde{\mathcal{O}}(H d_\phi^{3/2})$ and the choice $N = \lceil d_\phi^4 K / \log \lvert \mathcal{A} \rvert \rceil$, this yields
\begin{align}
    \text{Term (II)} + \text{Term (III)} \leq \widetilde{\mathcal{O}} \left( H^2 \sqrt{K} \right).
\end{align}

\paragraph{Putting Everything Together} Since $\SubOpt(\overline{\pi}_K) = \mathrm{Reg}_K / K$, combining the above bounds leads to the stated bound, and the chosen value of $K$ leads to the stated sample complexity.
\end{proof}

\paragraph{Comparison to~\citet{liu2023optimistic}} The on-policy sample complexity in~\cref{thm:on_policy_sample_efficiency} matches the dependence on $\epsilon$ of the on-policy variant in~\citet[Theorem 1]{liu2023optimistic}. However, unlike the proposed algorithm,~\citet{liu2023optimistic} uses $\NPG$ with implicit policies for the actor. Consequently, sampling from the current policy (to interact with the environment) requires calculating $\sum_{\tau=1}^{k-1} \widehat{Q}_{\tau,h} (\cdot, \cdot)$ for each encountered state and every action. Since they use clipped Q functions with $\UCB$ bonuses for the critic, the above sum of Q functions cannot be stored succinctly. Hence, sampling from the policy requires instantiating each previous Q function for each step $h$ and episode $k$, resulting in a computational cost of $\mathcal{O}(d_\phi^2 H \lvert \mathcal{A} \rvert \epsilon^{-2})$.

\subsection{Off-Policy Setting}
\label{subsec:analysis_for_off_policy}
Next, we show that, in the off-policy setting,~\cref{alg:optimistic_actor_critic} can achieve $\widetilde{\mathcal{O}}(\epsilon^{-2})$ sample complexity.

\begin{theorem}
\label{thm:off_policy_sample_efficiency}
Under~\cref{asm:bias,asm:opt_error}, consider~\cref{alg:optimistic_actor_critic} in the off-policy setting with the $\PNPG$ actor (\cref{alg:actor_npg}) and the $\LMC$ critic (\cref{alg:lmc}). For an appropriate choice of the actor and critic parameters, including $\eta = \sqrt{2 \log \lvert \mathcal{A} \rvert / (K H^2)}$, and any $\delta \in (0,1)$, it holds that with probability at least $1 - \delta$,
\begin{align}
    \SubOpt(\overline{\pi}_K) \leq \widetilde{\mathcal{O}} \left( H^2 \sqrt{\frac{d_\phi^4 + d_\phi^2 d_x}{K}} + \varepsilon_{\proj} \right),
\end{align}
where $\varepsilon_{\proj}$ is the projection error defined in~\cref{eq:eps_proj} and controlled by~\cref{lem:projection_error}. Hence, for any $\epsilon > 0$,~\cref{alg:optimistic_actor_critic} with $K = (\log \lvert \mathcal{A} \rvert + d_\phi^4 + d_\phi^2 d_x) H^4 \epsilon^{-2}$ yields an $\widetilde{\mathcal{O}}(\epsilon + \varepsilon_{\proj})$-optimal mixture policy, and therefore requires $K \times 1 = \widetilde{\mathcal{O}}((d_\phi^4 + d_\phi^2 d_x) H^4 \epsilon^{-2})$ trajectories.
\end{theorem}

\begin{proof}[Proof Sketch]
    We follow the proof structure of~\cref{thm:on_policy_sample_efficiency} and use the same regret decomposition in~\cref{eq:regret_decomposition}. Term (I) is bounded in the same way, and the optimism guarantee in~\cref{lem:informal_lmc_optimism_and_error_bound} again implies that $\text{Term (II)} \leq 0$. Compared to the on-policy setting, the most significant difference is the bound on Term (III). Using the error bound in~\cref{lem:informal_lmc_optimism_and_error_bound}, we have that
    \begin{align}
        \text{Term (III)} \leq \Gamma_{\LMC} \sum_{k=1}^K \sum_{h=1}^H \mathbb{E}_{\pi_{k}} \left\| \bm{\phi}(s_h, a_h) \right\|_{\bm{\Lambda}_{k,h}^{-1} }.
    \end{align}
    The term, $ \sum_{k=1}^K \sum_{h=1}^H \mathbb{E}_{\pi_{k}} \left\| \bm{\phi}(s_h, a_h) \right\|_{\bm{\Lambda}_{k,h}^{-1} }$, can be bounded by $\widetilde{\mathcal{O}}(H \sqrt{d_\phi K})$ using the standard elliptical potential lemma and a martingale concentration argument. However, since we are in the off-policy setting, bounding $\Gamma_{\LMC}$ requires the uniform concentration argument from~\citet{jin2020provably} over the value function class $\mathcal{V}_{h}$ defined below, which yields that, with probability at least $1 - \delta$, for all $(k, h)$,
    \begin{align}
        \left\| \sum_{\tau_i \in \mathcal{D}_k} \bm{\phi}_{i,h} \left[ \widehat{V}_{k,h+1}(s_{i,h+1}) - \mathbb{P}_h \widehat{V}_{k,h+1}(s_{i,h}, a_{i,h}) \right] \right\|_{\bm{\Lambda}_{k,h}^{-1}} \leq \mathcal{O} \left( H \sqrt{d_\phi \log (K H / \delta) + \log N_\Delta (\mathcal{V}_{h+1})} + K \Delta \right).
    \end{align}
    This involves obtaining a bound on $\log N_\Delta (\mathcal{V}_{h})$, the logarithm of the covering number of the value function class. The covering number is a measure of the complexity of the space of value functions. In particular, we show that for an actor with log-linear policies, we can bound the logarithm of the covering number using the following lemma.
    \begin{lemma}
        \label{lem:covering_number_of_v}
        Let $\Pi_{R_\theta} \coloneq \{ \pi(\cdot \mid \cdot, \bm{\theta}) \propto \exp(\bm{x}(\cdot, \cdot)^\top \bm{\theta}) \mid \bm{\theta} \in \mathbb{R}^{d_x}, \left\| \bm{\theta} \right\|_2 \leq R_\theta \}$ be the class of log-linear policies with bounded parameters. For $h \in [H]$, let $\mathcal{Q}_h \coloneq \{ \Clip_{[0, H-h+1]} ( \max_{m \in [M]} \bm{\phi}(\cdot, \cdot)^\top \bm{w}^{(m)} ) \mid \left\| \bm{w}^{(m)} \right\|_2 \leq R_w, \forall m \in [M] \}$ be the $Q$-function class and $\mathcal{V}_h \coloneq \{ \left\langle Q(\cdot, \cdot), \pi (\cdot \mid \cdot) \right\rangle_{\mathcal{A}} \mid Q \in \mathcal{Q}_h, \pi \in \Pi_{R_\theta}\}$ be the corresponding value function class. Then, for every $h \in [H]$, it holds that \[\log N_\Delta (\mathcal{V}_h) \leq \mathsf{V} \coloneq M d_\phi \log \left( 1 + \frac{4 R_w}{\Delta} \right) + d_x \log \left( 1 + \frac{8 H R_\theta}{\Delta} \right).\]
    \end{lemma}
    We apply~\cref{lem:covering_number_of_v} with $R_w$ and $R_\theta$ given by the bounds on the critic and actor parameters. In particular, with high probability, the norm of every critic parameter is bounded by a polynomial in $K$ (\cref{lem:bounded_critic_parameters}), and the norm of every actor parameter is bounded by $\mathcal{O}(K / \lambda_G)$, a deterministic bound fixed before data collection (\cref{lem:bounded_parameters}). Since $M = \widetilde{\mathcal{O}}(d_\phi)$ and both radii are polynomial in $K$ and $1 / \lambda_G$, setting $\Delta = H / K$ yields $\mathsf{V} = \widetilde{\mathcal{O}}(d_\phi^2 + d_x)$, and hence $\Gamma_{\LMC} = \widetilde{\mathcal{O}}(H \sqrt{d_\phi^3 + d_\phi d_x})$. Putting everything together yields that
    \begin{align}
        \text{Term (II)} + \text{Term (III)} \leq \widetilde{\mathcal{O}} \left( \sqrt{(d_\phi^4 + d_\phi^2 d_x) H^4 K} \right).
    \end{align}
    Following a proof similar to~\cref{thm:on_policy_sample_efficiency} leads to the desired sample complexity.
\end{proof}

\begin{remark}
To effectively bound the logarithm of the covering number, previous work~\citep{sherman2023rate,cassel2024warm,tan2025actor} has incorporated various algorithmic tweaks, including reward-free warm-ups, feature contractions, and rare-switching. On the contrary, since our algorithm projects the implicit policy onto the log-linear policy class at each iteration, the logarithm of the covering number can be bounded in a more direct manner.
\end{remark}

\paragraph{Comparison to~\citet{sherman2023rate,cassel2024warm}} The off-policy sample complexity in~\cref{thm:off_policy_sample_efficiency} matches the dependence on $\epsilon$ in~\citet{sherman2023rate,cassel2024warm}. Similar to~\citet{liu2023optimistic}, both works use $\NPG$ with implicit policies for the actor and $\UCB$ bonuses for the critic. Consequently, their cost of policy inference scales at least quadratically with $d_\phi$ (\cref{tab:comparison}), whereas ours scales linearly with $d_x$. Finally, we note that similar to the proposed algorithm (see~\cref{subsec:bounding_projection_error}),~\citet{sherman2023rate} also uses a reward-free warm-up phase~\citep{wagenmaker2022reward}. However, while we may use the warm-up phase to identify a coreset to efficiently minimize the actor loss (a computational reason), ~\citet{sherman2023rate} requires the warm-up procedure to restrict the subsequent regret minimization procedure to high occupancy regions of the state-action space and effectively control the capacity of the policy class (a statistical reason).

\section{Conclusion}
\label{sec:conclusion}
We proposed an optimistic actor-critic algorithm with explicitly parameterized policies and a systematic exploration mechanism. In particular, for the actor, we demonstrated that using $\PNPG$ with parametric policies is not only practical, but also equipped with theoretical guarantees. For the critic, we demonstrated that $\LMC$ is a principled and easy-to-implement exploration scheme for policy optimization methods. We derived theoretical guarantees in both the on-policy and off-policy settings, showcasing that the proposed actor-critic algorithm can simultaneously achieve sample efficiency and practicality. Specifically, when the policy features are sufficiently expressive, it achieves $\widetilde{\mathcal{O}}(\epsilon^{-4})$ and $\widetilde{\mathcal{O}}(\epsilon^{-2})$ sample complexities in the on-policy and off-policy settings, respectively. Compared to implicit policies, a parametric policy is more aligned with practice, keeps the cost of policy inference low, and simplifies the covering argument in the off-policy setting.

For future work, eliminating the dependence on $\mathcal{C}$ and $\rho_{\mathcal{C}}$ will result in a more practical algorithm. Another direction is to analyze other randomized exploration schemes for the critic, such as ensemble sampling, within our framework. We also aim to investigate the actor-critic algorithms in more realistic setups (e.g., infinite-horizon discounted MDPs) with more general function approximation schemes beyond linear models for both the environment and the policy. It would also be fruitful to further evaluate their empirical performance in more challenging large-scale deep RL applications.

\section*{Acknowledgments}
We would like to thank Xingtu Liu, Yunxiang Li, and the anonymous reviewers for their helpful feedback. This work was partially supported by the Natural Sciences and Engineering Research Council of Canada (NSERC) Discovery Grant RGPIN-2022-04816, and enabled in part by support provided by the Digital Research Alliance of Canada (alliancecan.ca).

\bibliography{references}

\clearpage
\appendix

\addtocontents{toc}{\protect\setcounter{tocdepth}{3}}
{\small\tableofcontents}
\clearpage

\section{Notations}
\label{sec:notation}
\begin{table}[!ht]
\centering
\caption{Notations for problem definition and algorithm design.}
\label{tab:notation}
\begin{tabular}{@{}ll@{}}
\toprule[1.2pt]
Notation & Meaning \\
\midrule
\textbf{Problem Definition} & \\
$\mathcal{S}$, $\mathcal{A}$ & state space and action space \\
$H$, $h$ & horizon length (total number of steps), index of the current step \\
$r = \{r_h\}_{h \in [H]}$ & reward functions \\
$\mathbb{P} = \{\mathbb{P}_h\}_{h \in [H]}$ & transition kernels \\
$\bm{\phi}: \mathcal{S} \times \mathcal{A} \to \mathbb{R}^{d_\phi}$ & features of the linear MDP \\
$\bm{x}: \mathcal{S} \times \mathcal{A} \to \mathbb{R}^{d_x}$ & features of the log-linear policy \\
\midrule
\textbf{Algorithm Design} & \\
$K$, $k$ & total number of episodes, index of the current episode \\
$\mathcal{D}_k$, $\tau_i$ & data collected at episode $k$, the $i$-th trajectory in $\mathcal{D}_k$ \\
$N$ & number of trajectories collected per episode (on-policy setting) \\
$\bm{w} \in \mathbb{R}^{d_\phi}$ & critic parameters \\
$J_k$, $j$ & number of critic updates at episode $k$, index of the critic update \\
$\eta^{\cri}_{k,h}$ & critic learning rate \\
$\bm{\nu}$ & noise vector of $\LMC$ sampled from the standard normal distribution \\
$\beta$ & inverse temperature of $\LMC$ \\
$M$, $m$ & number of critic samples, index of the critic sample \\
$\mathcal{C}$, $\rho_{\mathcal{C}}$ & coreset in $\mathcal{S} \times \mathcal{A}$, distribution over the coreset \\
$\bm{\theta} \in \mathbb{R}^{d_x}$ & actor parameters \\
$T_k$, $t$ & number of actor updates at episode $k$, index of the actor update \\
$\eta^{\act}$ & actor learning rate \\
$\eta$ & policy optimization learning rate \\
\midrule
\textbf{Analysis} & \\
$\ell^{\cri}_{k,h}$ & critic loss \\
$\ell^{\act}_{k,h}$, $\widetilde{\ell}^{\act}_{k,h}$ & global actor loss, coreset actor loss \\
$\bm{G}$, $\lambda_G$, $\chi_G$ & coreset covariance matrix, its minimum eigenvalue, coreset quality \\
$\varepsilon_{\bias}$, $\varepsilon_{\opt}$ & approximation error and optimization error of the actor \\
$\varepsilon_{k,h}$, $\varepsilon_{\proj}$ & projection error at episode $k$ and step $h$, projection error per episode \\
\bottomrule[1.2pt]
\end{tabular}
\end{table}

\paragraph{Additional Notations} Throughout this paper, the first subscript indexes the learning episode and the second subscript indexes the step within the horizon. For example, $\widehat{V}_{k,h}$ denotes the estimated value function at step $h$ in episode $k$. When the step index is omitted, the symbol denotes the collection over all steps (e.g., $\widehat{V}_{k} \coloneq \{ \widehat{V}_{k,h} \}_{h \in [H]}$). Superscripts in parentheses index the inner iterations of the actor and the critic (e.g., $\bm{\theta}^{(t)}_{k,h}$ and $\bm{w}^{(m,j)}_{k,h}$), and $\bm{w}^{(m)}_{k,h} \coloneq \bm{w}^{(m,J_k)}_{k,h}$ denotes the last critic iterate. For a trajectory $\tau_i = (s_{i,1}, a_{i,1}, \ldots, s_{i,H}, a_{i,H}, s_{i,H+1})$, we write $\bm{\phi}_{i,h} \coloneq \bm{\phi}(s_{i,h}, a_{i,h})$, and $\lvert \mathcal{D}_{k} \rvert$ denotes the number of trajectories in $\mathcal{D}_{k}$. Vectors and matrices are written in bold, except for the identity matrix $I$. For any vector $\bm{v} \in \mathbb{R}^d$ and any positive semi-definite matrix $\bm{M} \in \mathbb{R}^{d \times d}$, we denote $\left\| \bm{v} \right\|_{\bm{M}} \coloneq \sqrt{\bm{v}^\top \bm{M} \bm{v}}$. Throughout the paper, we use $\widetilde{\mathcal{O}}(\cdot)$ to hide polylogarithmic factors in $K$, $H$, $d_\phi$, $d_x$, $\lvert \mathcal{A} \rvert$, $1/\lambda_G$, and $1/\delta$.

\section{Analysis of the PNPG Actor}
\label{sec:actor_analyses}

\subsection{Projected NPG Regret (Proof of \texorpdfstring{\cref{lem:generalized_omd_regret}}{})}

\begin{proof}[Proof of \cref{lem:generalized_omd_regret}]
We consider a fixed $u \in \Delta_n$, and write $\varepsilon_k \coloneq \varepsilon_k(u)$ for simplicity. Given the update of $p_{k+1/2}$ in~\cref{eq:standard_npg} and the fact that $\Delta_n$ is a convex set, we have the following optimality condition:
\begin{align}
    \left\langle u - p_{k+1/2}, - \eta g_{k} + \log(p_{k+1/2}) - \log(p_{k}) \right\rangle \geq 0. \label{eq:npg_optimality_condition}
\end{align}
Then, for each $k \in [K]$, we have that
    \begin{align}
    \MoveEqLeft
    \left\langle u - p_{k}, \eta g_{k} \right\rangle = \left\langle u - p_{k+1/2}, \eta g_{k} \right\rangle + \left\langle p_{k+1/2} - p_{k}, \eta g_{k} \right\rangle \\
    &= \left\langle u - p_{k+1/2}, \eta g_{k} - \log(p_{k+1/2}) + \log(p_{k}) \right\rangle + \left\langle u - p_{k+1/2}, \log(p_{k+1/2}) - \log(p_{k}) \right\rangle \\
    & \quad+ \left\langle p_{k+1/2} - p_{k}, \eta g_{k} \right\rangle \\
    &\leq \left\langle u - p_{k+1/2}, \log(p_{k+1/2}) - \log(p_{k}) \right\rangle + \left\langle p_{k+1/2}- p_{k}, \eta g_{k} \right\rangle \\
    &= \KL (u \| p_{k}) - \KL (u \| p_{k+1/2}) - \KL(p_{k+1/2} \| p_{k}) + \left\langle p_{k+1/2}- p_{k}, \eta g_{k} \right\rangle \\
    &\leq \KL (u \| p_{k}) - \KL (u \| p_{k+1/2}) - \KL(p_{k+1/2} \| p_{k}) + \frac{1}{2} \left\| p_{k+1/2}- p_{k} \right\|_1^2 + \frac{1}{2} \left\| \eta g_{k} \right\|_\infty^2 \\
    &\leq \KL (u \| p_{k}) - \KL (u \| p_{k+1/2}) - \KL(p_{k+1/2} \| p_{k}) + \KL(p_{k+1/2} \| p_{k}) + \frac{\eta^2 L^2}{2} \\
    &= \KL (u \| p_{k}) - \KL (u \| p_{k+1/2}) + \frac{\eta^2 L^2}{2} \\
    &= \KL (u \| p_{k}) - \KL (u \| p_{k+1}) + \KL (u \| p_{k+1}) - \KL (u \| p_{k+1/2}) + \frac{\eta^2 L^2}{2} \\
    &= \KL (u \| p_{k}) - \KL (u \| p_{k+1}) + \varepsilon_{k} + \frac{\eta^2 L^2}{2}.
\end{align}
The second equality adds and subtracts $\left\langle u - p_{k+1/2}, \log(p_{k+1/2}) - \log(p_{k}) \right\rangle$. The first inequality drops the first term due to~\cref{eq:npg_optimality_condition}. The third equality applies the three-point property of Bregman divergence (\cref{lem:three_point}) by setting $\bm{v}_1 = u$, $\bm{v}_2 = p_{k+1/2}$, and $\bm{v}_3 = p_{k}$. The second inequality follows from H\"older's inequality and then Young's inequality (i.e., $\langle \bm{v}_1, \bm{v}_2 \rangle \leq \left\| \bm{v}_1 \right\|_1 \left\| \bm{v}_2 \right\|_\infty \leq \left\| \bm{v}_1 \right\|_1^2 / 2 + \left\| \bm{v}_2 \right\|_\infty^2 / 2$), and the last inequality applies Pinsker's inequality and $\left\| g_{k} \right\|_\infty \leq L$. The last equality uses the definition of $\varepsilon_k$. Summing up the above inequality from $k=1$ to $K$ yields that
\begin{align}
    \MoveEqLeft
    \sum_{k=1}^K \left\langle u - p_{k}, \eta g_{k} \right\rangle \leq \sum_{k=1}^K \KL (u \| p_{k}) - \KL (u \| p_{k+1}) + \sum_{k=1}^K \varepsilon_{k} + \frac{\eta^2 L^2 K}{2} \\
    &= \KL (u \| p_{1}) - \KL (u \| p_{K+1}) + \sum_{k=1}^K \varepsilon_{k} + \frac{\eta^2 L^2 K}{2} \\
    &\leq \KL (u \| p_{1}) + \sum_{k=1}^K \varepsilon_{k} + \frac{\eta^2 L^2 K}{2} \\
    &= \sum_{i=1}^n u(i) \log(u(i)) - \sum_{i=1}^n u(i) \log(p_{1}(i)) + \sum_{k=1}^K \varepsilon_{k} + \frac{\eta^2 L^2 K}{2} \\
    &\leq \log n + \sum_{k=1}^K \varepsilon_{k} + \frac{\eta^2 L^2 K}{2}.
\end{align}
The first equality telescopes the sum. The second inequality follows from the fact that KL divergence is non-negative, and the last inequality holds because the first term is non-positive, and for the second term, $p_{1}$ is a uniform distribution. Dividing both sides by $\eta$, we have that
\begin{align}
    \sum_{k=1}^K \left\langle u - p_{k}, g_{k} \right\rangle \leq \frac{\log n + \sum_{k=1}^K \varepsilon_{k}}{\eta} + \frac{\eta L^2 K}{2},
\end{align}
which concludes the proof.
\end{proof}

\subsection{Projection Error (Proof of \texorpdfstring{\cref{lem:projection_error}}{})}

\begin{proof}[Proof of \cref{lem:projection_error}]
Consider specific $k \in [K]$, $h \in [H]$, $s \in \mathcal{S}$, and $u \in \Delta(\mathcal{A})$. We compare the updated parameters $\bm{\theta}_{k+1,h}$ with two reference parameters. The first one, denoted by $\bm{\theta}^\star_{k+1,h}$, is the minimizer of the coreset actor loss $\widetilde{\ell}^{\act}_{k,h}$. Since $\widetilde{\ell}^{\act}_{k,h}$ is a quadratic function with Hessian $2 \bm{G}$, and $\bm{G}$ is positive semi-definite by construction and invertible by the definition of $\chi_G$, $\widetilde{\ell}^{\act}_{k,h}$ is strictly convex, and its minimizer exists and is unique. The second one is a parameter $\bm{\theta}^\dagger \in \mathbb{R}^{d_x}$ that satisfies $\ell^{\act}_{k,h}(\bm{\theta}^\dagger) \leq \varepsilon_{\bias}$, e.g., a minimizer of the global actor loss, which exists by~\cref{asm:bias}. Moreover, we define $z(a) \coloneq z_{k,h}(s, a)$ and $y(a) \coloneq \bm{x}(s, a)^\top \bm{\theta}_{k+1,h}$ as the logits of $\pi_{k+1/2,h}(\cdot \mid s)$ and $\pi_{k+1,h}(\cdot \mid s)$, respectively, and let $Z(\bm{v}) \coloneq \log \sum_{a \in \mathcal{A}} \exp(\bm{v}(a))$ denote the log-partition function of a logit vector $\bm{v} \in \mathbb{R}^{\lvert \mathcal{A} \rvert}$. Then, we decompose the projection error as follows:
\begin{align}
    \MoveEqLeft
    \left\lvert \varepsilon^{\proj}_{k,h}(u, s) \right\rvert = \left\lvert \KL \left( u \| \pi_{k+1,h}(\cdot \mid s) \right) - \KL \left( u \| \pi_{k+1/2,h}(\cdot \mid s) \right) \right\rvert \\
    &= \left\lvert \sum_{a \in \mathcal{A}} u(a) \left( z(a) - y(a) \right) - \left( Z(z) - Z(y) \right) \right\rvert \\
    &\leq \max_{a \in \mathcal{A}} \left\lvert z(a) - y(a) \right\rvert + \left\lvert Z(z) - Z(y) \right\rvert \\
    &\leq 2 \max_{a \in \mathcal{A}} \left\lvert z(a) - y(a) \right\rvert \\
    &\leq 2 \sup_{(s^\prime, a^\prime) \in \mathcal{S} \times \mathcal{A}} \left\lvert \bm{x}(s^\prime, a^\prime)^\top \bm{\theta}_{k+1,h} - z_{k,h}(s^\prime, a^\prime) \right\rvert \\
    &\leq 2 \underbrace{\sup_{(s^\prime, a^\prime) \in \mathcal{S} \times \mathcal{A}} \left\lvert \bm{x}(s^\prime, a^\prime)^\top \left( \bm{\theta}_{k+1,h} - \bm{\theta}^\star_{k+1,h} \right) \right\rvert}_{\text{(I)}} \\
    &\quad + 2 \underbrace{\sup_{(s^\prime, a^\prime) \in \mathcal{S} \times \mathcal{A}} \left\lvert \bm{x}(s^\prime, a^\prime)^\top \left( \bm{\theta}^\star_{k+1,h} - \bm{\theta}^\dagger \right) \right\rvert}_{\text{(II)}} \\
    &\quad + 2 \underbrace{\sup_{(s^\prime, a^\prime) \in \mathcal{S} \times \mathcal{A}} \left\lvert \bm{x}(s^\prime, a^\prime)^\top \bm{\theta}^\dagger - z_{k,h}(s^\prime, a^\prime) \right\rvert}_{\text{(III)}}. \label{eq:proj_decomposition}
\end{align}
The first equality follows from the definition of $\varepsilon^{\proj}_{k,h}(u, s)$ in~\cref{eq:projection_error}. The second equality uses the expressions ${\log \pi_{k+1/2,h}(a \mid s) = z(a) - Z(z)}$ and ${\log \pi_{k+1,h}(a \mid s) = y(a) - Z(y)}$, and the fact that the terms $u(a) \log u(a)$ cancel in the difference between the two KL divergences. The first inequality applies the triangle inequality and uses that $u$ is a probability vector. The second inequality holds since $Z$ is $1$-Lipschitz with respect to the $\ell_\infty$-norm. Indeed, by the mean value theorem, $Z(z) - Z(y) = \nabla Z(\bar{\bm{v}})^\top (z - y)$ for some $\bar{\bm{v}}$ on the line segment between $y$ and $z$, and $\nabla Z(\bar{\bm{v}})$ is the softmax of $\bar{\bm{v}}$, which is a probability vector. The third inequality substitutes the definitions of $z$ and $y$, and takes the supremum over states. The last inequality applies the triangle inequality and the subadditivity of the supremum. Term (I) is due to the inexact minimization of the coreset actor loss, Term (II) is the gap between the two reference parameters, and Term (III) is the approximation error of the policy features.

We first start by bounding Term (I) of~\cref{eq:proj_decomposition}. To this end, we transfer an error in the parameter space to an error in the logit space. For every $\bm{v} \in \mathbb{R}^{d_x}$, we have that
\begin{align}
    \sup_{(s^\prime, a^\prime) \in \mathcal{S} \times \mathcal{A}} \left\lvert \bm{x}(s^\prime, a^\prime)^\top \bm{v} \right\rvert &= \sup_{(s^\prime, a^\prime) \in \mathcal{S} \times \mathcal{A}} \left\lvert \left( \bm{G}^{-1/2} \bm{x}(s^\prime, a^\prime) \right)^\top \bm{G}^{1/2} \bm{v} \right\rvert \\
    &\leq \sup_{(s^\prime, a^\prime) \in \mathcal{S} \times \mathcal{A}} \left\| \bm{x}(s^\prime, a^\prime) \right\|_{\bm{G}^{-1}} \left\| \bm{v} \right\|_{\bm{G}} \\
    &= \chi_G \left\| \bm{v} \right\|_{\bm{G}}. \label{eq:logit_transfer}
\end{align}
The first equality uses $\bm{G}^{-1/2} \bm{G}^{1/2} = I$ and the symmetry of $\bm{G}^{-1/2}$. The inequality applies the Cauchy–Schwarz inequality, together with $\| \bm{G}^{-1/2} \bm{x} \|_2 = \| \bm{x} \|_{\bm{G}^{-1}}$ and $\| \bm{G}^{1/2} \bm{v} \|_2 = \| \bm{v} \|_{\bm{G}}$. The last equality uses the definition of $\chi_G$. It remains to relate the $\bm{G}$-norm of $\bm{\theta}_{k+1,h} - \bm{\theta}^\star_{k+1,h}$ to the optimization error in~\cref{asm:opt_error}. Since $\bm{\theta}^\star_{k+1,h}$ minimizes the differentiable function $\widetilde{\ell}^{\act}_{k,h}$ in~\cref{eq:actor_loss}, the gradient of $\widetilde{\ell}^{\act}_{k,h}$ vanishes at $\bm{\theta}^\star_{k+1,h}$, i.e.,
\begin{align}
    \sum_{(s^\prime, a^\prime) \in \mathcal{C}} \rho_{\mathcal{C}}(s^\prime, a^\prime) \bm{x}(s^\prime, a^\prime) \left( \bm{x}(s^\prime, a^\prime)^\top \bm{\theta}^\star_{k+1,h} - z_{k,h}(s^\prime, a^\prime) \right) = \bm{0}.
\end{align}
Therefore, for every $\bm{\theta} \in \mathbb{R}^{d_x}$, we have that
\begin{align}
    \widetilde{\ell}^{\act}_{k,h}(\bm{\theta}) &= \sum_{(s^\prime, a^\prime) \in \mathcal{C}} \rho_{\mathcal{C}}(s^\prime, a^\prime) \left[ \bm{x}(s^\prime, a^\prime)^\top \left( \bm{\theta} - \bm{\theta}^\star_{k+1,h} \right) + \left( \bm{x}(s^\prime, a^\prime)^\top \bm{\theta}^\star_{k+1,h} - z_{k,h}(s^\prime, a^\prime) \right) \right]^2 \\
    &= \sum_{(s^\prime, a^\prime) \in \mathcal{C}} \rho_{\mathcal{C}}(s^\prime, a^\prime) \left( \bm{x}(s^\prime, a^\prime)^\top \left( \bm{\theta} - \bm{\theta}^\star_{k+1,h} \right) \right)^2 \\
    &\quad + 2 \left( \bm{\theta} - \bm{\theta}^\star_{k+1,h} \right)^\top \sum_{(s^\prime, a^\prime) \in \mathcal{C}} \rho_{\mathcal{C}}(s^\prime, a^\prime) \bm{x}(s^\prime, a^\prime) \left( \bm{x}(s^\prime, a^\prime)^\top \bm{\theta}^\star_{k+1,h} - z_{k,h}(s^\prime, a^\prime) \right) \\
    &\quad + \sum_{(s^\prime, a^\prime) \in \mathcal{C}} \rho_{\mathcal{C}}(s^\prime, a^\prime) \left( \bm{x}(s^\prime, a^\prime)^\top \bm{\theta}^\star_{k+1,h} - z_{k,h}(s^\prime, a^\prime) \right)^2 \\
    &= \left\| \bm{\theta} - \bm{\theta}^\star_{k+1,h} \right\|_{\bm{G}}^2 + \widetilde{\ell}^{\act}_{k,h}(\bm{\theta}^\star_{k+1,h}).
\end{align}
The first equality adds and subtracts $\bm{x}(s^\prime, a^\prime)^\top \bm{\theta}^\star_{k+1,h}$. The second equality expands the square. The last equality uses the definition of $\bm{G}$ for the first term, the first-order optimality condition above for the second term, and the definition of $\widetilde{\ell}^{\act}_{k,h}$ for the third term. Rearranging the terms, we obtain that, for every $\bm{\theta} \in \mathbb{R}^{d_x}$,
\begin{align}
    \widetilde{\ell}^{\act}_{k,h}(\bm{\theta}) - \widetilde{\ell}^{\act}_{k,h}(\bm{\theta}^\star_{k+1,h}) = \left\| \bm{\theta} - \bm{\theta}^\star_{k+1,h} \right\|_{\bm{G}}^2. \label{eq:quadratic_identity}
\end{align}
Putting everything together, we have that
\begin{align}
    \text{Term (I)} \leq \chi_G \left\| \bm{\theta}_{k+1,h} - \bm{\theta}^\star_{k+1,h} \right\|_{\bm{G}} = \chi_G \sqrt{\widetilde{\ell}^{\act}_{k,h}(\bm{\theta}_{k+1,h}) - \widetilde{\ell}^{\act}_{k,h}(\bm{\theta}^\star_{k+1,h})} \leq \chi_G \sqrt{\varepsilon_{\opt}}.
\end{align}
The first inequality applies~\cref{eq:logit_transfer} with $\bm{v} = \bm{\theta}_{k+1,h} - \bm{\theta}^\star_{k+1,h}$. The equality applies~\cref{eq:quadratic_identity} with $\bm{\theta} = \bm{\theta}_{k+1,h}$. The last inequality follows from~\cref{asm:opt_error}, since $\widetilde{\ell}^{\act}_{k,h}(\bm{\theta}^\star_{k+1,h}) = \min_{\bm{\theta} \in \mathbb{R}^{d_x}} \widetilde{\ell}^{\act}_{k,h}(\bm{\theta})$.

Next, we continue to bound Term (II) of~\cref{eq:proj_decomposition}. Similarly, applying~\cref{eq:logit_transfer} with $\bm{v} = \bm{\theta}^\star_{k+1,h} - \bm{\theta}^\dagger$ gives $\text{Term (II)} \leq \chi_G {\| \bm{\theta}^\dagger - \bm{\theta}^\star_{k+1,h} \|_{\bm{G}}}$, and it suffices to bound the $\bm{G}$-norm on the right-hand side. We have that
\begin{align}
    \MoveEqLeft
    \left\| \bm{\theta}^\dagger - \bm{\theta}^\star_{k+1,h} \right\|_{\bm{G}}^2 = \widetilde{\ell}^{\act}_{k,h}(\bm{\theta}^\dagger) - \widetilde{\ell}^{\act}_{k,h}(\bm{\theta}^\star_{k+1,h}) \leq \widetilde{\ell}^{\act}_{k,h}(\bm{\theta}^\dagger) = \sum_{(s^\prime, a^\prime) \in \mathcal{C}} \rho_{\mathcal{C}}(s^\prime, a^\prime) \left( \bm{x}(s^\prime, a^\prime)^\top \bm{\theta}^\dagger - z_{k,h}(s^\prime, a^\prime) \right)^2 \\
    &\leq \sum_{(s^\prime, a^\prime) \in \mathcal{C}} \rho_{\mathcal{C}}(s^\prime, a^\prime) \sup_{(s^{\prime\prime}, a^{\prime\prime}) \in \mathcal{S} \times \mathcal{A}} \left( \bm{x}(s^{\prime\prime}, a^{\prime\prime})^\top \bm{\theta}^\dagger - z_{k,h}(s^{\prime\prime}, a^{\prime\prime}) \right)^2 = \ell^{\act}_{k,h}(\bm{\theta}^\dagger) \leq \varepsilon_{\bias}.
\end{align}
The first equality applies~\cref{eq:quadratic_identity} with $\bm{\theta} = \bm{\theta}^\dagger$. The first inequality uses $\widetilde{\ell}^{\act}_{k,h} \geq 0$. The second equality uses the definition of $\widetilde{\ell}^{\act}_{k,h}$ in~\cref{eq:actor_loss}. The second inequality holds since $\mathcal{C} \subseteq \mathcal{S} \times \mathcal{A}$ and $\rho_{\mathcal{C}}(s^\prime, a^\prime) \geq 0$, so that each squared logit error on the coreset is at most its supremum over $\mathcal{S} \times \mathcal{A}$. The third equality uses $\sum_{(s^\prime, a^\prime) \in \mathcal{C}} \rho_{\mathcal{C}}(s^\prime, a^\prime) = 1$, since $\rho_{\mathcal{C}} \in \Delta(\mathcal{C})$, and the definition of $\ell^{\act}_{k,h}$ in~\cref{eq:global_actor_loss}. The last inequality follows from the choice of $\bm{\theta}^\dagger$. Therefore, $\text{Term (II)} \leq \chi_G \sqrt{\varepsilon_{\bias}}$.

Lastly, for Term (III) of~\cref{eq:proj_decomposition}, we have that $\text{Term (III)} = \sqrt{\ell^{\act}_{k,h}(\bm{\theta}^\dagger)} \leq \sqrt{\varepsilon_{\bias}}$, where the equality follows from the definition of $\ell^{\act}_{k,h}$ in~\cref{eq:global_actor_loss}, and the inequality follows from the choice of $\bm{\theta}^\dagger$.

Finally, substituting the bounds on Terms (I), (II), and (III) into~\cref{eq:proj_decomposition} and taking the supremum over $s \in \mathcal{S}$ and $u \in \Delta(\mathcal{A})$ yields
\begin{align}
    \varepsilon_{k,h} \leq \sup_{s \in \mathcal{S}} \sup_{u \in \Delta(\mathcal{A})} \left\lvert \varepsilon^{\proj}_{k,h}(u, s) \right\rvert \leq 2 (1 + \chi_G) \sqrt{\varepsilon_{\bias}} + 2 \chi_G \sqrt{\varepsilon_{\opt}}.
\end{align}
This concludes the proof.
\end{proof}

\begin{corollary}[Sufficient Conditions for Small Projection Error]
    \label{cor:projection_error}
    Under~\cref{asm:bias,asm:opt_error}, for any $\epsilon > 0$, if $\varepsilon_{\bias} \leq (\eta \epsilon / (4 H (1 + \chi_G)))^2$ and $\varepsilon_{\opt} \leq (\eta \epsilon / (4 H \chi_G))^2$, then $\varepsilon_{\proj} \leq \epsilon$.
\end{corollary}
This follows by plugging~\cref{lem:projection_error} into~\cref{eq:eps_proj}. Since the target logits in~\cref{eq:target_logits} are $\bm{x}(s, a)^\top \bm{\theta}_{k,h} + \eta \widehat{Q}_{k,h}(s, a)$, $\varepsilon_{\bias}$ scales with $\eta^2$. Hence, the first condition does not depend on $\eta$, and it requires the policy features to approximate every $\widehat{Q}_{k,h}$ uniformly within $\mathcal{O}(\epsilon / (H (1 + \chi_G)))$. Since $\varepsilon_{\opt} \leq \mathcal{O}(\exp(-T_k / \kappa_G))$, the second condition only requires $T_k = \mathcal{O}(\kappa_G \log(H \chi_G / (\eta \epsilon)))$.

\section{Constructing the Coreset via Experimental Design}
\label{sec:experimental_design}

In this section, we introduce various methods of experimental design to bound $\chi_G$ defined in~\cref{subsec:bounding_projection_error}. The experimental design problem can be written as
\begin{align}
    \MoveEqLeft
    \inf_{ \substack{\mathcal{C} \subseteq \mathcal{S} \times \mathcal{A} \\ \rho_{\mathcal{C}} \in \Delta(\mathcal{C}) } } \sup_{(s, a) \in \mathcal{S} \times \mathcal{A}} \left\| \bm{x}(s, a) \right\|_{\bm{G}^{-1}} \\
    &\text{s.t.} \quad \bm{G} = \sum_{(s, a) \in \mathcal{C}} \rho_{\mathcal{C}}(s, a) \bm{x}(s, a) \bm{x}(s, a)^\top.
\end{align}

In~\cref{subsec:kiefer_wolfowitz}, we consider constructing a coreset for the policy features. The Kiefer–Wolfowitz theorem guarantees that there exists a coreset with $\chi_G = \sqrt{d_x}$ and at most $d_x (d_x + 1)/2$ state-action pairs. Such a coreset can be approximately formed using G-experimental design. In~\cref{subsec:exploratory_policy}, we consider using the linear MDP features as the policy features and constructing $\mathcal{C}$ through limited interaction with the environment.

\subsection{Kiefer–Wolfowitz Theorem and G-Experimental Design}
\label{subsec:kiefer_wolfowitz}

We first introduce the Kiefer–Wolfowitz theorem, which guarantees that there exists a coreset $\mathcal{C}$ and its corresponding distribution $\rho_{\mathcal{C}}$ that can be used to bound $\chi_G$.

\begin{theorem}[Kiefer–Wolfowitz; \citealt{kiefer1960equivalence}; {\citealt[Theorem 21.1]{lattimore2020bandit}}]
\label{thm:kiefer_wolfowitz}
Let $\mathcal{X} \subset \mathbb{R}^d$ be a compact set that spans $\mathbb{R}^d$. For $\rho \in \Delta(\mathcal{X})$, define $\bm{G}(\rho) \coloneq \mathbb{E}_{\bm{v} \sim \rho} [ \bm{v} \bm{v}^\top ]$ and $g(\rho) \coloneq \sup_{\bm{v} \in \mathcal{X}} \left\| \bm{v} \right\|^2_{\bm{G}(\rho)^{-1}}$. The following are equivalent for $\rho^\star \in \Delta(\mathcal{X})$: (i) $\rho^\star$ maximizes $\log \det \bm{G}(\rho)$, (ii) $\rho^\star$ minimizes $g(\rho)$, and (iii) $g(\rho^\star) = d$. Moreover, some minimizer $\rho^\star$ is supported on at most $d (d+1)/2$ points of $\mathcal{X}$.
\end{theorem}
Applying~\cref{thm:kiefer_wolfowitz} to the image $\mathcal{X} = \{ \bm{x}(s, a) : (s, a) \in \mathcal{S} \times \mathcal{A} \}$ of the policy feature map, which we assume to be compact and to span $\mathbb{R}^{d_x}$ (e.g., when $\mathcal{S}$ is finite and the features span $\mathbb{R}^{d_x}$), gives a coreset $\mathcal{C}$, consisting of state-action pairs whose features form the support of $\rho^\star$, with $\lvert \mathcal{C} \rvert \leq d_x (d_x + 1)/2$ and $\chi_G = \sqrt{d_x}$. Here, $\lambda_G > 0$, since $g(\rho^\star) = d_x$ is finite only if $\bm{G}(\rho^\star)$ is invertible. The coreset depends only on the known policy features and is fixed before data collection. In practice, the coreset can be computed by the Frank–Wolfe algorithm~\citep{frank1956algorithm} for the G-optimal design problem~\citep{todd2016minimum}, which returns a design $\rho$ with $g(\rho) \leq 2 d_x$ supported on $\mathcal{O}(d_x \log \log d_x)$ points~\citep[Chapter 21]{lattimore2020bandit}, at the cost of replacing $\chi_G = \sqrt{d_x}$ by $\chi_G \leq \sqrt{2 d_x}$. Each iteration of this algorithm maximizes $\left\| \bm{x}(s, a) \right\|_{\bm{G}(\rho)^{-1}}$ over the state-action pairs, which requires either one pass over $\mathcal{S} \times \mathcal{A}$ before learning or an oracle for this maximization. In both cases, the coreset is computed once, and its cost does not grow with $K$. Using such $\mathcal{C}$ and $\rho_{\mathcal{C}}$ to construct the coreset actor loss in~\cref{alg:actor_npg},~\cref{lem:projection_error} bounds the projection error as $\varepsilon_{k,h} \leq \mathcal{O}(\sqrt{d_x} (\sqrt{\varepsilon_{\bias}} + \sqrt{\varepsilon_{\opt}}))$.

\subsection{Exploratory Policy and Minimum Eigenvalue}
\label{subsec:exploratory_policy}

Alternatively, we can choose to use the linear MDP features as the policy features (i.e., $\bm{x} = \bm{\phi}$) and construct $\mathcal{C}$ via interacting with the environment. Note that bounding $\chi_G$ is equivalent to controlling $\left\| \bm{\phi}(s, a) \right\|_{\bm{G}^{-1}}$ for all $(s, a) \in \mathcal{S} \times \mathcal{A}$. Consequently, given that $\left\| \bm{\phi}(s, a) \right\|_2 \leq 1$ by the linear MDP assumption and since
\begin{align}
    \left\| \bm{\phi}(s, a) \right\|_{\bm{G}^{-1}} \leq \frac{\left\| \bm{\phi}(s, a) \right\|_2}{\sqrt{\lambda_{\min} (\bm{G})}} \leq \frac{1}{\sqrt{\lambda_G}},
\end{align}
we only need a well-conditioned covariance matrix $\bm{G}$ that has a positive minimum eigenvalue.

Several existing works~\citep{hao2021online,agarwal2021online} assume access to an exploratory (not necessarily optimal) policy $\pi_{\mathrm{exp}}$ that is able to collect such covariance matrices with minimum eigenvalue bounded away from $0$. Given that, we can directly apply $\pi_{\mathrm{exp}}$ to roll-out trajectories and collect observations, which can be used to construct $\mathcal{C}$ and the corresponding covariance $\bm{G}$.

However, in practice, we rarely have access to such an oracle policy. Consequently,~\citet{wagenmaker2022reward} proposed a reward-free approach, \texttt{CoverTraj}, that can effectively collect such observations without assuming access to an exploratory policy. In particular, the \texttt{CoverTraj} algorithm offers the following theoretical guarantee.

\begin{theorem}[{\citealt[Theorem 4]{wagenmaker2022reward}}]
    Consider a fixed $h \in [H]$ and $\zeta \in [0, 1]$. Suppose there exists a problem-dependent constant $\varepsilon_{\mathcal{M}} > 0$ such that $\sup_{\pi \in \Pi} \lambda_{\min} \left( \mathbb{E}_{\pi} \left[ \bm{\phi}(s, a) \bm{\phi}(s, a)^\top \right] \right) \geq \varepsilon_{\mathcal{M}}$. Running $K_0$ rounds of $\texttt{CoverTraj}$ to collect $\mathcal{C} = \left\{ (s_{\tau,h}, a_{\tau,h}) \right\}_{\tau = 1}^{K_0}$ where
    \begin{align}
        K_0 = \widetilde{\mathcal{O}} \left( \frac{1}{\varepsilon_{\mathcal{M}}} \cdot \max \left\{ \frac{d_\phi}{\zeta^2}, d_\phi^4 H^3 \log^3 \frac{1}{\delta} \right\} \right)
    \end{align}
    ensures that for any $\delta \in (0, 1)$, with probability of at least $1 - \delta$,
    \begin{align}
        \lambda_{\min} \left( \bm{\Lambda}_{\mathcal{C}} \right) \geq \frac{\varepsilon_{\mathcal{M}}}{\zeta^2},
    \end{align}
    where $\bm{\Lambda}_{\mathcal{C}} \coloneq \sum_{\tau=1}^{K_0} \bm{\phi}(s_{\tau,h}, a_{\tau,h}) \bm{\phi}(s_{\tau,h}, a_{\tau,h})^\top$.
\end{theorem}
The above theorem bounds the unnormalized matrix $\bm{\Lambda}_{\mathcal{C}}$, whereas~\cref{lem:projection_error} uses the Gram matrix $\bm{G}$ weighted by $\rho_{\mathcal{C}} \in \Delta(\mathcal{C})$. Taking $\rho_{\mathcal{C}}$ to be uniform over the $K_0$ collected state-action pairs (counted with multiplicity) gives $\bm{G} = \bm{\Lambda}_{\mathcal{C}} / K_0$. Hence, $\lambda_G \geq \varepsilon_{\mathcal{M}} / (K_0 \zeta^2)$ and $\chi_G \leq 1 / \sqrt{\lambda_G} \leq \zeta \sqrt{K_0 / \varepsilon_{\mathcal{M}}}$. In particular, choosing $\zeta$ small enough that the first term in $K_0$ dominates gives $\chi_G = \widetilde{\mathcal{O}}(\sqrt{d_\phi} / \varepsilon_{\mathcal{M}})$.

Note that $\texttt{CoverTraj}$ does not utilize the reward function of the MDP and merely uses the transition kernel when interacting with the environment. Alternatively,~\citet{wagenmaker2022instance} provides another approach, \texttt{OptCov}, that utilizes regret minimization algorithms to construct the desired covariance matrix. According to~\citet[Theorem 9]{wagenmaker2022instance}, \texttt{OptCov} offers a similar guarantee on the minimum eigenvalue of the collected covariance matrix.

To conclude, the Frank–Wolfe algorithm can be used to form a coreset and subsequently bound $\chi_G$ for any given policy features. If we use the linear MDP features as the policy features, we can construct $\mathcal{C}$ by interacting with the environment. Either having access to an exploratory policy or running \texttt{CoverTraj} or \texttt{OptCov} can offer guarantees on the minimum eigenvalues of the covariance matrix, which will consequently control $\chi_G$.

\section{Analysis of the LMC Critic}
\label{sec:critic_analyses}

In this section, we analyze the $\LMC$ critic in~\cref{alg:lmc} and prove~\cref{lem:informal_lmc_optimism_and_error_bound}. In~\cref{sec:ucb_critic}, we additionally instantiate the critic with $\UCB$ bonuses and compare the resulting guarantees. Throughout this section, we set $\lambda = 1$, and we start by introducing the notation that will be used in the analysis.

\paragraph{Critic Loss and Its Minimizer} Recall that for every $(k, h) \in [K] \times [H]$, we define $\bm{\Lambda}_{k,h} \coloneq \sum_{\tau_i \in \mathcal{D}_k} \bm{\phi}_{i,h} \bm{\phi}_{i,h}^\top + \lambda I$. We first compute the gradient of the critic loss in~\cref{eq:critic_loss}, which is given by
\begin{align}
    \nabla \ell^{\cri}_{k,h}(\bm{w}) = \bm{\Lambda}_{k,h} \bm{w} - \bm{b}_{k,h}, \label{eq:derivative_of_critic_loss}
\end{align}
where $\bm{b}_{k,h} \coloneq \sum_{\tau_i \in \mathcal{D}_k} [ r_h(s_{i,h}, a_{i,h}) + \widehat{V}_{k,h+1}(s_{i,h+1}) ] \bm{\phi}_{i,h}$. Consequently, by setting $\nabla \ell^{\cri}_{k,h}(\bm{w}) = \bm{0}$, we get the minimizer of $\ell^{\cri}_{k,h}(\bm{w})$ as
\begin{align}
    \widehat{\bm{w}}_{k,h} \coloneq \bm{\Lambda}_{k,h}^{-1} \bm{b}_{k,h}. \label{eq:w_hat}
\end{align}
We also denote the regression target by $\mathcal{T}_{k,h}(s, a) \coloneq r_h(s, a) + \mathbb{P}_h \widehat{V}_{k,h+1}(s, a)$, so that the Bellman error in~\cref{def:bellman_error} can be written as $\xi_{k,h} = \mathcal{T}_{k,h} - \widehat{Q}_{k,h}$. Since $\widehat{Q}_{k,h+1}$ is clipped to $[0, H-h]$ in~\cref{alg:lmc}, the estimated value function $\widehat{V}_{k,h+1}$ takes values in $[0, H - h]$. Together with $r_h(s, a) \in [0, 1]$, this implies that $\mathcal{T}_{k,h}(s, a) \in [0, H - h + 1]$ for all $(s, a) \in \mathcal{S} \times \mathcal{A}$.

\subsection{Proof of \texorpdfstring{\cref{lem:informal_lmc_optimism_and_error_bound}}{}}
\label{subsec:proof_of_lmc_lemma}

In order to prove~\cref{lem:informal_lmc_optimism_and_error_bound}, we define the following good event, under which the estimated value functions concentrate and the Gaussian perturbations injected by $\LMC$ behave well.
\begin{definition}[Good Event]
    \label{def:good_event}
    For $C_\delta \geq 1$ and $\delta \in (0, 1)$, we define the good event $\mathcal{E} \coloneq \mathcal{E}_V \cap \mathcal{E}_w \cap \mathcal{E}_M$, where
    \begin{align}
        \mathcal{E}_V &\coloneq \left\{ \forall (k, h) \in [K] \times [H]: \left\| \sum_{\tau_i \in \mathcal{D}_k} \bm{\phi}_{i,h} \left[ \widehat{V}_{k,h+1}(s_{i,h+1}) - \mathbb{P}_h \widehat{V}_{k,h+1}(s_{i,h}, a_{i,h}) \right] \right\|_{\bm{\Lambda}_{k,h}^{-1}} \leq C_\delta H \sqrt{d_\phi} \right\}, \\
        \mathcal{E}_w &\coloneq \left\{ \forall (k, h, m) \in [K] \times [H] \times [M]: \left\| \bm{\Sigma}_{k,h}^{-1/2} \left( \bm{w}^{(m)}_{k,h} - \bm{\mu}^{(m)}_{k,h} \right) \right\|_2 \leq \gamma_\delta \coloneq \sqrt{d_\phi} + \sqrt{2 \log \left( \frac{3 K H M}{\delta} \right)} \right\}, \\
        \mathcal{E}_M &\coloneq \left\{ \forall (k, h) \in [K] \times [H], \forall \bm{u} \in \mathcal{U}, \exists m \in [M]: \bm{u}^\top \bm{\Sigma}_{k,h}^{-1/2} \left( \bm{w}^{(m)}_{k,h} - \bm{\mu}^{(m)}_{k,h} \right) \geq 1 \right\},
    \end{align}
    $\bm{\mu}^{(m)}_{k,h}$ and $\bm{\Sigma}_{k,h}$ are defined in~\cref{lem:gaussian_w}, and $\mathcal{U} \subseteq \mathbb{S}^{d_\phi - 1}$ is a $(1 / (2 \gamma_\delta))$-net of the unit sphere. That is, for every $\bm{u} \in \mathbb{S}^{d_\phi - 1}$, there exists $\bm{u}^\prime \in \mathcal{U}$ such that $\| \bm{u} - \bm{u}^\prime \|_2 \leq 1 / (2 \gamma_\delta)$.
\end{definition}
The event $\mathcal{E}_V$ states that the estimated value functions concentrate, which is standard in the analysis of linear MDPs~\citep{jin2020provably}. The event $\mathcal{E}_w$ states that the Gaussian perturbations injected by $\LMC$, normalized by their conditional covariance, are bounded. The event $\mathcal{E}_M$ states that, for every direction in the net $\mathcal{U}$, at least one of the $M$ critic samples deviates from its mean by at least one standard deviation in that direction, which ensures optimism at all the state-action pairs simultaneously. The following lemma shows that the good event holds with high probability.
\begin{lemma}[Probability of the Good Event]
    \label{lem:good_event}
    Let $\mathcal{E} = \mathcal{E}_V \cap \mathcal{E}_w \cap \mathcal{E}_M$ be the good event in~\cref{def:good_event}. It holds that $\Pr(\mathcal{E}_w) \geq 1 - \delta / 3$, and if $M$ is chosen as in~\cref{lem:lmc_optimism_and_error_bound}, it holds that $\Pr(\mathcal{E}_M) \geq 1 - \delta / 3$. Moreover, for a suitable choice of $C_\delta$, it holds that $\Pr(\mathcal{E}_w \setminus \mathcal{E}_V) \leq \delta / 3$, and hence $\Pr(\mathcal{E}) \geq 1 - \delta$.
\end{lemma}
The first two statements are proved in~\cref{lem:gaussian_event,lem:anti_concentration_event}, respectively. The suitable choice of $C_\delta$ varies between the on-policy and the off-policy settings. We prove the third statement in~\cref{lem:on_policy_self_normalized_concentration} for the on-policy setting (\cref{sec:on_policy_sample_complexity}) and in~\cref{lem:off_policy_self_normalized_concentration} for the off-policy setting (\cref{sec:off_policy_sample_complexity}).

Next, we show that, on the event $\mathcal{E}_w$, the critic samples generated by $\LMC$ remain bounded throughout the learning process.

\begin{lemma}[Bounded Critic Parameters]
    \label{lem:bounded_critic_parameters}
    Consider~\cref{alg:optimistic_actor_critic} with the $\LMC$ critic from~\cref{alg:lmc}, and let $\mathcal{E}_w$ be the event in~\cref{def:good_event}. For some $C_\delta \geq 1$, suppose that we choose $\lambda = 1$, $\eta^{\cri}_{k,h} = 1/ \left( 2 \lambda_{\max}(\bm{\Lambda}_{k,h}) \right)$, and $\beta = ( 12 C_\delta H \sqrt{d_\phi} + 2 )^{-2}$. Let $D_{\max} \coloneq \max_{k \in [K]} \left\lvert \mathcal{D}_{k} \right\rvert$ be the maximum number of trajectories used by the critic, i.e., $D_{\max} = N$ in the on-policy setting and $D_{\max} = K$ in the off-policy setting. Suppose further that, for every $k \in [K]$, the number of $\LMC$ updates in episode $k$ satisfies ${J_k \geq 2 \kappa_k \log ( 3 \sqrt{\lvert \mathcal{D}_{k} \rvert + 1} ( \overline{R}_w + 2 H \sqrt{d_\phi D_{\max}} ) )}$, where $\kappa_k \coloneq \max_{h \in [H]} \lambda_{\max} (\bm{\Lambda}_{k,h}) / \lambda_{\min} (\bm{\Lambda}_{k,h})$ is the maximum condition number of $\bm{\Lambda}_{k,h}$ over $h \in [H]$, and $\overline{R}_w$ is defined in~\cref{eq:w_bar} below. Then, on the event $\mathcal{E}_w$, for all $(k, h, m) \in [K] \times [H] \times [M]$, it holds that
    \begin{align}
        \left\| \bm{w}^{(m)}_{k,h} \right\|_2 \leq \overline{R}_w \coloneq \frac{1}{3} + 2 H \sqrt{d_\phi D_{\max}} + \sqrt{\frac{2}{3 \beta}} \gamma_\delta, \label{eq:w_bar}
    \end{align}
    where $\gamma_\delta$ is defined in~\cref{def:good_event}.
\end{lemma}
Here, $2 H \sqrt{d_\phi D_{\max}}$ bounds the norm of the minimizer $\widehat{\bm{w}}_{k,h}$ (\cref{lem:w_hat_bound}), and $\overline{R}_w$ bounds the norm of the critic samples. We remark that the choice of $J_k$ is valid, since $\kappa_k$ and $\left\lvert \mathcal{D}_{k} \right\rvert$ are available when the critic is updated in episode $k$, and $\overline{R}_w$ and $D_{\max}$ only depend on the hyperparameters.

Next, on the good event $\mathcal{E}$, we present a formal version of~\cref{lem:informal_lmc_optimism_and_error_bound}, which provides an error bound and an optimism guarantee for the Bellman error induced by the $\LMC$ critic.

\begin{lemma}[Formal version of~\cref{lem:informal_lmc_optimism_and_error_bound}]
    \label{lem:lmc_optimism_and_error_bound}
    Consider~\cref{alg:optimistic_actor_critic} with the $\LMC$ critic from~\cref{alg:lmc}, and let $\mathcal{E}$ be the good event in~\cref{def:good_event} for some $C_\delta \geq 1$. Under the hyperparameter choices of~\cref{lem:bounded_critic_parameters}, suppose further that
    \begin{align}
        M \geq \frac{d_\phi \log (1 + 4 \gamma_\delta) + \log \left( 3 K H / \delta \right)}{\log \left( 1/(1-c) \right)},
    \end{align}
    where $c \coloneq 1 / (2 \sqrt{2 e \pi})$ is an absolute constant that satisfies $\Pr(Z > 1) \geq c$ for a standard Gaussian random variable $Z$ (\cref{lem:gaussian_abramowitz}). Then, on the event $\mathcal{E}$, for all $(k, h, s, a) \in [K] \times [H] \times \mathcal{S} \times \mathcal{A}$, it holds that
    \begin{align}
        - \Gamma_{\LMC} \left\| \bm{\phi}(s, a) \right\|_{\bm{\Lambda}_{k,h}^{-1}} \leq \xi_{k,h}(s, a) \leq 0, \label{eq:lmc_optimism_and_error_bound}
    \end{align}
    where $\Gamma_{\LMC} \coloneq 3 C_\delta H \sqrt{d_\phi} + 1/3 + \sqrt{2 / (3 \beta)} \gamma_\delta$. Consequently, by~\cref{lem:good_event}, the above inequalities hold with probability at least $1 - \delta$.
\end{lemma}
We now comment on the orders of $M$ and $\Gamma_{\LMC}$. Since $\gamma_\delta$ depends on $M$ only through $\log M$, the condition on $M$ is satisfied by some $M = \mathcal{O}(d_\phi \log(d_\phi K H / \delta)) = \widetilde{\mathcal{O}}(d_\phi)$. Next, since $C_\delta, H, d_\phi \geq 1$, we have $12 C_\delta H \sqrt{d_\phi} + 2 \leq 14 C_\delta H \sqrt{d_\phi}$, and hence $\sqrt{2 / (3 \beta)} = \sqrt{2/3} (12 C_\delta H \sqrt{d_\phi} + 2) \leq 12 C_\delta H \sqrt{d_\phi}$. Moreover, since $\log(3KHM/\delta) \geq \log 3 \geq 1$, we have $\gamma_\delta \leq (1 + \sqrt{2}) \sqrt{d_\phi \log(3KHM / \delta)}$. Therefore,
\begin{align}
    \Gamma_{\LMC} &= 3 C_\delta H \sqrt{d_\phi} + \frac{1}{3} + \sqrt{\frac{2}{3 \beta}} \gamma_\delta \\
    &\leq 3 C_\delta H \sqrt{d_\phi} + \frac{1}{3} + 12 (1 + \sqrt{2}) C_\delta H d_\phi \sqrt{\log(3KHM/\delta)} \\
    &\leq \mathcal{O} \left( C_\delta H d_\phi \sqrt{\log(3KHM/\delta)} \right).
\end{align}
The equality uses the definition of $\Gamma_{\LMC}$, and the first inequality plugs in the above two bounds.

\subsubsection{Preliminary Properties}
In this section, we state the supporting lemmas that will be helpful in proving the above result. Their proofs are deferred to~\cref{subsec:proofs_of_preliminary_properties}. We first show that, conditioned on the history, the critic parameters sampled by $\LMC$ follow a Gaussian distribution (see also~\citealp[Proposition B.1]{ishfaq2024provable} for a related result). Since the $\LMC$ critic is randomized, the critic samples generated in episode $k$ depend on the collected data as well as on the critic samples generated in the previous episodes. To handle this dependence, for every $(k, h) \in [K] \times [H]$, we denote by $\mathcal{F}_{k,h}$ the $\sigma$-field generated by (i) all the trajectories collected up to episode $k$ (i.e., up to and including $\mathcal{D}_k$), (ii) all the noise vectors $\bm{\nu}^{(m,j)}_{k^\prime,h^\prime}$ used by the critic in the episodes $k^\prime < k$, and (iii) the noise vectors $\bm{\nu}^{(m,j)}_{k,h^\prime}$ used in episode $k$ at the steps $h^\prime > h$. Since the critic proceeds backwards from $h = H$ to $1$, the quantities $\pi_k$, $\bm{\Lambda}_{k,h}$, $\widehat{V}_{k,h+1}$, $\widehat{\bm{w}}_{k,h}$, $\mathcal{T}_{k,h}$, $\eta^{\cri}_{k,h}$, $J_k$, and $\{\bm{w}^{(m)}_{k-1,h}\}_{m \in [M]}$ are all $\mathcal{F}_{k,h}$-measurable. In contrast, the noise vectors $\{\bm{\nu}^{(m,j)}_{k,h}\}_{(m, j) \in [M] \times [J_k]}$ used at step $h$ of episode $k$ are independent of $\mathcal{F}_{k,h}$.
\begin{lemma}
    \label{lem:gaussian_w}
    Consider~\cref{alg:optimistic_actor_critic} with the $\LMC$ critic from~\cref{alg:lmc} and $\eta^{\cri}_{k,h} = 1/ \left( 2 \lambda_{\max}(\bm{\Lambda}_{k,h}) \right)$. For any $(k, h) \in [K] \times [H]$, we denote $\bm{A}_{k,h} \coloneq I - \eta^{\cri}_{k,h} \bm{\Lambda}_{k,h}$. Then, conditioned on $\mathcal{F}_{k,h}$, the samples $\{\bm{w}_{k,h}^{(m)}\}_{m \in [M]}$ are independent, and $\bm{w}_{k,h}^{(m)} \sim \mathcal{N} ( \bm{\mu}^{(m)}_{k,h}, \bm{\Sigma}_{k,h} )$, where
    \begin{align}
        \bm{\mu}^{(m)}_{k,h} \coloneq \bm{A}_{k,h}^{J_k} \bm{w}_{k-1,h}^{(m)} + \left( I - \bm{A}_{k,h}^{J_k} \right) \widehat{\bm{w}}_{k,h}, \qquad
        \bm{\Sigma}_{k,h} \coloneq \frac{1}{\beta} \left( I - \bm{A}_{k,h}^{2 J_k} \right) \bm{\Lambda}_{k,h}^{-1} \left( I + \bm{A}_{k,h} \right)^{-1}. \label{eq:w_variance}
    \end{align}
\end{lemma}

Since $\bm{w}_{k,h}^{(m)}$ is conditionally Gaussian, conditioned on $\mathcal{F}_{k,h}$, the scalar $\bm{\phi}(s, a)^\top \bm{w}_{k,h}^{(m)}$ follows the Gaussian distribution with mean $\bm{\phi}(s, a)^\top \bm{\mu}^{(m)}_{k,h}$ and variance $\left\| \bm{\phi}(s, a) \right\|^2_{\bm{\Sigma}_{k,h}}$. Therefore, we introduce the following lemmas to bound the terms related to the covariance and the mean.
\begin{lemma}
    \label{lem:phi_sigma_bound}
    Under the hyperparameter choices of~\cref{lem:bounded_critic_parameters}, for every $(k, h) \in [K] \times [H]$, it holds that
    \begin{align}
        \frac{1}{4 \beta} \bm{\Lambda}_{k,h}^{-1} \preceq \bm{\Sigma}_{k,h} \preceq \frac{2}{3 \beta} \bm{\Lambda}_{k,h}^{-1}.
    \end{align}
    Consequently, for any $(s, a) \in \mathcal{S} \times \mathcal{A}$,
    \begin{align}
        \frac{1}{2 \sqrt{\beta}} \left\| \bm{\phi}(s, a) \right\|_{\bm{\Lambda}_{k,h}^{-1}} \leq \left\| \bm{\phi}(s, a) \right\|_{\bm{\Sigma}_{k,h}} \leq \sqrt{\frac{2}{3 \beta}} \left\| \bm{\phi}(s, a) \right\|_{\bm{\Lambda}_{k,h}^{-1}}.
    \end{align}
\end{lemma}

\begin{lemma}
    \label{lem:phi_mu_minus_w_hat}
    Under the hyperparameter choices of~\cref{lem:bounded_critic_parameters}, consider any $(k, h, m) \in [K] \times [H] \times [M]$ such that $\left\| \bm{w}^{(m)}_{k-1,h} \right\|_2 \leq \overline{R}_w$. Then, for any $(s, a) \in \mathcal{S} \times \mathcal{A}$,
    \begin{align}
        \left\lvert \bm{\phi}(s, a)^\top \left( \bm{\mu}^{(m)}_{k,h}-\widehat{\bm{w}}_{k,h} \right) \right\rvert \leq \frac{1}{3} \left\| \bm{\phi}(s, a) \right\|_{\bm{\Lambda}_{k,h}^{-1}}.
    \end{align}
    Moreover, it holds that $\left\| \bm{\mu}^{(m)}_{k,h} \right\|_2 \leq \frac{1}{3} + 2 H \sqrt{d_\phi D_{\max}}$.
\end{lemma}

Additionally, we outline the necessary supporting lemmas that are useful for bounding the Bellman error. Recall that $\left\lvert \mathcal{D}_{k} \right\rvert$ denotes the number of trajectories in $\mathcal{D}_{k}$.
\begin{lemma}
    \label{lem:w_hat_bound}
    Consider~\cref{alg:optimistic_actor_critic} with any critic that ensures $\widehat{V}_{k,h+1} \in [0, H]$. For any $(k, h) \in [K]\times[H]$, it holds that
    \begin{align}
        \left\| \widehat{\bm{w}}_{k,h} \right\|_2 \leq 2 H \sqrt{ d_\phi \left\lvert \mathcal{D}_{k} \right\rvert / \lambda}.
    \end{align}
    In particular, if $\lambda = 1$, it holds that $\left\| \widehat{\bm{w}}_{k,h} \right\|_2 \leq 2 H \sqrt{d_\phi D_{\max}}$, where $D_{\max} \coloneq \max_{k \in [K]} \lvert \mathcal{D}_{k} \rvert$ is the maximum number of trajectories used by the critic.
\end{lemma}

\begin{lemma}
    \label{lem:gaussian_event}
    Consider~\cref{alg:optimistic_actor_critic} with the $\LMC$ critic from~\cref{alg:lmc} and $\eta^{\cri}_{k,h} = 1/ \left( 2 \lambda_{\max}(\bm{\Lambda}_{k,h}) \right)$. For any $\delta \in (0, 1)$, it holds that $\Pr(\mathcal{E}_w) \geq 1 - \delta / 3$.
\end{lemma}

\begin{lemma}
    \label{lem:anti_concentration_event}
    Consider~\cref{alg:optimistic_actor_critic} with the $\LMC$ critic from~\cref{alg:lmc} and $\eta^{\cri}_{k,h} = 1/ \left( 2 \lambda_{\max}(\bm{\Lambda}_{k,h}) \right)$. For any $\delta \in (0, 1)$, if $M$ is chosen as in~\cref{lem:lmc_optimism_and_error_bound}, it holds that $\Pr(\mathcal{E}_M) \geq 1 - \delta / 3$.
\end{lemma}

\begin{lemma}
    \label{lem:iota_for_w_hat}
    Consider~\cref{alg:optimistic_actor_critic} with any critic that ensures $\widehat{V}_{k,h+1} \in [0, H]$, and suppose that $\lambda = 1$ and $C_\delta \geq 1$. Consider any $(k, h) \in [K] \times [H]$ such that the inequality in the definition of $\mathcal{E}_V$ holds for $(k, h)$. Then, for any $(s, a) \in \mathcal{S} \times \mathcal{A}$, it holds that
    \begin{align}
         \left\lvert \bm{\phi}(s, a)^\top \widehat{\bm{w}}_{k,h} - \mathcal{T}_{k,h}(s, a) \right\rvert \leq 3 C_\delta H \sqrt{d_\phi} \left\| \bm{\phi}(s, a) \right\|_{\bm{\Lambda}_{k,h}^{-1}}.
    \end{align}
\end{lemma}

\subsubsection{Main Analysis}
We will use the above lemmas to complete the main proof in this section.
\begin{proof}[Proof of \cref{lem:bounded_critic_parameters}]
    We proceed by induction over $k$. For $k = 0$, we have $\bm{w}^{(m)}_{0,h} = \bm{0}$ for all $(h, m)$ due to the initialization in~\cref{alg:optimistic_actor_critic}, and hence the statement holds. Now, suppose that $\| \bm{w}^{(m)}_{k-1,h} \|_2 \leq \overline{R}_w$ for all $(h, m) \in [H] \times [M]$, and consider any $(h, m) \in [H] \times [M]$. Then, we have that
    \begin{align}
        \left\| \bm{w}^{(m)}_{k,h} \right\|_2 &\leq \left\| \bm{\mu}^{(m)}_{k,h} \right\|_2 + \left\| \bm{w}^{(m)}_{k,h} - \bm{\mu}^{(m)}_{k,h} \right\|_2 \\
        &= \left\| \bm{\mu}^{(m)}_{k,h} \right\|_2 + \left\| \bm{\Sigma}_{k,h}^{1/2} \bm{\Sigma}_{k,h}^{-1/2} \left( \bm{w}^{(m)}_{k,h} - \bm{\mu}^{(m)}_{k,h} \right) \right\|_2 \\
        &\leq \left\| \bm{\mu}^{(m)}_{k,h} \right\|_2 + \left\| \bm{\Sigma}_{k,h} \right\|_2^{1/2} \left\| \bm{\Sigma}_{k,h}^{-1/2} \left( \bm{w}^{(m)}_{k,h} - \bm{\mu}^{(m)}_{k,h} \right) \right\|_2 \\
        &\leq \frac{1}{3} + 2 H \sqrt{d_\phi D_{\max}} + \left\| \bm{\Sigma}_{k,h} \right\|_2^{1/2} \gamma_\delta \\
        &\leq \frac{1}{3} + 2 H \sqrt{d_\phi D_{\max}} + \sqrt{\frac{2}{3 \beta}} \gamma_\delta = \overline{R}_w.
    \end{align}
    The first inequality uses the triangle inequality. The first equality uses $\bm{\Sigma}_{k,h}^{1/2} \bm{\Sigma}_{k,h}^{-1/2} = I$. The second inequality uses the definition of the matrix norm and $\| \bm{\Sigma}_{k,h}^{1/2} \|_2 = \| \bm{\Sigma}_{k,h} \|_2^{1/2}$. The third inequality uses~\cref{lem:phi_mu_minus_w_hat}, which applies due to the induction hypothesis, and the event $\mathcal{E}_w$. The last inequality follows from~\cref{lem:phi_sigma_bound}, which gives $\bm{\Sigma}_{k,h} \preceq \frac{2}{3 \beta} \bm{\Lambda}_{k,h}^{-1} \preceq \frac{2}{3 \beta} I$ since $\bm{\Lambda}_{k,h} \succeq \lambda I = I$. This completes the induction, which concludes the proof.
\end{proof}

\begin{proof}[Proof of \cref{lem:lmc_optimism_and_error_bound}]
    We prove the two inequalities separately. Throughout the proof, we use the shorthand $\bm{\phi} \coloneq \bm{\phi}(s, a)$ whenever the state-action pair $(s, a)$ is clear from the context. We also note that $\bm{\phi}(s, a) \neq \bm{0}$ for every $(s, a)$, since $\bm{\phi}(s, a)^\top \int_{\mathcal{S}} \bm{\psi}_h(s^\prime) \, \mathrm{d} s^\prime = 1$ by~\cref{def:linear_mdp}.

    \textbf{Error Bound (LHS)} \,
    We consider a fixed $(k, h, s, a) \in [K] \times [H] \times \mathcal{S} \times \mathcal{A}$, and denote $y \coloneq \max_{m \in [M]} \bm{\phi}^\top \bm{w}_{k,h}^{(m)}$. By the definition in~\cref{alg:lmc}, $\widehat{Q}_{k,h}(s, a) = \Clip_{[0, H-h+1]}(y)$. We consider two cases. If $y < 0$, then $\widehat{Q}_{k,h}(s, a) = 0$, and hence $- \xi_{k,h}(s, a) = - \mathcal{T}_{k,h}(s, a) \leq 0$, so the statement trivially holds. Otherwise, if $y \geq 0$, the clipping can only decrease the value, i.e., $\widehat{Q}_{k,h}(s, a) \leq y$. Therefore,
    \begin{align}
        \MoveEqLeft
        - \xi_{k,h}(s, a) = \widehat{Q}_{k,h}(s, a) - \mathcal{T}_{k,h}(s, a) \leq y - \mathcal{T}_{k,h}(s, a) = \max_{m \in [M]} \left[ \bm{\phi}^\top \bm{w}_{k,h}^{(m)} - \mathcal{T}_{k,h}(s, a) \right] \\
        &= \max_{m \in [M]} \left[ \bm{\phi}^\top \left( \bm{w}_{k,h}^{(m)} - \bm{\mu}^{(m)}_{k,h} \right) + \bm{\phi}^\top \left( \bm{\mu}^{(m)}_{k,h} - \widehat{\bm{w}}_{k,h} \right) + \bm{\phi}^\top \widehat{\bm{w}}_{k,h} - \mathcal{T}_{k,h}(s, a) \right] \\
        &\leq \underbrace{\max_{m \in [M]} \left\lvert \bm{\phi}^\top \left( \bm{w}_{k,h}^{(m)} - \bm{\mu}^{(m)}_{k,h} \right) \right\rvert}_{\text{(I)}} + \underbrace{\max_{m \in [M]} \left\lvert \bm{\phi}^\top \left( \bm{\mu}^{(m)}_{k,h} - \widehat{\bm{w}}_{k,h} \right) \right\rvert}_{\text{(II)}} + \underbrace{\left\lvert \bm{\phi}^\top \widehat{\bm{w}}_{k,h} - \mathcal{T}_{k,h}(s, a) \right\rvert}_{\text{(III)}}. \label{eq:lmc_error_decomposition}
    \end{align}
    The last equality adds and subtracts $\bm{\phi}^\top \bm{\mu}^{(m)}_{k,h}$ and $\bm{\phi}^\top \widehat{\bm{w}}_{k,h}$. The last inequality uses the triangle inequality and the fact that the maximum of a sum is at most the sum of the maxima.

    We first start by bounding Term (I) of~\cref{eq:lmc_error_decomposition}. For any $m \in [M]$, we have that
    \begin{align}
        \left\lvert \bm{\phi}^\top \left( \bm{w}_{k,h}^{(m)} - \bm{\mu}^{(m)}_{k,h} \right) \right\rvert &= \left\lvert \left( \bm{\Sigma}_{k,h}^{1/2} \bm{\phi} \right)^\top \bm{\Sigma}_{k,h}^{-1/2} \left( \bm{w}^{(m)}_{k,h} - \bm{\mu}^{(m)}_{k,h} \right) \right\rvert \\
        &\leq \left\| \bm{\Sigma}_{k,h}^{1/2} \bm{\phi} \right\|_2 \left\| \bm{\Sigma}_{k,h}^{-1/2} \left( \bm{w}^{(m)}_{k,h} - \bm{\mu}^{(m)}_{k,h} \right) \right\|_2 \\
        &\leq \left\| \bm{\phi} \right\|_{\bm{\Sigma}_{k,h}} \gamma_\delta \leq \sqrt{\frac{2}{3 \beta}} \gamma_\delta \left\| \bm{\phi} \right\|_{\bm{\Lambda}_{k,h}^{-1}}.
    \end{align}
    The equality uses $\bm{\Sigma}_{k,h}^{1/2} \bm{\Sigma}_{k,h}^{-1/2} = I$ and the symmetry of $\bm{\Sigma}_{k,h}^{1/2}$. The first inequality applies the Cauchy–Schwarz inequality. The second inequality uses $\| \bm{\Sigma}_{k,h}^{1/2} \bm{\phi} \|_2 = \| \bm{\phi} \|_{\bm{\Sigma}_{k,h}}$ and the event $\mathcal{E}_w$. The last inequality follows from~\cref{lem:phi_sigma_bound}.

    Next, we bound Term (II) of~\cref{eq:lmc_error_decomposition}. Since $\mathcal{E} \subseteq \mathcal{E}_w$,~\cref{lem:bounded_critic_parameters} gives $\| \bm{w}^{(m)}_{k-1,h} \|_2 \leq \overline{R}_w$ for all $m \in [M]$, and hence~\cref{lem:phi_mu_minus_w_hat} gives $\text{Term (II)} \leq \frac{1}{3} \left\| \bm{\phi} \right\|_{\bm{\Lambda}_{k,h}^{-1}}$. Lastly, for Term (III) of~\cref{eq:lmc_error_decomposition}, on the event $\mathcal{E}_V$,~\cref{lem:iota_for_w_hat} gives $\text{Term (III)} \leq 3 C_\delta H \sqrt{d_\phi} \left\| \bm{\phi} \right\|_{\bm{\Lambda}_{k,h}^{-1}}$. Putting everything together, we have that
    \begin{align}
        - \xi_{k,h}(s, a) \leq \left( \sqrt{\frac{2}{3 \beta}} \gamma_\delta + \frac{1}{3} + 3 C_\delta H \sqrt{d_\phi} \right) \left\| \bm{\phi}(s, a) \right\|_{\bm{\Lambda}_{k,h}^{-1}} = \Gamma_{\LMC} \left\| \bm{\phi}(s, a) \right\|_{\bm{\Lambda}_{k,h}^{-1}},
    \end{align}
    which proves the LHS of~\cref{eq:lmc_optimism_and_error_bound}.

    \textbf{Optimism (RHS)} \,
    We consider a fixed $(k, h, s, a) \in [K] \times [H] \times \mathcal{S} \times \mathcal{A}$. Since $\bm{\Sigma}_{k,h} \succ 0$ and $\bm{\phi} \neq \bm{0}$, we have $\left\| \bm{\phi} \right\|_{\bm{\Sigma}_{k,h}} > 0$, and we can define the unit vector
    \begin{align}
        \bm{u} \coloneq \frac{\bm{\Sigma}_{k,h}^{1/2} \bm{\phi}}{\left\| \bm{\phi} \right\|_{\bm{\Sigma}_{k,h}}} \in \mathbb{S}^{d_\phi - 1}.
    \end{align}
    By the definition of the net $\mathcal{U}$, there exists $\bm{u}^\prime \in \mathcal{U}$ such that $\| \bm{u} - \bm{u}^\prime \|_2 \leq 1 / (2 \gamma_\delta)$. On the event $\mathcal{E}_M$, there further exists $m^\star \in [M]$ such that $(\bm{u}^\prime)^\top \bm{\Sigma}_{k,h}^{-1/2} \left( \bm{w}^{(m^\star)}_{k,h} - \bm{\mu}^{(m^\star)}_{k,h} \right) \geq 1$. For this sample, we first bound the Gaussian perturbation in the direction of $\bm{\phi}$ as follows.
    \begin{align}
        \bm{\phi}^\top \left( \bm{w}_{k,h}^{(m^\star)} - \bm{\mu}^{(m^\star)}_{k,h} \right) &= \left\| \bm{\phi} \right\|_{\bm{\Sigma}_{k,h}} \bm{u}^\top \bm{\Sigma}_{k,h}^{-1/2} \left( \bm{w}^{(m^\star)}_{k,h} - \bm{\mu}^{(m^\star)}_{k,h} \right) \\
        &= \left\| \bm{\phi} \right\|_{\bm{\Sigma}_{k,h}} \left[ (\bm{u}^\prime)^\top \bm{\Sigma}_{k,h}^{-1/2} \left( \bm{w}^{(m^\star)}_{k,h} - \bm{\mu}^{(m^\star)}_{k,h} \right) + \left( \bm{u} - \bm{u}^\prime \right)^\top \bm{\Sigma}_{k,h}^{-1/2} \left( \bm{w}^{(m^\star)}_{k,h} - \bm{\mu}^{(m^\star)}_{k,h} \right) \right] \\
        &\geq \left\| \bm{\phi} \right\|_{\bm{\Sigma}_{k,h}} \left[ (\bm{u}^\prime)^\top \bm{\Sigma}_{k,h}^{-1/2} \left( \bm{w}^{(m^\star)}_{k,h} - \bm{\mu}^{(m^\star)}_{k,h} \right) - \left\| \bm{u} - \bm{u}^\prime \right\|_2 \left\| \bm{\Sigma}_{k,h}^{-1/2} \left( \bm{w}^{(m^\star)}_{k,h} - \bm{\mu}^{(m^\star)}_{k,h} \right) \right\|_2 \right] \\
        &\geq \left\| \bm{\phi} \right\|_{\bm{\Sigma}_{k,h}} \left( 1 - \frac{1}{2 \gamma_\delta} \cdot \gamma_\delta \right) = \frac{1}{2} \left\| \bm{\phi} \right\|_{\bm{\Sigma}_{k,h}}.
    \end{align}
    The first equality uses $\bm{\Sigma}_{k,h}^{1/2} \bm{\Sigma}_{k,h}^{-1/2} = I$, the symmetry of $\bm{\Sigma}_{k,h}^{1/2}$, and the definition of $\bm{u}$. The second equality adds and subtracts $\bm{u}^\prime$. The first inequality applies the Cauchy–Schwarz inequality to the second term. The second inequality uses the choice of $m^\star$, $\| \bm{u} - \bm{u}^\prime \|_2 \leq 1 / (2 \gamma_\delta)$, and the event $\mathcal{E}_w$. Therefore, we have that
    \begin{align}
        \bm{\phi}^\top \bm{w}_{k,h}^{(m^\star)} - \mathcal{T}_{k,h}(s, a) &= \bm{\phi}^\top \left( \bm{w}_{k,h}^{(m^\star)} - \bm{\mu}^{(m^\star)}_{k,h} \right) + \bm{\phi}^\top \left( \bm{\mu}^{(m^\star)}_{k,h} - \widehat{\bm{w}}_{k,h} \right) + \bm{\phi}^\top \widehat{\bm{w}}_{k,h} - \mathcal{T}_{k,h}(s, a) \\
        &\geq \frac{1}{2} \left\| \bm{\phi} \right\|_{\bm{\Sigma}_{k,h}} - \left( 3 C_\delta H \sqrt{d_\phi} + \frac{1}{3} \right) \left\| \bm{\phi} \right\|_{\bm{\Lambda}_{k,h}^{-1}} \\
        &\geq \left( \frac{1}{4 \sqrt{\beta}} - 3 C_\delta H \sqrt{d_\phi} - \frac{1}{3} \right) \left\| \bm{\phi} \right\|_{\bm{\Lambda}_{k,h}^{-1}} \\
        &= \frac{1}{6} \left\| \bm{\phi} \right\|_{\bm{\Lambda}_{k,h}^{-1}} \geq 0.
    \end{align}
    The first inequality uses the previous display, and bounds the last two terms using~\cref{lem:phi_mu_minus_w_hat}, which applies due to~\cref{lem:bounded_critic_parameters}, and~\cref{lem:iota_for_w_hat}, which applies on the event $\mathcal{E}_V$. The second inequality uses the lower bound in~\cref{lem:phi_sigma_bound}. The second equality follows from the choice of $\beta$, which gives $1 / (4 \sqrt{\beta}) = 3 C_\delta H \sqrt{d_\phi} + 1/2$.

    Therefore, $y = \max_{m \in [M]} \bm{\phi}^\top \bm{w}_{k,h}^{(m)} \geq \mathcal{T}_{k,h}(s, a) \geq 0$. Since $\mathcal{T}_{k,h}(s, a) \leq H - h + 1$, we have that $\widehat{Q}_{k,h}(s, a) = \min \{ y, H - h + 1 \} \geq \mathcal{T}_{k,h}(s, a)$, i.e., $\xi_{k,h}(s, a) \leq 0$. Since $(k, h, s, a)$ is arbitrary, this proves the RHS of~\cref{eq:lmc_optimism_and_error_bound}.

    Lastly, by~\cref{lem:good_event}, the event $\mathcal{E}$ holds with probability at least $1 - \delta$, which concludes the proof.
\end{proof}

\subsection{Proofs of Preliminary Properties}
\label{subsec:proofs_of_preliminary_properties}

\subsubsection{Proof of \texorpdfstring{\cref{lem:gaussian_w}}{}}
\begin{proof}
    We consider a fixed $(k, h, m) \in [K] \times [H] \times [M]$. To simplify the notation, we write $\bm{A} \coloneq \bm{A}_{k,h}$, $\eta^{\cri} \coloneq \eta^{\cri}_{k,h}$, and $J \coloneq J_k$. For any $j \in [J]$, the critic update rule at the $j$-th round can be written as
    \begin{align}
        \bm{w}_{k,h}^{(m,j)} &= \bm{w}_{k,h}^{(m,j-1)} - \eta^{\cri} \nabla \ell^{\cri}_{k,h}(\bm{w}_{k,h}^{(m,j-1)}) + \sqrt{\eta^{\cri} / \beta} \bm{\nu}_{k,h}^{(m,j)} \\
        &= \bm{w}_{k,h}^{(m,j-1)} - \eta^{\cri} \left( \bm{\Lambda}_{k,h} \bm{w}_{k,h}^{(m,j-1)} - \bm{b}_{k,h} \right) + \sqrt{\eta^{\cri} / \beta} \bm{\nu}_{k,h}^{(m,j)} \\
        &= \left( I - \eta^{\cri} \bm{\Lambda}_{k,h} \right) \bm{w}_{k,h}^{(m,j-1)} + \eta^{\cri} \bm{\Lambda}_{k,h} \widehat{\bm{w}}_{k,h} + \sqrt{\eta^{\cri} / \beta} \bm{\nu}_{k,h}^{(m,j)} \\
        &= \bm{A} \bm{w}_{k,h}^{(m,j-1)} + \left( I - \bm{A} \right) \widehat{\bm{w}}_{k,h} + \sqrt{\eta^{\cri} / \beta} \bm{\nu}_{k,h}^{(m,j)}.
    \end{align}
    The second equality plugs in the gradient in~\cref{eq:derivative_of_critic_loss}. The third equality uses $\bm{b}_{k,h} = \bm{\Lambda}_{k,h} \widehat{\bm{w}}_{k,h}$ from~\cref{eq:w_hat}. The last equality uses the definition $\bm{A} = I - \eta^{\cri} \bm{\Lambda}_{k,h}$. By applying the above equation recursively for $j = J, J-1, \ldots, 1$, we have that
    \begin{align}
        \bm{w}_{k,h}^{(m)} = \bm{w}_{k,h}^{(m,J)} &= \bm{A}^{J} \bm{w}_{k,h}^{(m,0)} + \sum_{l=0}^{J-1} \bm{A}^l \left( I - \bm{A} \right) \widehat{\bm{w}}_{k,h} + \sqrt{\eta^{\cri} / \beta} \sum_{l=0}^{J-1} \bm{A}^l \bm{\nu}_{k,h}^{(m,J-l)} \\
        &= \bm{A}^{J} \bm{w}_{k-1,h}^{(m)} + \left( I - \bm{A}^J \right) \widehat{\bm{w}}_{k,h} + \sqrt{\eta^{\cri} / \beta} \sum_{l=0}^{J-1} \bm{A}^l \bm{\nu}_{k,h}^{(m,J-l)} \\
        &= \bm{\mu}^{(m)}_{k,h} + \sqrt{\eta^{\cri} / \beta} \sum_{l=0}^{J-1} \bm{A}^l \bm{\nu}_{k,h}^{(m,J-l)}.
    \end{align}
    The second equality comes from telescoping the previous equation from $l = 0$ to $J - 1$. The third equality uses the warm start $\bm{w}^{(m,0)}_{k,h} = \bm{w}^{(m)}_{k-1,h}$ in~\cref{alg:lmc} and the identity $\sum_{l=0}^{J-1} \bm{A}^l (I - \bm{A}) = I - \bm{A}^J$. The last equality uses the definition of $\bm{\mu}^{(m)}_{k,h}$. Since we set $\eta^{\cri} = 1 / (2 \lambda_{\max}(\bm{\Lambda}_{k,h}))$, all the eigenvalues of $\eta^{\cri} \bm{\Lambda}_{k,h}$ lie in $(0, 1/2]$, and hence all the eigenvalues of the symmetric matrix $\bm{A}$ lie in $[1/2, 1)$. In particular, $\bm{A}$ is positive definite and $I - \bm{A}$ is invertible.

    Conditioned on $\mathcal{F}_{k,h}$, the mean $\bm{\mu}^{(m)}_{k,h}$ is fixed, whereas the noise vectors $\{\bm{\nu}^{(m,j)}_{k,h}\}_{(m,j) \in [M] \times [J]}$ are i.i.d.\ standard Gaussian vectors that are independent of $\mathcal{F}_{k,h}$. Note that if $\bm{y} \sim \mathcal{N}(\bm{0}, I)$, then $\bm{B} \bm{y} + \bm{v} \sim \mathcal{N}(\bm{v}, \bm{B} \bm{B}^\top)$ for any fixed matrix $\bm{B}$ and vector $\bm{v}$. Therefore, $\bm{w}^{(m)}_{k,h}$ is conditionally Gaussian with mean $\bm{\mu}^{(m)}_{k,h}$. Moreover, since different samples $m \in [M]$ use different noise vectors, the samples are conditionally independent across $m$. We then derive the conditional covariance. Since the noise vectors are independent across $l$, we have that
    \begin{align}
        \frac{\eta^{\cri}}{\beta} \sum_{l=0}^{J-1} \bm{A}^l \left( \bm{A}^l \right)^\top &= \frac{\eta^{\cri}}{\beta} \sum_{l=0}^{J-1} \bm{A}^{2l} = \frac{\eta^{\cri}}{\beta} \left( I - \bm{A}^{2J} \right) \left( I - \bm{A}^2 \right)^{-1} \\
        &= \frac{\eta^{\cri}}{\beta} \left( I - \bm{A}^{2J} \right) \left( I - \bm{A} \right)^{-1} \left( I + \bm{A} \right)^{-1} = \frac{1}{\beta} \left( I - \bm{A}^{2J} \right) \bm{\Lambda}_{k,h}^{-1} \left( I + \bm{A} \right)^{-1} = \bm{\Sigma}_{k,h}.
    \end{align}
    The first equality uses the symmetry of $\bm{A}$. The second equality uses the fact that $I + \bm{B} + \ldots + \bm{B}^{n-1} = (I - \bm{B}^n)(I - \bm{B})^{-1}$ with $\bm{B} = \bm{A}^2$. The third equality uses $I - \bm{A}^2 = (I - \bm{A})(I + \bm{A})$. The fourth equality uses $(I - \bm{A})^{-1} = (\eta^{\cri} \bm{\Lambda}_{k,h})^{-1}$, and the last equality uses the definition of $\bm{\Sigma}_{k,h}$. We remark that all the matrices above are functions of $\bm{\Lambda}_{k,h}$, and hence they commute. This concludes the proof.
\end{proof}

\subsubsection{Proof of \texorpdfstring{\cref{lem:phi_sigma_bound}}{}}
\begin{proof}
    We consider a fixed $(k, h) \in [K] \times [H]$, and write $\bm{A} \coloneq \bm{A}_{k,h}$, $\eta^{\cri} \coloneq \eta^{\cri}_{k,h}$, $J \coloneq J_k$, and $\kappa \coloneq \kappa_k$. Let $\bm{\Lambda}_{k,h} = \sum_{i=1}^{d_\phi} \lambda_i \bm{u}_i \bm{u}_i^\top$ be the eigendecomposition of $\bm{\Lambda}_{k,h}$, where $\lambda_1, \ldots, \lambda_{d_\phi}$ are its eigenvalues and $\bm{u}_1, \ldots, \bm{u}_{d_\phi}$ are the corresponding orthonormal eigenvectors. Since $\bm{A}$, $I - \bm{A}^{2J}$, $(I + \bm{A})^{-1}$, and $\bm{\Lambda}_{k,h}^{-1}$ are all functions of $\bm{\Lambda}_{k,h}$, they share the same eigenvectors. Hence, by~\cref{lem:gaussian_w},
    \begin{align}
        \bm{\Sigma}_{k,h} = \sum_{i=1}^{d_\phi} \sigma_i \bm{u}_i \bm{u}_i^\top,
    \end{align}
    where $\sigma_i \coloneq (1 - a_i^{2J}) / (\beta \lambda_i (1 + a_i))$ and $a_i \coloneq 1 - \eta^{\cri} \lambda_i$. We first bound $a_i$. Since $\eta^{\cri} = 1 / (2 \lambda_{\max}(\bm{\Lambda}_{k,h}))$ and $\lambda_{\min}(\bm{\Lambda}_{k,h}) \leq \lambda_i \leq \lambda_{\max}(\bm{\Lambda}_{k,h})$, we have that
    \begin{align}
        \frac{1}{2 \kappa} \leq \frac{\lambda_{\min}(\bm{\Lambda}_{k,h})}{2 \lambda_{\max}(\bm{\Lambda}_{k,h})} \leq \eta^{\cri} \lambda_i \leq \frac{1}{2} \qquad \Longrightarrow \qquad \frac{1}{2} \leq a_i \leq 1 - \frac{1}{2 \kappa}, \label{eq:a_i_bound}
    \end{align}
    where the first inequality uses $\kappa = \kappa_k \geq \lambda_{\max}(\bm{\Lambda}_{k,h}) / \lambda_{\min}(\bm{\Lambda}_{k,h})$.

    We first bound the RHS. For every $i \in [d_\phi]$, we have that
    \begin{align}
        \sigma_i \leq \frac{1}{\beta} \cdot \frac{1}{\lambda_i (1 + a_i)} \leq \frac{2}{3 \beta} \cdot \frac{1}{\lambda_i},
    \end{align}
    where the first inequality uses $a_i^{2J} \geq 0$, and the second inequality uses $a_i \geq 1/2$ from~\cref{eq:a_i_bound}.

    We then proceed to bound the LHS. We first note that
    \begin{align}
        a_i^{2J} \leq \left( 1 - \frac{1}{2 \kappa} \right)^{2J} \leq \exp \left( - \frac{J}{\kappa} \right) \leq \frac{1}{2}.
    \end{align}
    The first inequality uses~\cref{eq:a_i_bound}. The second inequality uses $1 - y \leq e^{-y}$ for all $y \in \mathbb{R}$. The last inequality holds since $J \geq \kappa \log 2$, which is implied by the choice of $J_k$ in~\cref{lem:bounded_critic_parameters}. Indeed, since $\left\lvert \mathcal{D}_{k} \right\rvert \geq 1$ and $\overline{R}_w \geq 1/3$, we have $3 \sqrt{\left\lvert \mathcal{D}_{k} \right\rvert + 1} (\overline{R}_w + 2 H \sqrt{d_\phi D_{\max}}) \geq \sqrt{2}$, and hence $J \geq 2 \kappa \log \sqrt{2} = \kappa \log 2$. Therefore, for every $i \in [d_\phi]$,
    \begin{align}
        \sigma_i \geq \frac{1}{\beta} \cdot \frac{1/2}{\lambda_i (1 + a_i)} \geq \frac{1}{4 \beta} \cdot \frac{1}{\lambda_i},
    \end{align}
    where the first inequality uses the previous display, and the second inequality uses $a_i \leq 1$. Since $\bm{\Lambda}_{k,h}^{-1} = \sum_{i=1}^{d_\phi} \lambda_i^{-1} \bm{u}_i \bm{u}_i^\top$, the above two bounds imply that $\frac{1}{4 \beta} \bm{\Lambda}_{k,h}^{-1} \preceq \bm{\Sigma}_{k,h} \preceq \frac{2}{3 \beta} \bm{\Lambda}_{k,h}^{-1}$, which proves the first statement. For the second statement, we note that $\left\| \bm{\phi}(s, a) \right\|^2_{\bm{\Sigma}_{k,h}} = \bm{\phi}(s, a)^\top \bm{\Sigma}_{k,h} \bm{\phi}(s, a)$. Evaluating the above matrix inequalities at $\bm{\phi}(s, a)$ and taking the square roots concludes the proof.
\end{proof}

\subsubsection{Proof of \texorpdfstring{\cref{lem:phi_mu_minus_w_hat}}{}}
\begin{proof}
    We consider a fixed $(k, h, m) \in [K] \times [H] \times [M]$ such that $\| \bm{w}^{(m)}_{k-1,h} \|_2 \leq \overline{R}_w$, and write $\bm{A} \coloneq \bm{A}_{k,h}$, $J \coloneq J_k$, and $\kappa \coloneq \kappa_k$. Using the definition of $\bm{\mu}^{(m)}_{k,h}$ in~\cref{lem:gaussian_w}, we first have that
    \begin{align}
        \bm{\mu}^{(m)}_{k,h} - \widehat{\bm{w}}_{k,h} = \bm{A}^J \bm{w}^{(m)}_{k-1,h} + \left( I - \bm{A}^J \right) \widehat{\bm{w}}_{k,h} - \widehat{\bm{w}}_{k,h} = \bm{A}^J \left( \bm{w}^{(m)}_{k-1,h} - \widehat{\bm{w}}_{k,h} \right).
    \end{align}
    This implies that, for any $(s, a) \in \mathcal{S} \times \mathcal{A}$,
    \begin{align}
        \left\lvert \bm{\phi}(s, a)^\top \left( \bm{\mu}^{(m)}_{k,h} - \widehat{\bm{w}}_{k,h} \right) \right\rvert &= \left\lvert \bm{\phi}(s, a)^\top \bm{A}^J \left( \bm{w}^{(m)}_{k-1,h} - \widehat{\bm{w}}_{k,h} \right) \right\rvert \\
        &= \left\lvert \left( \bm{\Lambda}_{k,h}^{-1/2} \bm{\phi}(s, a) \right)^\top \bm{\Lambda}_{k,h}^{1/2} \bm{A}^J \left( \bm{w}^{(m)}_{k-1,h} - \widehat{\bm{w}}_{k,h} \right) \right\rvert \\
        &\leq \left\| \bm{\Lambda}_{k,h}^{-1/2} \bm{\phi}(s, a) \right\|_2 \left\| \bm{\Lambda}_{k,h}^{1/2} \bm{A}^J \left( \bm{w}^{(m)}_{k-1,h} - \widehat{\bm{w}}_{k,h} \right) \right\|_2 \\
        &\leq \left\| \bm{\phi}(s, a) \right\|_{\bm{\Lambda}_{k,h}^{-1}} \left\| \bm{\Lambda}_{k,h}^{1/2} \bm{A}^J \right\|_2 \left( \left\| \bm{w}^{(m)}_{k-1,h} \right\|_2 + \left\| \widehat{\bm{w}}_{k,h} \right\|_2 \right) \\
        &\leq \left\| \bm{\phi}(s, a) \right\|_{\bm{\Lambda}_{k,h}^{-1}} \left\| \bm{\Lambda}_{k,h}^{1/2} \bm{A}^J \right\|_2 \left( \overline{R}_w + 2 H \sqrt{d_\phi D_{\max}} \right).
    \end{align}
    The second equality uses $\bm{\Lambda}_{k,h}^{-1/2} \bm{\Lambda}_{k,h}^{1/2} = I$ and the fact that $\bm{\Lambda}_{k,h}$ and $\bm{A}$ commute. The first inequality applies the Cauchy–Schwarz inequality. The second inequality uses $\| \bm{\Lambda}_{k,h}^{-1/2} \bm{\phi}(s, a) \|_2 = \| \bm{\phi}(s, a) \|_{\bm{\Lambda}_{k,h}^{-1}}$, the definition of the matrix norm, and the triangle inequality. The last inequality uses the assumption $\| \bm{w}^{(m)}_{k-1,h} \|_2 \leq \overline{R}_w$ and~\cref{lem:w_hat_bound}.

    It remains to bound $\| \bm{\Lambda}_{k,h}^{1/2} \bm{A}^J \|_2$. Using the notation in the proof of~\cref{lem:phi_sigma_bound}, the eigenvalues of $\bm{\Lambda}_{k,h}^{1/2} \bm{A}^J$ are $\sqrt{\lambda_i} a_i^J$ for $i \in [d_\phi]$. Therefore,
    \begin{align}
        \MoveEqLeft
        \left\| \bm{\Lambda}_{k,h}^{1/2} \bm{A}^J \right\|_2 = \max_{i \in [d_\phi]} \sqrt{\lambda_i} a_i^J \leq \sqrt{\lambda_{\max}(\bm{\Lambda}_{k,h})} \left( 1 - \frac{1}{2 \kappa} \right)^J \\
        &\leq \sqrt{\left\lvert \mathcal{D}_{k} \right\rvert + 1} \exp \left( - \frac{J}{2 \kappa} \right) \leq \frac{1}{3 \left( \overline{R}_w + 2 H \sqrt{d_\phi D_{\max}} \right)}.
    \end{align}
    The first inequality uses~\cref{eq:a_i_bound}. The second inequality uses $\lambda_{\max}(\bm{\Lambda}_{k,h}) \leq \lambda + \sum_{\tau_i \in \mathcal{D}_k} \| \bm{\phi}_{i,h} \|_2^2 \leq \left\lvert \mathcal{D}_{k} \right\rvert + 1$ and $1 - y \leq e^{-y}$. The last inequality follows from the choice $J_k \geq 2 \kappa_k \log ( 3 \sqrt{\left\lvert \mathcal{D}_{k} \right\rvert + 1} (\overline{R}_w + 2 H \sqrt{d_\phi D_{\max}}) )$. Substituting the above bound into the previous display proves the first statement.

    For the second statement, we have that
    \begin{align}
        \left\| \bm{\mu}^{(m)}_{k,h} \right\|_2 \leq \left\| \bm{A}^J \right\|_2 \left\| \bm{w}^{(m)}_{k-1,h} \right\|_2 + \left\| I - \bm{A}^J \right\|_2 \left\| \widehat{\bm{w}}_{k,h} \right\|_2 \leq \exp \left( - \frac{J}{2 \kappa} \right) \overline{R}_w + 2 H \sqrt{d_\phi D_{\max}} \leq \frac{1}{3} + 2 H \sqrt{d_\phi D_{\max}}.
    \end{align}
    The first inequality uses the definition of $\bm{\mu}^{(m)}_{k,h}$, the triangle inequality, and the definition of the matrix norm. The second inequality holds since the eigenvalues of $\bm{A}^J$ are $a_i^J \in (0, (1 - 1/(2\kappa))^J]$, which gives $\| \bm{A}^J \|_2 \leq \exp(-J / (2 \kappa))$ and $\| I - \bm{A}^J \|_2 \leq 1$, together with the assumption $\| \bm{w}^{(m)}_{k-1,h} \|_2 \leq \overline{R}_w$ and~\cref{lem:w_hat_bound}. The last inequality holds since the choice of $J_k$ gives $\exp(-J / (2 \kappa)) \leq 1 / (3 \sqrt{\left\lvert \mathcal{D}_{k} \right\rvert + 1} (\overline{R}_w + 2 H \sqrt{d_\phi D_{\max}})) \leq 1 / (3 \overline{R}_w)$. This concludes the proof.
\end{proof}

\subsubsection{Proof of \texorpdfstring{\cref{lem:w_hat_bound}}{}}
\begin{proof}
    Given the definition of $\widehat{\bm{w}}_{k,h}$ in~\cref{eq:w_hat}, we have that
    \begin{align}
        \left\| \widehat{\bm{w}}_{k,h} \right\|_2 &= \left\| \bm{\Lambda}_{k,h}^{-1} \sum_{\tau_i \in \mathcal{D}_k} \left[ r_h(s_{i,h}, a_{i,h}) + \widehat{V}_{k,h+1}(s_{i,h+1}) \right] \bm{\phi}_{i,h} \right\|_2 \\
        &\leq \sqrt{\frac{\left\lvert \mathcal{D}_{k} \right\rvert}{\lambda}} \left( \sum_{\tau_i \in \mathcal{D}_k} \left\| \left[ r_h(s_{i,h}, a_{i,h}) + \widehat{V}_{k,h+1}(s_{i,h+1}) \right] \bm{\phi}_{i,h} \right\|^2_{\bm{\Lambda}_{k,h}^{-1}} \right)^{1/2} \\
        &\leq 2 H \sqrt{\frac{\left\lvert \mathcal{D}_{k} \right\rvert}{\lambda}} \left( \sum_{\tau_i \in \mathcal{D}_k} \|\bm{\phi}_{i,h}\|^2_{\bm{\Lambda}_{k,h}^{-1}} \right)^{1/2} \\
        &\leq 2 H \sqrt{d_\phi \left\lvert \mathcal{D}_{k} \right\rvert / \lambda}.
    \end{align}
    The first inequality follows from~\cref{lem:l1_norm_and_matrix_norm} with $\lambda_{\max}(\bm{\Lambda}_{k,h}^{-1}) \leq 1 / \lambda$. The second inequality is due to the fact that $\widehat{V}_{k,h+1} \in [0, H]$ and the reward function is bounded by $1$. The last inequality follows from~\cref{lem:elliptical_potential}. Finally, the second statement follows from $\lambda = 1$ and $\left\lvert \mathcal{D}_{k} \right\rvert \leq D_{\max}$.
\end{proof}

\subsubsection{Proof of \texorpdfstring{\cref{lem:gaussian_event}}{}}
\begin{proof}
    We consider a fixed $(k, h, m) \in [K] \times [H] \times [M]$. By~\cref{lem:gaussian_w}, conditioned on $\mathcal{F}_{k,h}$, we have $\bm{w}^{(m)}_{k,h} \sim \mathcal{N}(\bm{\mu}^{(m)}_{k,h}, \bm{\Sigma}_{k,h})$. Since $\bm{\Sigma}_{k,h}^{-1/2} \left( \bm{w}^{(m)}_{k,h} - \bm{\mu}^{(m)}_{k,h} \right)$ is an affine transformation of $\bm{w}^{(m)}_{k,h}$ with covariance $\bm{\Sigma}_{k,h}^{-1/2} \bm{\Sigma}_{k,h} \bm{\Sigma}_{k,h}^{-1/2} = I$, it follows the standard Gaussian distribution $\mathcal{N}(\bm{0}, I)$ in $\mathbb{R}^{d_\phi}$ conditioned on $\mathcal{F}_{k,h}$. Therefore, applying~\cref{lem:gaussian_norm_concentration} with failure probability $\delta / (3KHM)$ gives
    \begin{align}
        \MoveEqLeft
        \Pr \left( \left\| \bm{\Sigma}_{k,h}^{-1/2} \left( \bm{w}^{(m)}_{k,h} - \bm{\mu}^{(m)}_{k,h} \right) \right\|_2 > \gamma_\delta \given \mathcal{F}_{k,h} \right) \\
        &\leq \Pr \left( \left\| \bm{\Sigma}_{k,h}^{-1/2} \left( \bm{w}^{(m)}_{k,h} - \bm{\mu}^{(m)}_{k,h} \right) \right\|_2 \geq \sqrt{d_\phi} + \sqrt{2 \log \left( \frac{3KHM}{\delta} \right)} \given \mathcal{F}_{k,h} \right) \leq \frac{\delta}{3KHM}.
    \end{align}
    Taking the expectation over $\mathcal{F}_{k,h}$, we have $\Pr ( \| \bm{\Sigma}_{k,h}^{-1/2} \left( \bm{w}^{(m)}_{k,h} - \bm{\mu}^{(m)}_{k,h} \right) \|_2 > \gamma_\delta ) \leq \delta / (3KHM)$. Finally, applying a union bound over $(k, h, m) \in [K] \times [H] \times [M]$ gives
    \begin{align}
        \Pr \left( \mathcal{E}_w^c \right) \leq \sum_{k=1}^K \sum_{h=1}^H \sum_{m=1}^M \Pr \left( \left\| \bm{\Sigma}_{k,h}^{-1/2} \left( \bm{w}^{(m)}_{k,h} - \bm{\mu}^{(m)}_{k,h} \right) \right\|_2 > \gamma_\delta \right) \leq \frac{\delta}{3}.
    \end{align}
    This concludes the proof.
\end{proof}

\subsubsection{Proof of \texorpdfstring{\cref{lem:anti_concentration_event}}{}}
\begin{proof}
    We consider a fixed $(k, h) \in [K] \times [H]$ and a fixed $\bm{u} \in \mathcal{U}$. As shown in the proof of~\cref{lem:gaussian_event}, conditioned on $\mathcal{F}_{k,h}$, the vector $\bm{\Sigma}_{k,h}^{-1/2} \left( \bm{w}^{(m)}_{k,h} - \bm{\mu}^{(m)}_{k,h} \right)$ follows the standard Gaussian distribution $\mathcal{N}(\bm{0}, I)$ for each $m \in [M]$. Moreover, by~\cref{lem:gaussian_w}, these vectors are conditionally independent across $m \in [M]$. Since $\left\| \bm{u} \right\|_2 = 1$ and the net $\mathcal{U}$ does not depend on the data, $\bm{u}^\top \bm{\Sigma}_{k,h}^{-1/2} \left( \bm{w}^{(m)}_{k,h} - \bm{\mu}^{(m)}_{k,h} \right)$ follows the standard Gaussian distribution $\mathcal{N}(0, 1)$ conditioned on $\mathcal{F}_{k,h}$. Therefore,
    \begin{align}
        \MoveEqLeft
        \Pr \left( \max_{m \in [M]} \bm{u}^\top \bm{\Sigma}_{k,h}^{-1/2} \left( \bm{w}^{(m)}_{k,h} - \bm{\mu}^{(m)}_{k,h} \right) < 1 \given \mathcal{F}_{k,h} \right) \\
        &= \prod_{m=1}^M \Pr \left( \bm{u}^\top \bm{\Sigma}_{k,h}^{-1/2} \left( \bm{w}^{(m)}_{k,h} - \bm{\mu}^{(m)}_{k,h} \right) < 1 \given \mathcal{F}_{k,h} \right) \leq (1 - c)^M.
    \end{align}
    The equality uses the conditional independence across $m \in [M]$. The inequality applies~\cref{lem:gaussian_abramowitz} with $z = 1$, which shows that each factor is at most $1 - e^{-1/2} / \sqrt{8 \pi} = 1 - c$. Taking the expectation over $\mathcal{F}_{k,h}$ and applying a union bound over $(k, h) \in [K] \times [H]$ and $\bm{u} \in \mathcal{U}$, we have that
    \begin{align}
        \Pr \left( \mathcal{E}_M^c \right) \leq K H \lvert \mathcal{U} \rvert (1 - c)^M \leq K H \left( 1 + 4 \gamma_\delta \right)^{d_\phi} (1 - c)^M \leq \frac{\delta}{3}.
    \end{align}
    The first inequality uses the union bound. The second inequality uses that $\mathcal{U}$ can be chosen with $\lvert \mathcal{U} \rvert \leq (1 + 4 \gamma_\delta)^{d_\phi}$ (\cref{lem:sphere_net}). The last inequality follows from the choice of $M$, which ensures $M \log(1 / (1 - c)) \geq d_\phi \log(1 + 4 \gamma_\delta) + \log(3KH / \delta)$. This concludes the proof.
\end{proof}

\subsubsection{Proof of \texorpdfstring{\cref{lem:iota_for_w_hat}}{}}

\begin{proof}
    Recall that $\mathbb{P}_h V (s, a) \coloneq \mathbb{E}_{s^\prime \sim \mathbb{P}_h(\cdot \mid s, a)} V(s^\prime)$ and $\mathbb{P}_h (\cdot \mid s, a) = \bm{\phi}(s, a)^\top \bm{\psi}_h(\cdot)$ due to the linear MDP assumption (\cref{def:linear_mdp}). We also denote that $\widehat{\bm{\Psi}}_{k,h} \coloneq \left\langle \bm{\psi}_h, \widehat{V}_{k,h+1} \right\rangle_{\mathcal{S}}$ and thus $\mathbb{P}_h \widehat{V}_{k,h+1}(s, a) = \bm{\phi}(s, a)^\top \widehat{\bm{\Psi}}_{k,h}$. Then, we have that
    \begin{align}
        \MoveEqLeft
        \mathbb{P}_h \widehat{V}_{k,h+1}(s, a) = \bm{\phi}(s, a)^\top \widehat{\bm{\Psi}}_{k,h} \\
        &= \bm{\phi}(s, a)^\top \bm{\Lambda}_{k,h}^{-1} \bm{\Lambda}_{k,h} \widehat{\bm{\Psi}}_{k,h} \\
        &= \bm{\phi}(s, a)^\top \bm{\Lambda}_{k,h}^{-1} \left( \sum_{\tau_i \in \mathcal{D}_k} \bm{\phi}_{i,h} \bm{\phi}_{i,h}^\top + \lambda I \right) \widehat{\bm{\Psi}}_{k,h} \\
        &= \bm{\phi}(s, a)^\top \bm{\Lambda}_{k,h}^{-1} \left( \sum_{\tau_i \in \mathcal{D}_k}\bm{\phi}_{i,h}(\mathbb{P}_h \widehat{V}_{k,h+1})(s_{i,h}, a_{i,h}) + \lambda \widehat{\bm{\Psi}}_{k,h} \right).
    \end{align}
    This further implies that
    \begin{align}
        \MoveEqLeft
        \bm{\phi}(s, a)^\top \widehat{\bm{w}}_{k,h} - \mathcal{T}_{k,h}(s, a) = \bm{\phi}(s, a)^\top \bm{\Lambda}_{k,h}^{-1}\sum_{\tau_i \in \mathcal{D}_k}\left[ r_h(s_{i,h}, a_{i,h}) + \widehat{V}_{k,h+1}(s_{i,h+1}) \right] \bm{\phi}_{i,h} - r_h(s, a) \\
            &\quad -\bm{\phi}(s, a)^\top \bm{\Lambda}_{k,h}^{-1} \left( \sum_{\tau_i \in \mathcal{D}_k}\bm{\phi}_{i,h}(\mathbb{P}_h \widehat{V}_{k,h+1})(s_{i,h}, a_{i,h}) + \lambda \widehat{\bm{\Psi}}_{k,h} \right) \\
            &= \bm{\phi}(s, a)^\top \bm{\Lambda}_{k,h}^{-1} \left( \sum_{\tau_i \in \mathcal{D}_k}\bm{\phi}_{i,h} \left[ \widehat{V}_{k,h+1}(s_{i,h+1})-\mathbb{P}_h \widehat{V}_{k,h+1}(s_{i,h}, a_{i,h}) \right] \right) \\
            &\quad + \bm{\phi}(s, a)^\top\bm{\Lambda}_{k,h}^{-1} \left( \sum_{\tau_i \in \mathcal{D}_k} r_h(s_{i,h}, a_{i,h}) \bm{\phi}_{i,h} \right) - r_h(s, a) - \lambda \bm{\phi}(s, a)^\top \bm{\Lambda}_{k,h}^{-1} \widehat{\bm{\Psi}}_{k,h}.
    \end{align}
    The first equality uses $\mathcal{T}_{k,h}(s, a) = r_h(s, a) + \mathbb{P}_h \widehat{V}_{k,h+1}(s, a)$, the definition of $\widehat{\bm{w}}_{k,h}$ in~\cref{eq:w_hat}, and the above expression of $\mathbb{P}_h \widehat{V}_{k,h+1}(s, a)$. The second equality rearranges the terms. Applying the triangle inequality, we obtain
    \begin{align}
        \MoveEqLeft
        \left\lvert \bm{\phi}(s, a)^\top \widehat{\bm{w}}_{k,h} - \mathcal{T}_{k,h}(s, a) \right\rvert \\
            &\leq \underbrace{\left\lvert \bm{\phi}(s, a)^\top \bm{\Lambda}_{k,h}^{-1} \left( \sum_{\tau_i \in \mathcal{D}_k}\bm{\phi}_{i,h} \left[ \widehat{V}_{k,h+1}(s_{i,h+1})-\mathbb{P}_h \widehat{V}_{k,h+1}(s_{i,h}, a_{i,h}) \right] \right) \right\rvert}_{\text{(I)}} \\
            &\quad + \underbrace{\left\lvert \bm{\phi}(s, a)^\top\bm{\Lambda}_{k,h}^{-1} \left( \sum_{\tau_i \in \mathcal{D}_k} r_h(s_{i,h}, a_{i,h}) \bm{\phi}_{i,h} \right) - r_h(s, a) \right\rvert}_{\text{(II)}} + \underbrace{\lambda \left\lvert \bm{\phi}(s, a)^\top \bm{\Lambda}_{k,h}^{-1} \widehat{\bm{\Psi}}_{k,h} \right\rvert}_{\text{(III)}}. \label{eq:w_hat_error_decomposition}
    \end{align}
    We first bound Term (I) of~\cref{eq:w_hat_error_decomposition}. We have that
    \begin{align}
        \MoveEqLeft
        \text{Term (I)} \leq \left\| \sum_{\tau_i \in \mathcal{D}_k} \bm{\phi}_{i,h}\left[ \widehat{V}_{k,h+1}(s_{i,h+1})-\mathbb{P}_h \widehat{V}_{k,h+1}(s_{i,h}, a_{i,h}) \right] \right\|_{\bm{\Lambda}_{k,h}^{-1}}\left\| \bm{\phi}(s, a) \right\|_{\bm{\Lambda}_{k,h}^{-1}} \\
        &\leq C_\delta H \sqrt{d_\phi} \left\| \bm{\phi}(s, a) \right\|_{\bm{\Lambda}_{k,h}^{-1}}.
    \end{align}
    The first inequality follows from the Cauchy–Schwarz inequality, and the second inequality follows from the assumption that the inequality in the definition of $\mathcal{E}_V$ holds for $(k, h)$.

    Next, we bound Term (II) of~\cref{eq:w_hat_error_decomposition}. We have that
    \begin{align}
        \MoveEqLeft
        \text{Term (II)} = \left\lvert \bm{\phi}(s, a)^\top\bm{\Lambda}_{k,h}^{-1} \left( \sum_{\tau_i \in \mathcal{D}_k} r_h(s_{i,h}, a_{i,h}) \bm{\phi}_{i,h} \right) - \bm{\phi}(s, a)^\top \bm{\upsilon}_h \right\rvert \\
        &= \left\lvert \bm{\phi}(s, a)^\top\bm{\Lambda}_{k,h}^{-1} \left( \sum_{\tau_i \in \mathcal{D}_k} r_h(s_{i,h}, a_{i,h}) \bm{\phi}_{i,h} - \bm{\Lambda}_{k,h}\bm{\upsilon}_h \right) \right\rvert \\
        &= \left\lvert \bm{\phi}(s, a)^\top\bm{\Lambda}_{k,h}^{-1} \left( \sum_{\tau_i \in \mathcal{D}_k} r_h(s_{i,h}, a_{i,h}) \bm{\phi}_{i,h} - \sum_{\tau_i \in \mathcal{D}_k} \bm{\phi}_{i,h} \bm{\phi}_{i,h}^\top \bm{\upsilon}_h - \lambda \bm{\upsilon}_h \right) \right\rvert \\
        &= \left\lvert \bm{\phi}(s, a)^\top\bm{\Lambda}_{k,h}^{-1} \left( \sum_{\tau_i \in \mathcal{D}_k} r_h(s_{i,h}, a_{i,h}) \bm{\phi}_{i,h} - \sum_{\tau_i \in \mathcal{D}_k} \bm{\phi}_{i,h} r_h(s_{i,h}, a_{i,h}) - \lambda \bm{\upsilon}_h \right) \right\rvert \\
        &= \lambda \left\lvert \bm{\phi}(s, a)^\top \bm{\Lambda}_{k,h}^{-1} \bm{\upsilon}_h \right\rvert \\
        &\leq \lambda \left\| \bm{\phi}(s, a) \right\|_{\bm{\Lambda}_{k,h}^{-1}}\left\| \bm{\upsilon}_h \right\|_{\bm{\Lambda}_{k,h}^{-1}} \\
        &\leq \sqrt{\lambda d_\phi} \left\| \bm{\phi}(s, a) \right\|_{\bm{\Lambda}_{k,h}^{-1}}.
    \end{align}
    The first and the fourth equalities follow from the definition $r_h(s, a)=\bm{\phi}(s, a)^\top \bm{\upsilon}_h$, and the third equality uses the definition of $\bm{\Lambda}_{k,h}$. The first inequality applies the Cauchy–Schwarz inequality. The second inequality follows from $\left\| \bm{\upsilon}_h \right\|_{\bm{\Lambda}_{k,h}^{-1}} \leq \left\| \bm{\upsilon}_h \right\|_2 / \sqrt{\lambda}$ and $\left\| \bm{\upsilon}_h \right\|_2 \leq \sqrt{d_\phi}$ (\cref{def:linear_mdp}).

    Lastly, we bound Term (III) of~\cref{eq:w_hat_error_decomposition}. We have that
    \begin{align}
            \text{Term (III)}
            &\leq \lambda \left\| \bm{\phi}(s, a) \right\|_{\bm{\Lambda}_{k,h}^{-1}} \left\| \widehat{\bm{\Psi}}_{k,h} \right\|_{\bm{\Lambda}_{k,h}^{-1}} \\
            &\leq \sqrt{\lambda} \left\| \bm{\phi}(s, a) \right\|_{\bm{\Lambda}_{k,h}^{-1}} \left\| \widehat{\bm{\Psi}}_{k,h} \right\|_2 \\
            &= H \sqrt{\lambda} \left\| \bm{\phi}(s, a) \right\|_{\bm{\Lambda}_{k,h}^{-1}} \left\| \int_{\mathcal{S}} \bm{\psi}_h(s^\prime) \left( \widehat{V}_{k,h+1}(s^\prime) / H \right) \, \mathrm{d} s^\prime \right\|_2 \\
            &\leq H \sqrt{\lambda d_\phi} \left\| \bm{\phi}(s, a) \right\|_{\bm{\Lambda}_{k,h}^{-1}}.
    \end{align}
    The first inequality applies the Cauchy–Schwarz inequality. The second inequality follows from $\left\| \widehat{\bm{\Psi}}_{k,h} \right\|_{\bm{\Lambda}_{k,h}^{-1}} \leq \left\| \widehat{\bm{\Psi}}_{k,h} \right\|_2 / \sqrt{\lambda}$. The last inequality comes from the assumption that $\left\| \int_{\mathcal{S}} \bm{\psi}_h(s^\prime) V(s^\prime) \, \mathrm{d} s^\prime \right\|_2 \leq \sqrt{d_\phi}$ for any $V: \mathcal{S} \to [0, 1]$ (\cref{def:linear_mdp}), applied to $V = \widehat{V}_{k,h+1} / H \in [0, 1]$.

    Putting everything together and setting $\lambda = 1$, we have that
    \begin{align}
        \left\lvert \bm{\phi}(s, a)^\top \widehat{\bm{w}}_{k,h} - \mathcal{T}_{k,h}(s, a) \right\rvert &\leq \text{Term (I)} + \text{Term (II)} + \text{Term (III)} \\
        &\leq \left( C_\delta H \sqrt{d_\phi} + \sqrt{d_\phi} + H \sqrt{d_\phi} \right) \left\| \bm{\phi}(s, a) \right\|_{\bm{\Lambda}_{k,h}^{-1}} \\
        &\leq 3 C_\delta H \sqrt{d_\phi} \left\| \bm{\phi}(s, a) \right\|_{\bm{\Lambda}_{k,h}^{-1}}.
    \end{align}
    The first inequality is~\cref{eq:w_hat_error_decomposition}. The second inequality combines the bounds on the three terms above. The last inequality uses $C_\delta \geq 1$ and $H \geq 1$. This concludes the proof.
\end{proof}

\section{Proof of Sample Complexity in the On-Policy Setting}
\label{sec:on_policy_sample_complexity}

\subsection{Good Events}

Recall the class $\mathcal{Q}_h$ defined in~\cref{lem:covering_number_of_v}. In the on-policy setting, we only need to cover this $Q$-function class, since the policy $\pi_k$ is fixed before $\mathcal{D}_k$ is collected.

\begin{lemma}
    \label{lem:on_policy_self_normalized_concentration}
    Consider~\cref{alg:optimistic_actor_critic} in the on-policy setting with the $\LMC$ critic and the hyperparameter choices of~\cref{lem:lmc_optimism_and_error_bound}. Let $C_\delta = C_\delta^\on \geq 1$ satisfy
    \begin{align}
        C_\delta^\on \geq 2 \sqrt{\frac{1}{2} \log(1 + N) + M \log \left( 1 + \frac{2 N \overline{R}_w}{H} \right) + \frac{\log(3KH/\delta) + 2}{d_\phi}}, \label{eq:c_on}
    \end{align}
    where $\overline{R}_w$ is defined in~\cref{eq:w_bar} with $D_{\max} = N$ and depends on $C_\delta^\on$ through $\beta$. Then, $\Pr(\mathcal{E}_w \setminus \mathcal{E}_V) \leq \delta / 3$. Moreover, there exists such a choice with $C_\delta^\on = \mathcal{O}(\sqrt{M \log(d_\phi H N K / \delta)}) = \widetilde{\mathcal{O}}(\sqrt{d_\phi})$.
\end{lemma}

\begin{proof}[Proof of \cref{lem:on_policy_self_normalized_concentration}]
    We consider a fixed $(k, h) \in [K] \times [H]$. When $h = H$, the inequality in the definition of $\mathcal{E}_V$ trivially holds since $\widehat{V}_{k,H+1} = 0$. Hence, it suffices to consider $h < H$. We denote by $\mathcal{F}_{k-1}$ the $\sigma$-field generated by all the trajectories and noise vectors before $\mathcal{D}_k$ is collected, so that $\pi_k$ is $\mathcal{F}_{k-1}$-measurable. We then consider the value function class $\mathcal{V}^{\on}_{k,h+1} \coloneq \{ \langle Q(\cdot, \cdot), \pi_{k,h+1}(\cdot \mid \cdot) \rangle_{\mathcal{A}} \mid Q \in \mathcal{Q}_{h+1} \}$, which is fixed given $\mathcal{F}_{k-1}$ and consists of functions taking values in $[0, H]$. Since $\lvert \langle Q - Q^\prime, \pi_{k,h+1} \rangle_{\mathcal{A}} \rvert \leq \sup_{(s,a)} \lvert Q(s, a) - Q^\prime(s, a) \rvert$, the proof of~\cref{lem:covering_number_of_v} gives $\log N_\Delta(\mathcal{V}^\on_{k,h+1}) \leq \log N_\Delta(\mathcal{Q}_{h+1}) \leq M d_\phi \log(1 + 2 \overline{R}_w / \Delta)$.

    Conditioned on $\mathcal{F}_{k-1}$, the trajectories $\tau_1, \ldots, \tau_N$ in $\mathcal{D}_k$ are independent and collected using $\pi_k$. Consider the filtration $\{\mathcal{G}_i\}_{i=0}^N$, where $\mathcal{G}_i$ is generated by $\mathcal{F}_{k-1}$, the trajectories $\tau_1, \ldots, \tau_i$, and the first $h$ state-action pairs $(s_{i+1,1}, a_{i+1,1}, \ldots, s_{i+1,h}, a_{i+1,h})$ of $\tau_{i+1}$. Then, $\bm{\phi}_{i,h}$ is $\mathcal{G}_{i-1}$-measurable, $s_{i,h+1}$ is $\mathcal{G}_i$-measurable, and $\mathbb{E}[V(s_{i,h+1}) \mid \mathcal{G}_{i-1}] = \mathbb{P}_h V(s_{i,h}, a_{i,h})$ for every fixed $V$. Applying~\cref{lem:d4_jin2020provably} conditionally on $\mathcal{F}_{k-1}$ with the class $\mathcal{V}^\on_{k,h+1}$, $\Delta = H / N$, and failure probability $\delta / (3KH)$, with probability at least $1 - \delta / (3KH)$, for all $V \in \mathcal{V}^\on_{k,h+1}$,
    \begin{align}
        \MoveEqLeft
        \left\| \sum_{\tau_i \in \mathcal{D}_k} \bm{\phi}_{i,h} \left[ V(s_{i,h+1}) - \mathbb{P}_h V(s_{i,h}, a_{i,h}) \right] \right\|^2_{\bm{\Lambda}_{k,h}^{-1}} \\
        &\leq 4 H^2 \left[ \frac{d_\phi}{2} \log(1 + N) + M d_\phi \log \left( 1 + \frac{2 N \overline{R}_w}{H} \right) + \log \left( \frac{3KH}{\delta} \right) \right] + 8 H^2 \\
        &\leq \left( C^\on_\delta H \sqrt{d_\phi} \right)^2.
    \end{align}
    On the event $\mathcal{E}_w$,~\cref{lem:bounded_critic_parameters} gives $\| \bm{w}^{(m)}_{k,h+1} \|_2 \leq \overline{R}_w$ for all $m$, and hence $\widehat{Q}_{k,h+1} \in \mathcal{Q}_{h+1}$ and $\widehat{V}_{k,h+1} \in \mathcal{V}^\on_{k,h+1}$. Therefore, the probability that $\mathcal{E}_w$ holds but the inequality in the definition of $\mathcal{E}_V$ fails for $(k, h)$ is at most $\delta / (3KH)$. A union bound over $(k, h)$ proves the first statement.

    For the second statement, note that for $C_\delta \geq 1$, we have $12 C_\delta H \sqrt{d_\phi} + 2 \leq 14 C_\delta H \sqrt{d_\phi}$, and hence $\overline{R}_w \leq w_1 C_\delta$ with $w_1 \coloneq 1/3 + 2 H \sqrt{d_\phi D_{\max}} + 14 \sqrt{2/3} H \sqrt{d_\phi} \gamma_\delta$, which does not depend on $C_\delta$. Since $\log(1 + x C_\delta) \leq \log(1 + x) + \log C_\delta$ for $C_\delta \geq 1$, the right-hand side of~\cref{eq:c_on} is at most $2 \sqrt{a + M \log C_\delta}$, where $a \coloneq \frac{1}{2} \log(1 + N) + M \log(1 + 2 N w_1 / H) + (\log(3KH/\delta) + 2) / d_\phi$. By~\cref{lem:fixed_point} with $b = M$, the choice $C^\on_\delta = 4 \sqrt{a + M \log(4 (a + M))}$ satisfies~\cref{eq:c_on}. Since $w_1$ is polynomial in $(d_\phi, H, N, \log(3KHM/\delta))$, we have $a = \mathcal{O}(M \log(d_\phi H N K / \delta))$, which gives the stated order. The last equality follows from $M = \widetilde{\mathcal{O}}(d_\phi)$.
\end{proof}

\subsection{Proof of \texorpdfstring{\cref{thm:on_policy_sample_efficiency}}{}}
\label{subsec:proof_on_policy_sample_efficiency}

Using~\cref{lem:on_policy_self_normalized_concentration}, we can instantiate~\cref{lem:lmc_optimism_and_error_bound} in the on-policy setting with $C_\delta = C^\on_\delta$, which gives
\begin{align}
    \Gamma_\LMC^\on \coloneq 3 C_\delta^\on H \sqrt{d_\phi} + \frac{1}{3} + \sqrt{\frac{2}{3 \beta}} \gamma_\delta = \mathcal{O} \left( C^\on_\delta H d_\phi \sqrt{\log(3KHM / \delta)} \right) = \widetilde{\mathcal{O}} \left( H d_\phi^{3/2} \right).
\end{align}
We also need the following bound on the expected bonuses.
\begin{lemma}
    \label{lem:on_policy_bonus}
    Consider~\cref{alg:optimistic_actor_critic} in the on-policy setting with $\lambda = 1$ and any critic. For any $\delta \in (0, 1)$, with probability at least $1 - \delta$, for all $k \in [K]$,
    \begin{align}
        \sum_{h=1}^H \mathbb{E}_{\pi_{k}} \left[ \left\| \bm{\phi}(s_h, a_h) \right\|_{\bm{\Lambda}_{k,h}^{-1} } \right] \leq H \sqrt{\frac{2}{N}} \left( \sqrt{d_\phi \log \left( 1 + \frac{N}{d_\phi} \right)} + \sqrt{\log \left( \frac{K}{\delta} \right)} \right),
    \end{align}
    where the expectation is taken over a trajectory $(s_1, a_1, \ldots, s_H, a_H)$ collected using $\pi_k$, while $\bm{\Lambda}_{k,h}$ is fixed.
\end{lemma}
\begin{proof}[Proof of \cref{lem:on_policy_bonus}]
    We consider a fixed $k \in [K]$ and index the $N$ trajectories in $\mathcal{D}_{k}$ as $\{\tau_i\}_{i \in [N]}$. We further define $\bm{\Lambda}^{(i)}_{k,h} \coloneq \sum_{j=1}^i \bm{\phi}_{j,h} \bm{\phi}_{j,h}^\top + \lambda I$ for $i \in \{0, \ldots, N\}$, so that $\bm{\Lambda}^{(N)}_{k,h} = \bm{\Lambda}_{k,h}$. Since $\bm{\Lambda}^{(i-1)}_{k,h} \preceq \bm{\Lambda}_{k,h}$ for every $i \in [N]$,
    \begin{align}
        \sum_{h=1}^H \mathbb{E}_{\pi_{k}} \left[ \left\| \bm{\phi}(s_h, a_h) \right\|_{\bm{\Lambda}_{k,h}^{-1} } \right] &\leq \frac{1}{N} \sum_{i=1}^N \sum_{h=1}^H \mathbb{E}_{\pi_{k}} \left[ \left\| \bm{\phi}(s_h, a_h) \right\|_{ (\bm{\Lambda}^{(i-1)}_{k,h})^{-1} } \right] \\
        &= \frac{1}{N} \sum_{i=1}^N \sum_{h=1}^H \left\| \bm{\phi}_{i,h} \right\|_{ (\bm{\Lambda}^{(i-1)}_{k,h})^{-1} } + \frac{1}{N} \sum_{i=1}^N X_i, \label{eq:on_policy_ellpitical_potential_sum}
    \end{align}
    where $X_i \coloneq \sum_{h=1}^H \left( \mathbb{E}_{\pi_{k}} \left[ \left\| \bm{\phi}(s_h, a_h) \right\|_{ (\bm{\Lambda}^{(i-1)}_{k,h})^{-1} } \right] - \left\| \bm{\phi}_{i,h} \right\|_{ (\bm{\Lambda}^{(i-1)}_{k,h})^{-1} } \right)$.

    \paragraph{Applying the elliptical potential lemma}
    For the first term of~\cref{eq:on_policy_ellpitical_potential_sum}, the Cauchy–Schwarz inequality,~\cref{lem:sum_induced_norm} (with $\bm{\Lambda}_0 = I$ and $L = 1$), and~\cref{lem:determinant_trace_inequality} give
    \begin{align}
        \sum_{h=1}^H \sum_{i=1}^N \left\| \bm{\phi}_{i,h} \right\|_{ (\bm{\Lambda}^{(i-1)}_{k,h})^{-1} } &\leq \sum_{h=1}^H \sqrt{N \sum_{i=1}^N \left\| \bm{\phi}_{i,h} \right\|^2_{ (\bm{\Lambda}^{(i-1)}_{k,h})^{-1} }} \\
        &\leq \sum_{h=1}^H \sqrt{2 N \log \det(\bm{\Lambda}_{k,h})} \leq H \sqrt{2 N d_\phi \log \left( 1 + \frac{N}{d_\phi} \right)}.
    \end{align}

    \paragraph{A martingale difference sequence}
    For the second term of~\cref{eq:on_policy_ellpitical_potential_sum}, let $\mathcal{G}^\prime_{i}$ be the $\sigma$-field generated by $\mathcal{F}_{k-1}$ and $\tau_1, \ldots, \tau_i$. Since $\pi_k$ and $\bm{\Lambda}^{(i-1)}_{k,h}$ are $\mathcal{G}^\prime_{i-1}$-measurable and $\tau_i$ is collected using $\pi_k$ independently of $\tau_1, \ldots, \tau_{i-1}$ given $\mathcal{F}_{k-1}$, we have $\mathbb{E}[X_i \mid \mathcal{G}^\prime_{i-1}] = 0$. Moreover, $\lvert X_i \rvert \leq H$ since $\left\| \bm{\phi} \right\|_{(\bm{\Lambda}^{(i-1)}_{k,h})^{-1}} \leq \left\| \bm{\phi} \right\|_2 \leq 1$. By~\cref{lem:azuma}, with probability at least $1 - \delta / K$, $\sum_{i=1}^N X_i \leq H \sqrt{2 N \log(K / \delta)}$.

    \paragraph{Putting everything together} Combining the two bounds with~\cref{eq:on_policy_ellpitical_potential_sum} and applying a union bound over $k \in [K]$ concludes the proof.
\end{proof}

Finally, we prove the regret decomposition in~\cref{eq:regret_decomposition}, which holds for any actor and any critic.
\begin{lemma}[Regret Decomposition]
    \label{lem:regret_decomposition}
    For any sequence of policies $\{\pi_k\}_{k=1}^K$ and $Q$-function estimates $\{\widehat{Q}_{k,h}\}$, let $\widehat{V}_{k,h}(s) \coloneq \langle \widehat{Q}_{k,h}(s, \cdot), \pi_{k,h}(\cdot \mid s) \rangle_{\mathcal{A}}$ and $\widehat{V}_{k,H+1}(s) \coloneq 0$ for every $(k, h, s) \in [K] \times [H] \times \mathcal{S}$, and let $\xi_{k,h}$ be the Bellman error in~\cref{def:bellman_error}. Then, it holds that
    \begin{align}
        \mathrm{Reg}_K &= \underbrace{\sum_{k=1}^K \sum_{h=1}^H \mathbb{E}_{\pi^\star} \left[ \left\langle \pi^\star_h(\cdot \mid s_h) - \pi_{k,h}(\cdot \mid s_h), \widehat{Q}_{k,h}(s_h, \cdot) \right\rangle \right]}_{\textup{(I)}} \\
        &\quad + \underbrace{\sum_{k=1}^K \sum_{h=1}^H \mathbb{E}_{\pi^\star} \left[ \xi_{k,h}(s_h, a_h) \right]}_{\textup{(II)}} + \underbrace{\sum_{k=1}^K \sum_{h=1}^H \mathbb{E}_{\pi_k} \left[ -\xi_{k,h}(s_h, a_h) \right]}_{\textup{(III)}}.
    \end{align}
\end{lemma}
\begin{proof}[Proof of \cref{lem:regret_decomposition}]
    Consider a fixed $k \in [K]$. We first write
    \begin{align}
        V^\star_1(s_1) - V_1^{\pi_{k}}(s_1) = \left( V^\star_1(s_1) - \widehat{V}_{k,1}(s_1) \right) + \left( \widehat{V}_{k,1}(s_1) - V_1^{\pi_{k}}(s_1) \right).
    \end{align}
    Applying~\cref{lem:extended_value_diff} with $\pi = \pi^\star$, $\pi^\prime = \pi_k$, and $\widehat{Q} = \widehat{Q}_k$ gives
    \begin{align}
        V^\star_1(s_1) - \widehat{V}_{k,1}(s_1) = \sum_{h=1}^H \mathbb{E}_{\pi^\star} \left[ \left\langle \pi^\star_h(\cdot \mid s_h) - \pi_{k,h}(\cdot \mid s_h), \widehat{Q}_{k,h}(s_h, \cdot) \right\rangle \right] + \sum_{h=1}^H \mathbb{E}_{\pi^\star} \left[ \xi_{k,h}(s_h, a_h) \right].
    \end{align}
    Applying~\cref{lem:extended_value_diff} with $\pi = \pi^\prime = \pi_k$, the first sum vanishes, and we obtain
    \begin{align}
        \widehat{V}_{k,1}(s_1) - V_1^{\pi_{k}}(s_1) = \sum_{h=1}^H \mathbb{E}_{\pi_{k}} \left[ -\xi_{k,h}(s_h, a_h) \right].
    \end{align}
    Both identities use the definition of $\xi_{k,h}$ in~\cref{def:bellman_error}. Combining the three displays and summing over $k \in [K]$ concludes the proof, since $\mathrm{Reg}_K = \sum_{k=1}^K (V^\star_1(s_1) - V_1^{\pi_{k}}(s_1))$.
\end{proof}

\begin{proof}[Proof of \cref{thm:on_policy_sample_efficiency}]
    In~\cref{alg:optimistic_actor_critic}, the estimated value functions satisfy $\widehat{V}_{k,h}(s) = \langle \widehat{Q}_{k,h}(s, \cdot), \pi_{k,h}(\cdot \mid s) \rangle_{\mathcal{A}}$ and $\widehat{V}_{k,H+1} = 0$. Hence, applying~\cref{lem:regret_decomposition} gives~\cref{eq:regret_decomposition}. This decomposition holds deterministically, and the high-probability events are only needed to bound Terms (II) and (III).

    \paragraph{Bounding Term (I)}
    We first bound Term (I) of~\cref{eq:regret_decomposition}. We have that
    \begin{align}
        \MoveEqLeft
        \text{Term (I)} = \sum_{h=1}^H \mathbb{E}_{\pi^\star} \left[ \sum_{k=1}^K \left\langle \widehat{Q}_{k,h}(s_h,\cdot), \pi_h^\star(\cdot \mid s_h)- \pi_{k,h}(\cdot \mid s_h) \right\rangle \right] \\
        &\leq \sum_{h=1}^H \mathbb{E}_{\pi^\star} \left[ \frac{\log|\mathcal{A}|}{\eta} + \frac{\eta K H^2}{2} + \frac{1}{\eta} \sum_{k=1}^K \varepsilon_{k,h} \right] \\
        &\leq H \left( \frac{\log|\mathcal{A}|}{\eta} + \frac{\eta K H^2}{2} \right) + \frac{H K}{\eta} \max_{k \in [K], h \in [H]} \varepsilon_{k,h} \\
        &= H^2 \sqrt{2 K \log|\mathcal{A}|} + K \varepsilon_{\proj}.
    \end{align}
    The first inequality follows from~\cref{lem:generalized_omd_regret} with $n = \lvert \mathcal{A} \rvert$, $L = H$, $p_k = \pi_{k,h}(\cdot \mid s_h)$, $g_k = \widehat{Q}_{k,h}(s_h, \cdot)$, and $u = \pi^\star_h(\cdot \mid s_h)$, which applies since $\widehat{Q}_{k,h}$ is clipped to $[0, H-h+1] \subseteq [0, H]$ in~\cref{alg:lmc}, together with $\varepsilon^{\proj}_{k,h}(\pi^\star_h(\cdot \mid s), s) \leq \varepsilon_{k,h}$. The last equality follows from the definition of $\varepsilon_{\proj}$ in~\cref{eq:eps_proj} and the choice $\eta = \sqrt{2 \log|\mathcal{A}| / (K H^2)}$. This is exactly~\cref{eq:npg-regret-bound}.

    \paragraph{Bounding Terms (II) and (III)}
    By~\cref{lem:good_event,lem:on_policy_self_normalized_concentration}, the event $\mathcal{E}$ holds with probability at least $1 - \delta$. On this event, the optimism guarantee (RHS) in~\cref{lem:lmc_optimism_and_error_bound} gives $\xi_{k,h}(s, a) \leq 0$ for all $(k, h, s, a)$, and hence $\text{Term (II)} \leq 0$. Therefore,
    \begin{align}
        \text{Term (II)} + \text{Term (III)} &\leq \Gamma^{\on}_{\LMC} \sum_{k=1}^K \sum_{h=1}^H \mathbb{E}_{\pi_{k}} \left[ \left\| \bm{\phi}(s_h, a_h) \right\|_{\bm{\Lambda}_{k,h}^{-1} } \right] \\
        &\leq \Gamma^{\on}_{\LMC} K H \sqrt{\frac{2}{N}} \left( \sqrt{d_\phi \log \left( 1 + \frac{N}{d_\phi} \right)} + \sqrt{\log \left( \frac{K}{\delta} \right)} \right) \\
        &= \widetilde{\mathcal{O}} \left( H^2 \sqrt{K} \right).
    \end{align}
    The first inequality uses $\text{Term (II)} \leq 0$ and the error bound (LHS) in~\cref{lem:lmc_optimism_and_error_bound}. The second inequality uses~\cref{lem:on_policy_bonus}, which holds with probability at least $1 - \delta$. The equality uses $\Gamma^{\on}_{\LMC} = \widetilde{\mathcal{O}}(H d_\phi^{3/2})$ and the choice $N = \lceil d_\phi^4 K / \log \lvert \mathcal{A} \rvert \rceil$.

    \paragraph{Putting everything together} Since $\SubOpt(\overline{\pi}_K) = \mathrm{Reg}_K / K$, combining the above bounds with~\cref{eq:regret_decomposition} gives that, with probability at least $1 - 2 \delta$,
    \begin{align}
        \SubOpt(\overline{\pi}_K) = \frac{1}{K} \left( \text{Term (I)} + \text{Term (II)} + \text{Term (III)} \right) \leq \widetilde{\mathcal{O}} \left( \frac{H^2}{\sqrt{K}} + \varepsilon_{\proj} \right).
    \end{align}
    Replacing $\delta$ by $\delta / 2$ proves the stated bound. Finally, with $K = H^4 \log \lvert \mathcal{A} \rvert \epsilon^{-2}$, the bound becomes $\widetilde{\mathcal{O}}(\epsilon + \varepsilon_{\proj})$. This concludes the proof.
\end{proof}

\section{Proof of Sample Complexity in the Off-Policy Setting}
\label{sec:off_policy_sample_complexity}

\subsection{Covering Number (Proof of \texorpdfstring{\cref{lem:covering_number_of_v}}{})}

We first show that the actor parameters remain bounded. This is not an assumption: the critic clips $\widehat{Q}_{k,h}$, so each $\PNPG$ step shifts the target logits by a bounded amount, and the coreset regression transfers this bound to the parameters.
\begin{lemma}[Bounded Parameters]
    \label{lem:bounded_parameters}
    Consider~\cref{alg:optimistic_actor_critic} with the $\PNPG$ actor (\cref{alg:actor_npg}) and the $\LMC$ critic (\cref{alg:lmc}). Under~\cref{asm:opt_error}, for every $k \in [K+1]$ and $h \in [H]$, it holds that
    \begin{align}
        \left\| \bm{\theta}_{k,h} \right\|_2 \leq \overline{R}_\theta \coloneq K \left( \frac{\eta H}{\lambda_G} + \sqrt{\frac{\varepsilon_{\opt}}{\lambda_G}} \right). \label{eq:theta_bound}
    \end{align}
    Consequently, the policy class $\Pi_{R_\theta}$ in~\cref{lem:covering_number_of_v} with $R_\theta = \overline{R}_\theta$ is fixed before data collection and contains $\pi_{k,h}$ for every $k \in [K+1]$ and $h \in [H]$.
\end{lemma}

\begin{proof}[Proof of \cref{lem:bounded_parameters}]
Consider specific $k$ and $h$. The gradient of the coreset actor loss in~\cref{eq:actor_loss} vanishes at its minimizer $\bm{\theta}^\star_{k+1,h}$. Since $z_{k,h}(s, a) - \bm{x}(s, a)^\top \bm{\theta}_{k,h} = \eta \widehat{Q}_{k,h}(s, a)$, this first-order condition reads
\begin{align}
    \bm{\theta}^\star_{k+1,h} - \bm{\theta}_{k,h} = \eta \bm{G}^{-1} \sum_{(s, a) \in \mathcal{C}} \rho_{\mathcal{C}}(s, a) \bm{x}(s, a) \widehat{Q}_{k,h}(s, a).
\end{align}
The clipping in~\cref{alg:lmc} gives $\lvert \widehat{Q}_{k,h}(s, a) \rvert \leq H$. Together with $\left\| \bm{G}^{-1} \right\|_2 = 1 / \lambda_G$ and $\left\| \bm{x}(s, a) \right\|_2 \leq 1$, this gives $\left\| \bm{\theta}^\star_{k+1,h} - \bm{\theta}_{k,h} \right\|_2 \leq \eta H / \lambda_G$. By~\cref{eq:quadratic_identity},~\cref{asm:opt_error}, and $\bm{G} \succeq \lambda_G I$, the updated parameters satisfy $\left\| \bm{\theta}_{k+1,h} - \bm{\theta}^\star_{k+1,h} \right\|_2 \leq \left\| \bm{\theta}_{k+1,h} - \bm{\theta}^\star_{k+1,h} \right\|_{\bm{G}} / \sqrt{\lambda_G} \leq \sqrt{\varepsilon_{\opt} / \lambda_G}$. The triangle inequality then gives $\left\| \bm{\theta}_{k+1,h} - \bm{\theta}_{k,h} \right\|_2 \leq \overline{R}_\theta / K$. Since $\bm{\theta}_{1,h} = \bm{0}$, summing at most $K$ such increments gives $\left\| \bm{\theta}_{k,h} \right\|_2 \leq \overline{R}_\theta$ for every $k \in [K+1]$. This argument only uses the clipping range of the estimated $Q$-functions, and therefore holds for every possible history. Hence, $\Pi_{\overline{R}_\theta}$ is fixed before data collection and contains every policy produced by~\cref{alg:actor_npg}.
\end{proof}

\begin{proof}[Proof of \cref{lem:covering_number_of_v}]
We consider a fixed $h \in [H]$ and any $Q, Q^\prime \in \mathcal{Q}_h$ with parameters $\{\bm{w}^{(m)}\}_{m \in [M]}$ and $\{\bm{w}^{\prime (m)}\}_{m \in [M]}$, respectively. Since the clipping function and the maximum over $m \in [M]$ are both $1$-Lipschitz with respect to the $\ell_\infty$-norm, and $\left\| \bm{\phi}(s, a) \right\|_2 \leq 1$ (\cref{def:linear_mdp}), the Cauchy–Schwarz inequality gives
\begin{align}
    \sup_{(s, a) \in \mathcal{S} \times \mathcal{A}} \left\lvert Q(s, a) - Q^\prime(s, a) \right\rvert \leq \sup_{(s, a) \in \mathcal{S} \times \mathcal{A}} \max_{m \in [M]} \left\lvert \bm{\phi}(s, a)^\top \left( \bm{w}^{(m)} - \bm{w}^{\prime (m)} \right) \right\rvert \leq \max_{m \in [M]} \left\| \bm{w}^{(m)} - \bm{w}^{\prime (m)} \right\|_2.
\end{align}
Consider any $\pi, \pi^\prime \in \Pi_{R_\theta}$ with parameters $\bm{\theta}$ and $\bm{\theta}^\prime$, respectively. By~\cref{lem:softmax_policy_distance} and $\left\| \bm{x}(s, a) \right\|_2 \leq 1$, for every $s \in \mathcal{S}$,
\begin{align}
    \left\| \pi(\cdot \mid s) - \pi^\prime(\cdot \mid s) \right\|_{1} \leq 2 \max_{a \in \mathcal{A}} \left\lvert \bm{x}(s, a)^\top \left( \bm{\theta} - \bm{\theta}^\prime \right) \right\rvert \leq 2 \left\| \bm{\theta} - \bm{\theta}^\prime \right\|_2.
\end{align}
Therefore, a $(\Delta / 2)$-cover of $\mathcal{Q}_h$ in the sup-norm is induced by $M$ Euclidean $(\Delta / 2)$-covers of the ball of radius $R_w$ in $\mathbb{R}^{d_\phi}$, and a $\Delta / (2H)$-cover of $\Pi_{R_\theta}$ with respect to $\sup_{s \in \mathcal{S}} \left\| \pi(\cdot \mid s) - \pi^\prime(\cdot \mid s) \right\|_1$ is induced by a Euclidean $\Delta / (4H)$-cover of the ball of radius $R_\theta$ in $\mathbb{R}^{d_x}$. Hence, we can bound the logarithm of the covering number of the value function class as follows:
\begin{align}
    \log N_\Delta (\mathcal{V}_h) &\leq \log N_{\Delta / 2} (\mathcal{Q}_h) + \log N_{\Delta / (2 H)} (\Pi_{R_\theta}) \\
    &\leq M d_\phi \log \left( 1 + \frac{4 R_w}{\Delta} \right) + d_x \log \left( 1 + \frac{8 H R_\theta}{\Delta} \right),
\end{align}
where the first inequality follows from~\cref{lem:covering_number_v_q_pi}, and the second inequality uses~\cref{lem:covering_number_euclidean_ball}. The same argument shows that $\log N_\Delta (\mathcal{Q}_h) \leq M d_\phi \log(1 + 2 R_w / \Delta)$.
This concludes the proof.
\end{proof}

\subsection{Good Events}

\begin{lemma}
    \label{lem:off_policy_self_normalized_concentration}
    Consider~\cref{alg:optimistic_actor_critic} in the off-policy setting with the $\PNPG$ actor, the $\LMC$ critic, and the hyperparameter choices of~\cref{lem:lmc_optimism_and_error_bound}. Let $C_\delta = C_\delta^\off \geq 1$ satisfy
    \begin{align}
        C_\delta^\off \geq 2 \sqrt{\frac{1}{2}\log(1 + K) + \frac{\mathsf{V} + \log(3H / \delta) + 2}{d_\phi}}, \label{eq:c_off}
    \end{align}
    where $\mathsf{V} \coloneq M d_\phi \log \left( 1 + 4 K \overline{R}_w / H \right) + d_x \log \left( 1 + 8 K \overline{R}_\theta \right)$, $\overline{R}_w$ is defined in~\cref{eq:w_bar} with $D_{\max} = K$ and depends on $C_\delta^\off$ through $\beta$, and $\overline{R}_\theta$ is defined in~\cref{eq:theta_bound}. Then, $\Pr(\mathcal{E}_w \setminus \mathcal{E}_V) \leq \delta / 3$. Moreover, there exists such a choice with $C_\delta^\off = \widetilde{\mathcal{O}}(\sqrt{d_\phi + d_x / d_\phi})$.
\end{lemma}
Note that $\mathsf{V}$ is the bound in~\cref{lem:covering_number_of_v} with $R_w = \overline{R}_w$, $R_\theta = \overline{R}_\theta$, and $\Delta = H / K$.

\begin{proof}[Proof of \cref{lem:off_policy_self_normalized_concentration}]
    We consider a fixed $h \in [H]$. When $h = H$, the inequality in the definition of $\mathcal{E}_V$ trivially holds since $\widehat{V}_{k,H+1} = 0$. Hence, it suffices to consider $h < H$. For every $k$, the estimated value function $\widehat{V}_{k,h+1}(\cdot) = \left\langle \widehat{Q}_{k,h+1}(\cdot, \cdot), \pi_{k,h+1}(\cdot \mid \cdot) \right\rangle_{\mathcal{A}}$ takes values in $[0, H]$. Moreover, on the event $\mathcal{E}_w$,~\cref{lem:bounded_critic_parameters,lem:bounded_parameters} imply that $\widehat{V}_{k,h+1} \in \mathcal{V}_{h+1}$ for every $k \in [K]$, where $\mathcal{V}_{h+1}$ is the value function class in~\cref{lem:covering_number_of_v} with $R_w = \overline{R}_w$ and $R_\theta = \overline{R}_\theta$, which is fixed before data collection.

    In the off-policy setting, $\mathcal{D}_k = \{\tau_i\}_{i \in [k]}$, where $\tau_i$ is collected using $\pi_i$. Consider the filtration $\{\mathcal{G}_i\}_{i = 0}^K$, where $\mathcal{G}_i$ is generated by the trajectories $\tau_1, \ldots, \tau_i$, all the noise vectors used by the critic in the episodes $1, \ldots, i$, and the first $h$ state-action pairs $(s_{i+1,1}, a_{i+1,1}, \ldots, s_{i+1,h}, a_{i+1,h})$ of $\tau_{i+1}$. Then, $\bm{\phi}_{i,h}$ is $\mathcal{G}_{i-1}$-measurable, $s_{i,h+1}$ is $\mathcal{G}_i$-measurable, and $\mathbb{E}[V(s_{i,h+1}) \mid \mathcal{G}_{i-1}] = \mathbb{P}_h V(s_{i,h}, a_{i,h})$ for every fixed $V$. Moreover, $\bm{\Lambda}_{k,h} = I + \sum_{i=1}^k \bm{\phi}_{i,h} \bm{\phi}_{i,h}^\top$. Hence, we can invoke~\cref{lem:d4_jin2020provably} with the class $\mathcal{V}_{h+1}$, $\Delta = H / K$, and failure probability $\delta / (3H)$: with probability at least $1-\delta / (3H)$, for all $k \in [K]$ and all $V \in \mathcal{V}_{h+1}$,
    \begin{align}
        \MoveEqLeft
        \left\| \sum_{\tau_i \in \mathcal{D}_k}\bm{\phi}_{i,h}\left[ V(s_{i,h+1})- \mathbb{P}_h V(s_{i,h}, a_{i,h}) \right] \right\|^2_{\bm{\Lambda}_{k,h}^{-1}} \\
        &\leq 4 H^2 \left[ \frac{d_\phi}{2} \log(1 + K) + \log N_{\Delta}(\mathcal{V}_{h+1}) + \log \left( \frac{3H}{\delta} \right) \right] + 8 K^2 \Delta^2 \\
        &\leq 4 H^2 d_\phi \left[ \frac{1}{2}\log(1 + K) + \frac{\mathsf{V} + \log(3H / \delta) + 2}{d_\phi} \right] \\
        &\leq \left( C^\off_\delta H \sqrt{d_\phi} \right)^2,
    \end{align}
    where the second inequality uses~\cref{lem:covering_number_of_v}. Therefore, the probability that $\mathcal{E}_w$ holds but the inequality in the definition of $\mathcal{E}_V$ fails for some $k$ at step $h$ is at most $\delta / (3H)$. A union bound over $h \in [H]$ proves the first statement.

    For the second statement, as in the proof of~\cref{lem:on_policy_self_normalized_concentration}, we have $\overline{R}_w \leq w_1 C_\delta$ for $C_\delta \geq 1$, where $w_1$ does not depend on $C_\delta$. Hence, the right-hand side of~\cref{eq:c_off} is at most $2 \sqrt{a + M \log C_\delta}$, where
    \begin{align}
        a \coloneq \frac{1}{2}\log(1 + K) + M \log \left( 1 + \frac{4 K w_1}{H} \right) + \frac{d_x}{d_\phi} \log \left( 1 + 8 K \overline{R}_\theta \right) + \frac{\log(3H/\delta) + 2}{d_\phi}.
    \end{align}
    By~\cref{lem:fixed_point} with $b = M$, the choice $C^\off_\delta = 4 \sqrt{a + M \log(4 (a + M))}$ satisfies~\cref{eq:c_off}. Since $M = \widetilde{\mathcal{O}}(d_\phi)$, and $w_1$ and $\overline{R}_\theta$ are polynomial in the problem parameters, we have $a = \widetilde{\mathcal{O}}(d_\phi + d_x / d_\phi)$, which gives the stated order.
\end{proof}

\subsection{Proof of \texorpdfstring{\cref{thm:off_policy_sample_efficiency}}{}}
\label{subsec:proof_off_policy_sample_efficiency}
    We first instantiate~\cref{lem:lmc_optimism_and_error_bound} in the off-policy setting with $C_\delta = C^\off_\delta$, which gives
    \begin{align}
        \Gamma_\LMC^\off \coloneq 3 C_\delta^\off H \sqrt{d_\phi} + \frac{1}{3} + \sqrt{\frac{2}{3 \beta}} \gamma_\delta = \mathcal{O} \left( C_\delta^\off H d_\phi \sqrt{\log(3KHM / \delta)} \right) = \widetilde{\mathcal{O}} \left( H \sqrt{d_\phi^3 + d_\phi d_x} \right).
    \end{align}
    We also need the following bound on the expected bonuses.
    \begin{lemma}
        \label{lem:off_policy_bonus}
        Consider~\cref{alg:optimistic_actor_critic} in the off-policy setting with $\lambda = 1$ and any critic. For any $\delta \in (0, 1)$, with probability at least $1 - \delta$,
        \begin{align}
            \sum_{k=1}^K \sum_{h=1}^H \mathbb{E}_{\pi_{k}} \left[ \left\| \bm{\phi}(s_h, a_h) \right\|_{\bm{\Lambda}_{k,h}^{-1} } \right] \leq H \sqrt{2 d_\phi K \log \left( 1 + \frac{K}{d_\phi} \right)} + H \sqrt{2 K \log(1 / \delta)},
        \end{align}
        where each expectation is taken over a trajectory collected using $\pi_k$, while $\bm{\Lambda}_{k,h}$ is fixed.
    \end{lemma}
    \begin{proof}[Proof of \cref{lem:off_policy_bonus}]
        Recall that $\mathcal{D}_k = \{\tau_i\}_{i \in [k]}$, and let $\bm{\Lambda}_{k-1,h} = \sum_{i=1}^{k-1} \bm{\phi}_{i,h} \bm{\phi}_{i,h}^\top + \lambda I$ with $\bm{\Lambda}_{0,h} = I$. Since $\bm{\Lambda}_{k-1,h} \preceq \bm{\Lambda}_{k,h}$, we have that
        \begin{align}
            \sum_{k=1}^K \sum_{h=1}^H \mathbb{E}_{\pi_{k}} \left[ \left\| \bm{\phi}(s_h, a_h) \right\|_{\bm{\Lambda}_{k,h}^{-1} } \right] \leq \sum_{k=1}^K \sum_{h=1}^H \left\| \bm{\phi}_{k,h} \right\|_{ \bm{\Lambda}_{k-1,h}^{-1} } + \sum_{k=1}^K X_k, \label{eq:off_policy_ellpitical_potential_sum}
        \end{align}
        where $X_k \coloneq \sum_{h=1}^H \left( \mathbb{E}_{\pi_{k}} \left[ \left\| \bm{\phi}(s_h, a_h) \right\|_{ \bm{\Lambda}_{k-1,h}^{-1} } \right] - \left\| \bm{\phi}_{k,h} \right\|_{ \bm{\Lambda}_{k-1,h}^{-1} } \right)$.

        \paragraph{Applying the elliptical potential lemma}
        For the first term of~\cref{eq:off_policy_ellpitical_potential_sum}, the Cauchy–Schwarz inequality,~\cref{lem:sum_induced_norm} (with $\bm{\Lambda}_0 = I$ and $L = 1$), and~\cref{lem:determinant_trace_inequality} give
        \begin{align}
            \sum_{h=1}^H \sum_{k=1}^K \left\| \bm{\phi}_{k,h} \right\|_{ \bm{\Lambda}_{k-1,h}^{-1} } \leq \sum_{h=1}^H \sqrt{K \sum_{k=1}^K \left\| \bm{\phi}_{k,h} \right\|^2_{ \bm{\Lambda}_{k-1,h}^{-1} }} \leq H \sqrt{2 d_\phi K \log \left( 1 + \frac{K}{d_\phi} \right)}.
        \end{align}

        \paragraph{A martingale difference sequence}
        For the second term of~\cref{eq:off_policy_ellpitical_potential_sum}, let $\mathcal{G}^\prime_k$ be the $\sigma$-field generated by all the trajectories and noise vectors up to the end of episode $k$. Since $\pi_k$ and $\bm{\Lambda}_{k-1,h}$ are $\mathcal{G}^\prime_{k-1}$-measurable and $\tau_k$ is collected using $\pi_k$, we have $\mathbb{E}[X_k \mid \mathcal{G}^\prime_{k-1}] = 0$. Moreover, $\lvert X_k \rvert \leq H$ since $\left\| \bm{\phi} \right\|_{\bm{\Lambda}_{k-1,h}^{-1}} \leq \left\| \bm{\phi} \right\|_2 \leq 1$. By~\cref{lem:azuma}, with probability at least $1 - \delta$, $\sum_{k=1}^K X_k \leq H \sqrt{2 K \log(1 / \delta)}$, which concludes the proof.
    \end{proof}

    \begin{proof}[Proof of \cref{thm:off_policy_sample_efficiency}]
        Following the proof of~\cref{thm:on_policy_sample_efficiency} in~\cref{subsec:proof_on_policy_sample_efficiency}, we apply~\cref{lem:regret_decomposition} and use the same regret decomposition in~\cref{eq:regret_decomposition}.

        \paragraph{Bounding Term (I)} Term (I) of~\cref{eq:regret_decomposition} can be bounded in the same way as in the proof of~\cref{thm:on_policy_sample_efficiency}, i.e., $\text{Term (I)} \leq H^2 \sqrt{2 K \log|\mathcal{A}|} + K \varepsilon_{\proj}$.

        \paragraph{Bounding Terms (II) and (III)} By~\cref{lem:good_event,lem:off_policy_self_normalized_concentration}, the event $\mathcal{E}$ holds with probability at least $1 - \delta$. On this event, the optimism guarantee (RHS) in~\cref{lem:lmc_optimism_and_error_bound} gives $\text{Term (II)} \leq 0$. Therefore, with probability at least $1 - 2 \delta$,
        \begin{align}
            \text{Term (II)} + \text{Term (III)} &\leq \Gamma^{\off}_{\LMC} \sum_{k=1}^K \sum_{h=1}^H \mathbb{E}_{\pi_{k}} \left[ \left\| \bm{\phi}(s_h, a_h) \right\|_{\bm{\Lambda}_{k,h}^{-1} } \right] \\
            &\leq \Gamma^{\off}_{\LMC} \left( H \sqrt{2 d_\phi K \log \left( 1 + \frac{K}{d_\phi} \right)} + H \sqrt{2 K \log(1 / \delta)} \right) \\
            &= \widetilde{\mathcal{O}} \left( \sqrt{(d_\phi^4 + d_\phi^2 d_x) H^4 K} \right).
        \end{align}
        The first inequality uses $\text{Term (II)} \leq 0$ and the error bound (LHS) in~\cref{lem:lmc_optimism_and_error_bound}. The second inequality uses~\cref{lem:off_policy_bonus}, which holds with probability at least $1 - \delta$. The equality uses $\Gamma^{\off}_{\LMC} = \widetilde{\mathcal{O}}(H \sqrt{d_\phi^3 + d_\phi d_x})$.

        \paragraph{Putting everything together} Since $\SubOpt(\overline{\pi}_K) = \mathrm{Reg}_K / K$, combining the above bounds with~\cref{eq:regret_decomposition} gives that, with probability at least $1 - 2 \delta$,
        \begin{align}
            \MoveEqLeft
            \SubOpt(\overline{\pi}_K) = \frac{1}{K} \left( \text{Term (I)} + \text{Term (II)} + \text{Term (III)} \right) \\
            &\leq \widetilde{\mathcal{O}} \left( H^2 \sqrt{\frac{d_\phi^4 + d_\phi^2 d_x}{K}} + \varepsilon_{\proj} \right).
        \end{align}
        Replacing $\delta$ by $\delta / 2$ proves the stated bound. Finally, with $K = (\log \lvert \mathcal{A} \rvert + d_\phi^4 + d_\phi^2 d_x) H^4 \epsilon^{-2}$, the bound becomes $\widetilde{\mathcal{O}}(\epsilon + \varepsilon_{\proj})$. This concludes the proof.
    \end{proof}

\section{Instantiating the Critic with UCB Bonuses}
\label{sec:ucb_critic}

Our framework in~\cref{alg:optimistic_actor_critic} also accommodates other critics. In this section, we instantiate the critic using $\UCB$ bonuses~\citep{jin2020provably}, which are used in prior work~\citep{zhong2023theoretical,liu2023optimistic,sherman2023rate,cassel2024warm}, and compare the resulting guarantees with those of the $\LMC$ critic.

\subsection{Algorithm Design}

The $\UCB$ critic is presented in~\cref{alg:ucb}. Compared to~\cref{alg:lmc}, it replaces the noisy gradient descent updates with the exact minimizer $\widehat{\bm{w}}_{k,h}$ of the critic loss in~\cref{eq:critic_loss}, and it replaces the maximum over the $M$ critic samples with an explicit bonus. Since the actor remains log-linear, the per-episode cost of policy inference remains $\mathcal{O}(d_x H \lvert \mathcal{A} \rvert)$, whereas each evaluation of $\widehat{Q}_{k,h}$ in the critic and actor updates costs $\mathcal{O}(d_\phi^2)$.

\begin{algorithm2e}[!ht]
\caption{Critic: $\UCB$}\label{alg:ucb}
\begin{algorithmic}[1]
\State \textbf{Input}: collected data $\mathcal{D}_{k}$, policy $\pi_{k}$, bonus coefficient $\alpha$
\State $\widehat{V}_{k,H+1} (s) \gets 0$, for every queried $s$
\For{$h=H, H-1, \ldots, 1$}
    \State Define the critic loss $\ell^{\cri}_{k,h} (\bm{w})$ using~\cref{eq:critic_loss}, and compute its minimizer $\widehat{\bm{w}}_{k,h} = \bm{\Lambda}_{k,h}^{-1} \bm{b}_{k,h}$
    \State $\displaystyle \widehat{Q}_{k,h} (s, a) \gets \Clip_{[0, H-h+1]} \left( \bm{\phi}(s, a)^\top \widehat{\bm{w}}_{k,h} + \alpha \left\| \bm{\phi}(s, a) \right\|_{\bm{\Lambda}_{k,h}^{-1}} \right)$, for every queried $(s, a)$
    \State $\widehat{V}_{k,h} (s) \gets \mathbb{E}_{a \sim \pi_{k,h}(\cdot \mid s)} \widehat{Q}_{k,h} (s, a)$, for every queried $s$
\EndFor
\State \textbf{Return}: estimated $Q$-function $\widehat{Q}_k = \{\widehat{Q}_{k,h}\}_{h \in [H]}$ and critic parameters $\bm{w}_k = \{\widehat{\bm{w}}_{k,h}\}_{h \in [H]}$
\end{algorithmic}
\end{algorithm2e}

\subsection{Optimism and Error Bound}

Similar to the $\LMC$ critic, the $\UCB$ critic is optimistic at every state-action pair simultaneously. However, it only requires the event $\mathcal{E}_V$ defined in~\cref{subsec:proof_of_lmc_lemma}, and it does not require the events $\mathcal{E}_w$ and $\mathcal{E}_M$. Moreover, the proofs of~\cref{lem:w_hat_bound,lem:iota_for_w_hat} in~\cref{sec:critic_analyses} only use the critic loss in~\cref{eq:critic_loss} and the fact that $\widehat{V}_{k,h+1} \in [0, H]$, and hence they hold for any critic whose estimated value function lies in $[0, H]$. Since the $\UCB$ critic also clips $\widehat{Q}_{k,h+1}$ to $[0, H - h]$, both lemmas apply to the $\UCB$ critic.
\begin{lemma}[Optimism and Error Bound for UCB]
    \label{lem:ucb_optimism_and_error_bound}
    Consider~\cref{alg:optimistic_actor_critic} with the $\UCB$ critic from~\cref{alg:ucb}, and let $\mathcal{E}_V$ be the event in~\cref{def:good_event} for some $C_\delta \geq 1$. Suppose that we choose $\lambda = 1$ and $\alpha = 3 C_\delta H \sqrt{d_\phi}$. Then, on the event $\mathcal{E}_V$, for all $(k, h, s, a) \in [K] \times [H] \times \mathcal{S} \times \mathcal{A}$, it holds that
    \begin{align}
        - \Gamma_{\UCB} \left\| \bm{\phi}(s, a) \right\|_{\bm{\Lambda}_{k,h}^{-1}} \leq \xi_{k,h}(s, a) \leq 0, \label{eq:ucb_optimism_and_error_bound}
    \end{align}
    where $\Gamma_{\UCB} \coloneq 2 \alpha$. Consequently, with $C_\delta$ chosen as in~\cref{lem:ucb_on_policy_concentration} (on-policy) or~\cref{lem:ucb_off_policy_concentration} (off-policy), the above inequalities hold with probability at least $1 - \delta$.
\end{lemma}
\begin{proof}
    We consider a fixed $(k, h, s, a) \in [K] \times [H] \times \mathcal{S} \times \mathcal{A}$ and denote $y \coloneq \bm{\phi}(s, a)^\top \widehat{\bm{w}}_{k,h} + \alpha \left\| \bm{\phi}(s, a) \right\|_{\bm{\Lambda}_{k,h}^{-1}}$, so that $\widehat{Q}_{k,h}(s, a) = \Clip_{[0, H-h+1]}(y)$. On the event $\mathcal{E}_V$,~\cref{lem:iota_for_w_hat} and the choice of $\alpha$ give that
    \begin{align}
        \left\lvert \bm{\phi}(s, a)^\top \widehat{\bm{w}}_{k,h} - \mathcal{T}_{k,h}(s, a) \right\rvert \leq 3 C_\delta H \sqrt{d_\phi} \left\| \bm{\phi}(s, a) \right\|_{\bm{\Lambda}_{k,h}^{-1}} = \alpha \left\| \bm{\phi}(s, a) \right\|_{\bm{\Lambda}_{k,h}^{-1}}. \label{eq:ucb_w_hat}
    \end{align}
    Using~\cref{eq:ucb_w_hat}, we first have that
    \begin{align}
        y = \bm{\phi}(s, a)^\top \widehat{\bm{w}}_{k,h} + \alpha \left\| \bm{\phi}(s, a) \right\|_{\bm{\Lambda}_{k,h}^{-1}} \geq \mathcal{T}_{k,h}(s, a) - \alpha \left\| \bm{\phi}(s, a) \right\|_{\bm{\Lambda}_{k,h}^{-1}} + \alpha \left\| \bm{\phi}(s, a) \right\|_{\bm{\Lambda}_{k,h}^{-1}} = \mathcal{T}_{k,h}(s, a) \geq 0,
    \end{align}
    where the first inequality uses~\cref{eq:ucb_w_hat}, and the second inequality uses $\mathcal{T}_{k,h}(s, a) \geq 0$. Since $y \geq 0$, we have $\widehat{Q}_{k,h}(s, a) = \min\{y, H - h + 1\}$. We now prove the two inequalities separately.

    \textbf{Error Bound (LHS)} \,
    We can lower bound $\xi_{k,h}$ as follows.
    \begin{align}
        - \xi_{k,h}(s, a) &= \widehat{Q}_{k,h}(s, a) - \mathcal{T}_{k,h}(s, a) \leq y - \mathcal{T}_{k,h}(s, a) \\
        &= \bm{\phi}(s, a)^\top \widehat{\bm{w}}_{k,h} - \mathcal{T}_{k,h}(s, a) + \alpha \left\| \bm{\phi}(s, a) \right\|_{\bm{\Lambda}_{k,h}^{-1}} \leq 2 \alpha \left\| \bm{\phi}(s, a) \right\|_{\bm{\Lambda}_{k,h}^{-1}},
    \end{align}
    where the first inequality uses $\widehat{Q}_{k,h}(s, a) = \min\{y, H - h + 1\} \leq y$, and the second inequality uses~\cref{eq:ucb_w_hat}. Since $\Gamma_{\UCB} = 2 \alpha$, this proves the LHS of~\cref{eq:ucb_optimism_and_error_bound}.

    \textbf{Optimism (RHS)} \,
    Since both $y \geq \mathcal{T}_{k,h}(s, a)$ and $H - h + 1 \geq \mathcal{T}_{k,h}(s, a)$, we have that $\widehat{Q}_{k,h}(s, a) = \min\{y, H - h + 1\} \geq \mathcal{T}_{k,h}(s, a)$, i.e., $\xi_{k,h}(s, a) = \mathcal{T}_{k,h}(s, a) - \widehat{Q}_{k,h}(s, a) \leq 0$, which proves the RHS of~\cref{eq:ucb_optimism_and_error_bound}. This concludes the proof.
\end{proof}

\subsection{Sample Complexity}

We are now ready to present the sample complexity of the actor-critic framework with the $\UCB$ critic. Similar to~\cref{sec:sample_efficiency_analysis}, we consider the on-policy and off-policy settings separately.
\begin{theorem}
\label{thm:ucb_on_policy}
Under~\cref{asm:bias,asm:opt_error}, consider~\cref{alg:optimistic_actor_critic} in the on-policy setting with the $\PNPG$ actor (\cref{alg:actor_npg}) and the $\UCB$ critic (\cref{alg:ucb}). For an appropriate choice of the actor and critic parameters, including $\lambda = 1$, $\alpha = 3 C^{\UCB, \on}_\delta H \sqrt{d_\phi}$ with $C^{\UCB, \on}_\delta$ chosen as in~\cref{lem:ucb_on_policy_concentration}, $\eta = \sqrt{2 \log \lvert \mathcal{A} \rvert / (K H^2)}$, and $N = \lceil d_\phi^3 K / \log \lvert \mathcal{A} \rvert \rceil$, and any $\delta \in (0,1)$, it holds that with probability at least $1 - \delta$,
\begin{align}
    \SubOpt(\overline{\pi}_K) \leq \widetilde{\mathcal{O}} \left( \frac{H^2}{\sqrt{K}} + \varepsilon_{\proj} \right),
\end{align}
where $\varepsilon_{\proj}$ is the projection error defined in~\cref{eq:eps_proj} and controlled by~\cref{lem:projection_error}. Hence, for any $\epsilon > 0$,~\cref{alg:optimistic_actor_critic} with $K = H^4 \log \lvert \mathcal{A} \rvert \epsilon^{-2}$ yields an $\widetilde{\mathcal{O}}(\epsilon + \varepsilon_{\proj})$-optimal mixture policy, and therefore requires $K \times N = \widetilde{\mathcal{O}}(d_\phi^3 H^8 \epsilon^{-4})$ trajectories.
\end{theorem}

\begin{theorem}
\label{thm:ucb_off_policy}
Under~\cref{asm:bias,asm:opt_error}, consider~\cref{alg:optimistic_actor_critic} in the off-policy setting with the $\PNPG$ actor (\cref{alg:actor_npg}) and the $\UCB$ critic (\cref{alg:ucb}). For an appropriate choice of the actor and critic parameters, including $\lambda = 1$, $\alpha = 3 C^{\UCB, \off}_\delta H \sqrt{d_\phi}$ with $C^{\UCB, \off}_\delta$ chosen as in~\cref{lem:ucb_off_policy_concentration}, and $\eta = \sqrt{2 \log \lvert \mathcal{A} \rvert / (K H^2)}$, and any $\delta \in (0,1)$, it holds that with probability at least $1 - \delta$,
\begin{align}
    \SubOpt(\overline{\pi}_K) \leq \widetilde{\mathcal{O}} \left( H^2 \sqrt{\frac{d_\phi^3 + d_\phi d_x}{K}} + \varepsilon_{\proj} \right),
\end{align}
where $\varepsilon_{\proj}$ is the projection error defined in~\cref{eq:eps_proj} and controlled by~\cref{lem:projection_error}. Hence, for any $\epsilon > 0$,~\cref{alg:optimistic_actor_critic} with $K = (\log \lvert \mathcal{A} \rvert + d_\phi^3 + d_\phi d_x) H^4 \epsilon^{-2}$ yields an $\widetilde{\mathcal{O}}(\epsilon + \varepsilon_{\proj})$-optimal mixture policy, and therefore requires $K \times 1 = \widetilde{\mathcal{O}}((d_\phi^3 + d_\phi d_x) H^4 \epsilon^{-2})$ trajectories.
\end{theorem}

\paragraph{Comparison with the LMC critic} Both critics guarantee optimism at all the state-action pairs simultaneously, and in both cases, the estimated $Q$-functions are described by $\widetilde{\mathcal{O}}(d_\phi^2)$ scalars: $d_\phi + d_\phi^2$ for the $\UCB$ critic (\cref{lem:ucb_covering}), and $M d_\phi$ with $M = \widetilde{\mathcal{O}}(d_\phi)$ for the $\LMC$ critic (\cref{lem:covering_number_of_v}). Consequently, the constants $C_\delta$ are of the same order for both critics, namely $\widetilde{\mathcal{O}}(\sqrt{d_\phi})$ in the on-policy setting and $\widetilde{\mathcal{O}}(\sqrt{d_\phi + d_x / d_\phi})$ in the off-policy setting. However, the error bound of the $\UCB$ critic is $\Gamma_{\UCB} = 2 \alpha = \mathcal{O}(C_\delta H \sqrt{d_\phi})$, whereas $\Gamma_{\LMC} = \mathcal{O}(C_\delta H d_\phi \sqrt{\log(3KHM/\delta)})$ has an additional $\sqrt{d_\phi}$ factor. This factor comes from controlling the Gaussian perturbations uniformly over all the directions in $\mathbb{R}^{d_\phi}$ (the term $\sqrt{2 / (3 \beta)} \gamma_\delta$ in $\Gamma_{\LMC}$ in~\cref{lem:lmc_optimism_and_error_bound}). As a result, for the $\UCB$ critic, $\text{Term (III)} / K$ in~\cref{eq:regret_decomposition} is $\widetilde{\mathcal{O}}(H^2 \sqrt{d_\phi^3 / N})$ and $\widetilde{\mathcal{O}}(H^2 \sqrt{(d_\phi^3 + d_\phi d_x) / K})$ in the on-policy and off-policy settings, respectively, whereas for the $\LMC$ critic, it is $\widetilde{\mathcal{O}}(H^2 \sqrt{d_\phi^4 / N})$ and $\widetilde{\mathcal{O}}(H^2 \sqrt{(d_\phi^4 + d_\phi^2 d_x) / K})$. Consequently, in the on-policy setting, the $\UCB$ critic only requires $N = \lceil d_\phi^3 K / \log \lvert \mathcal{A} \rvert \rceil$ trajectories per episode instead of $N = \lceil d_\phi^4 K / \log \lvert \mathcal{A} \rvert \rceil$, which reduces the total number of trajectories from $\widetilde{\mathcal{O}}(d_\phi^4 H^8 \epsilon^{-4})$ to $\widetilde{\mathcal{O}}(d_\phi^3 H^8 \epsilon^{-4})$ (\cref{thm:on_policy_sample_efficiency,thm:ucb_on_policy}). In the off-policy setting, the $\UCB$ critic improves the dependence of $K$ from $d_\phi^4 + d_\phi^2 d_x$ to $d_\phi^3 + d_\phi d_x$ (\cref{thm:off_policy_sample_efficiency,thm:ucb_off_policy}). This $\sqrt{d_\phi}$ gap is consistent with the gap between Thompson sampling and UCB-based algorithms for linear bandits~\citep{agrawal2013thompson,abeille2017linear}, and it can be viewed as the price for replacing the explicit bonus with randomized exploration, which is easier to implement and to extend beyond linear function approximation.

\subsection{Proof of Sample Complexity in the On-Policy Setting}
\label{subsec:ucb_on_policy_proof}

\subsubsection{Good Events}

The bonus introduces the matrix $\bm{\Lambda}_{k,h+1}^{-1}$ into the estimated value function, which enlarges the function class in the uniform concentration argument. For $h \in [H]$, we consider the $Q$-function class
\begin{align}
    \mathcal{Q}^{\UCB}_h \coloneq \left\{ \Clip_{[0, H-h+1]} \left( \bm{\phi}(\cdot, \cdot)^\top \bm{w} + \alpha \left\| \bm{\phi}(\cdot, \cdot) \right\|_{\bm{\Upsilon}} \right) \given \left\| \bm{w} \right\|_2 \leq 2 H \sqrt{d_\phi D_{\max}}, \bm{\Upsilon} \succeq 0, \left\| \bm{\Upsilon} \right\|_2 \leq 1 \right\},
\end{align}
where $D_{\max} \coloneq \max_{k \in [K]} \lvert \mathcal{D}_{k} \rvert$ is the maximum number of trajectories used by the critic, i.e., $D_{\max} = N$ in the on-policy setting and $D_{\max} = K$ in the off-policy setting.
By~\cref{lem:w_hat_bound}, $\| \widehat{\bm{w}}_{k,h} \|_2 \leq 2 H \sqrt{d_\phi D_{\max}}$. Moreover, since $\bm{\Lambda}_{k,h} \succeq I$, we have $\bm{\Lambda}_{k,h}^{-1} \succeq 0$ and $\| \bm{\Lambda}_{k,h}^{-1} \|_2 \leq 1$. Therefore, $\widehat{Q}_{k,h} \in \mathcal{Q}^{\UCB}_h$ for every $(k, h)$ and every possible history. The following lemma bounds the covering number of this class.
\begin{lemma}
    \label{lem:ucb_covering}
    For any $\Delta > 0$ and $h \in [H]$, it holds that
    \begin{align}
        \log N_\Delta(\mathcal{Q}^{\UCB}_h) \leq d_\phi \log \left( 1 + \frac{8 H \sqrt{d_\phi D_{\max}}}{\Delta} \right) + d_\phi^2 \log \left( 1 + \frac{8 \sqrt{d_\phi} \alpha^2}{\Delta^2} \right),
    \end{align}
    where the covering number is with respect to $\sup_{(s,a)} \lvert Q(s, a) - Q^\prime(s, a) \rvert$.
\end{lemma}
\begin{proof}
    We follow the proof of~\citet[Lemma D.6]{jin2020provably}. Consider any $Q, Q^\prime \in \mathcal{Q}^{\UCB}_h$ with parameters $(\bm{w}, \bm{\Upsilon})$ and $(\bm{w}^\prime, \bm{\Upsilon}^\prime)$, respectively. For any $(s, a) \in \mathcal{S} \times \mathcal{A}$, we write $\bm{\phi} = \bm{\phi}(s, a)$. Then, we have that
    \begin{align}
        \left\lvert Q(s, a) - Q^\prime(s, a) \right\rvert &\leq \left\lvert \bm{\phi}^\top \left( \bm{w} - \bm{w}^\prime \right) + \alpha \left( \left\| \bm{\phi} \right\|_{\bm{\Upsilon}} - \left\| \bm{\phi} \right\|_{\bm{\Upsilon}^\prime} \right) \right\rvert \\
        &\leq \left\| \bm{\phi} \right\|_2 \left\| \bm{w} - \bm{w}^\prime \right\|_2 + \alpha \left\lvert \sqrt{\bm{\phi}^\top \bm{\Upsilon} \bm{\phi}} - \sqrt{\bm{\phi}^\top \bm{\Upsilon}^\prime \bm{\phi}} \right\rvert \\
        &\leq \left\| \bm{w} - \bm{w}^\prime \right\|_2 + \alpha \sqrt{\left\lvert \bm{\phi}^\top \left( \bm{\Upsilon} - \bm{\Upsilon}^\prime \right) \bm{\phi} \right\rvert} \\
        &\leq \left\| \bm{w} - \bm{w}^\prime \right\|_2 + \alpha \sqrt{\left\| \bm{\Upsilon} - \bm{\Upsilon}^\prime \right\|_F}.
    \end{align}
    The first inequality uses the fact that the clipping function is $1$-Lipschitz. The second inequality applies the triangle inequality and the Cauchy–Schwarz inequality. The third inequality uses $\left\| \bm{\phi} \right\|_2 \leq 1$ and $\lvert \sqrt{x} - \sqrt{y} \rvert \leq \sqrt{\lvert x - y \rvert}$ for all $x, y \geq 0$. The last inequality uses $\lvert \bm{\phi}^\top \bm{B} \bm{\phi} \rvert \leq \| \bm{B} \|_2 \| \bm{\phi} \|_2^2 \leq \| \bm{B} \|_F$ for any symmetric matrix $\bm{B}$.

    Next, note that every $\bm{\Upsilon} \succeq 0$ with $\| \bm{\Upsilon} \|_2 \leq 1$ satisfies $\| \bm{\Upsilon} \|_F \leq \sqrt{d_\phi}$. Let $\mathcal{C}_{\bm{w}}$ be a $(\Delta / 2)$-cover of the Euclidean ball of radius $2 H \sqrt{d_\phi D_{\max}}$ in $\mathbb{R}^{d_\phi}$, and let $\mathcal{C}_{\bm{\Upsilon}}$ be a $\Delta^2 / (4 \alpha^2)$-cover of the set $\{ \bm{\Upsilon} : \| \bm{\Upsilon} \|_F \leq \sqrt{d_\phi} \}$ with respect to the Frobenius norm, which is a Euclidean ball of radius $\sqrt{d_\phi}$ in $\mathbb{R}^{d_\phi^2}$. If a center $\bm{\Upsilon}^\prime$ is not positive semi-definite, we define the corresponding function with $\sqrt{\max\{\bm{\phi}^\top \bm{\Upsilon}^\prime \bm{\phi}, 0\}}$, for which the above display still holds since $\lvert \sqrt{a} - \sqrt{\max\{b, 0\}} \rvert \leq \sqrt{\lvert a - b \rvert}$ for all $a \geq 0$ and $b \in \mathbb{R}$. By the above display, for every $Q \in \mathcal{Q}^{\UCB}_h$, there exists $(\bm{w}^\prime, \bm{\Upsilon}^\prime) \in \mathcal{C}_{\bm{w}} \times \mathcal{C}_{\bm{\Upsilon}}$ such that the corresponding function $Q^\prime$ satisfies $\sup_{(s,a)} \lvert Q(s, a) - Q^\prime(s, a) \rvert \leq \Delta / 2 + \alpha \sqrt{\Delta^2 / (4 \alpha^2)} = \Delta$. Therefore, using~\cref{lem:covering_number_euclidean_ball}, we have that
    \begin{align}
        \log N_\Delta(\mathcal{Q}^{\UCB}_h) \leq \log \lvert \mathcal{C}_{\bm{w}} \rvert + \log \lvert \mathcal{C}_{\bm{\Upsilon}} \rvert \leq d_\phi \log \left( 1 + \frac{8 H \sqrt{d_\phi D_{\max}}}{\Delta} \right) + d_\phi^2 \log \left( 1 + \frac{8 \sqrt{d_\phi} \alpha^2}{\Delta^2} \right).
    \end{align}
    This concludes the proof.
\end{proof}

Using the above covering number, we can now establish the event $\mathcal{E}_V$ for the $\UCB$ critic in the on-policy setting.
\begin{lemma}
    \label{lem:ucb_on_policy_concentration}
    Consider~\cref{alg:optimistic_actor_critic} in the on-policy setting with the $\UCB$ critic from~\cref{alg:ucb}, $\lambda = 1$, and $\alpha = 3 C_\delta H \sqrt{d_\phi}$, where $C_\delta = C^{\UCB, \on}_\delta \geq 1$ satisfies
    \begin{align}
        C_\delta^2 \geq 4 \left[ \frac{1}{2} \log(1 + N) + \log \left( 1 + 8 \sqrt{d_\phi} N^{3/2} \right) + d_\phi \log \left( 1 + 72 d_\phi^{3/2} N^2 C_\delta^2 \right) + \frac{\log(KH/\delta) + 2}{d_\phi} \right]. \label{eq:c_ucb_on}
    \end{align}
    Then, for any $\delta \in (0, 1)$, it holds that $\Pr(\mathcal{E}_V) \geq 1 - \delta$. Moreover, there exists such a choice with $C^{\UCB, \on}_\delta = \mathcal{O}(\sqrt{d_\phi \log(d_\phi H N K / \delta)})$.
\end{lemma}
\begin{proof}
    We consider a fixed $(k, h) \in [K] \times [H]$. When $h = H$, the inequality in the definition of $\mathcal{E}_V$ trivially holds since $\widehat{V}_{k,H+1} = 0$. Hence, it suffices to consider $h < H$. We use the same $\sigma$-field $\mathcal{F}_{k-1}$ and the same filtration $\{\mathcal{G}_i\}_{i=0}^N$ as in the proof of~\cref{lem:on_policy_self_normalized_concentration}. Recall that $\pi_k$ is $\mathcal{F}_{k-1}$-measurable, and conditioned on $\mathcal{F}_{k-1}$, the $N$ trajectories in $\mathcal{D}_k$ are independent and collected using $\pi_k$.

    We consider the value function class $\mathcal{V}^{\UCB, \on}_{k,h+1} \coloneq \{ \langle Q(\cdot, \cdot), \pi_{k,h+1}(\cdot \mid \cdot) \rangle_{\mathcal{A}} \mid Q \in \mathcal{Q}^{\UCB}_{h+1} \}$, which is fixed given $\mathcal{F}_{k-1}$ and consists of functions taking values in $[0, H]$. Since $\sup_s \lvert \langle Q(s, \cdot) - Q^\prime(s, \cdot), \pi_{k,h+1}(\cdot \mid s) \rangle_{\mathcal{A}} \rvert \leq \sup_{(s,a)} \lvert Q(s, a) - Q^\prime(s, a) \rvert$,~\cref{lem:ucb_covering} implies that
    \begin{align}
        \log N_\Delta \left( \mathcal{V}^{\UCB, \on}_{k,h+1} \right) \leq \log N_\Delta \left( \mathcal{Q}^{\UCB}_{h+1} \right) \leq d_\phi \log \left( 1 + \frac{8 H \sqrt{d_\phi D_{\max}}}{\Delta} \right) + d_\phi^2 \log \left( 1 + \frac{8 \sqrt{d_\phi} \alpha^2}{\Delta^2} \right).
    \end{align}
    Moreover, unlike the $\LMC$ critic, we have $\widehat{V}_{k,h+1} \in \mathcal{V}^{\UCB, \on}_{k,h+1}$ for every possible history, since $\widehat{Q}_{k,h+1} \in \mathcal{Q}^{\UCB}_{h+1}$. Therefore, applying~\cref{lem:d4_jin2020provably} conditionally on $\mathcal{F}_{k-1}$ with $\Delta = H / N$ and failure probability $\delta / (KH)$, with probability at least $1 - \delta / (KH)$,
    \begin{align}
        \MoveEqLeft
        \left\| \sum_{\tau_i \in \mathcal{D}_k} \bm{\phi}_{i,h} \left[ \widehat{V}_{k,h+1}(s_{i,h+1}) - \mathbb{P}_h \widehat{V}_{k,h+1}(s_{i,h}, a_{i,h}) \right] \right\|^2_{\bm{\Lambda}_{k,h}^{-1}} \\
        &\leq 4 H^2 \left[ \frac{d_\phi}{2} \log(1 + N) + \log N_\Delta \left( \mathcal{V}^{\UCB, \on}_{k,h+1} \right) + \log \left( \frac{KH}{\delta} \right) \right] + \frac{8 N^2 \Delta^2}{\lambda} \\
        &\leq 4 H^2 \left[ \frac{d_\phi}{2} \log(1 + N) + d_\phi \log \left( 1 + 8 \sqrt{d_\phi} N^{3/2} \right) + d_\phi^2 \log \left( 1 + \frac{8 \sqrt{d_\phi} \alpha^2 N^2}{H^2} \right) + \log \left( \frac{KH}{\delta} \right) \right] + 8 H^2 \\
        &= 4 H^2 d_\phi \left[ \frac{1}{2} \log(1 + N) + \log \left( 1 + 8 \sqrt{d_\phi} N^{3/2} \right) + d_\phi \log \left( 1 + 72 d_\phi^{3/2} N^2 C_\delta^2 \right) + \frac{\log(KH/\delta) + 2}{d_\phi} \right] \\
        &\leq \left( C_\delta H \sqrt{d_\phi} \right)^2.
    \end{align}
    The first inequality applies~\cref{lem:d4_jin2020provably} with $\lambda = 1$. The second inequality uses the covering number bound above and $\Delta = H / N$. The equality uses $\alpha^2 = 9 C_\delta^2 H^2 d_\phi$. The last inequality follows from~\cref{eq:c_ucb_on}. Applying a union bound over $(k, h) \in [K] \times [H]$ gives $\Pr(\mathcal{E}_V) \geq 1 - \delta$.

    For the second statement, we note that for $C_\delta \geq 1$, it holds that $\log(1 + x C_\delta^2) \leq \log((1 + x) C_\delta^2) = \log(1 + x) + 2 \log C_\delta$. Therefore, the right-hand side of~\cref{eq:c_ucb_on} is at most $4 (a + 2 d_\phi \log C_\delta)$, where
    \begin{align}
        a \coloneq \frac{1}{2} \log(1 + N) + \log \left( 1 + 8 \sqrt{d_\phi} N^{3/2} \right) + d_\phi \log \left( 1 + 72 d_\phi^{3/2} N^2 \right) + \frac{\log(KH/\delta) + 2}{d_\phi}
    \end{align}
    does not depend on $C_\delta$. By~\cref{lem:fixed_point} with $b = 2 d_\phi$, the choice $C^{\UCB, \on}_\delta = 4 \sqrt{a + 2 d_\phi \log(4 (a + 2 d_\phi))}$ satisfies~\cref{eq:c_ucb_on}. Since $a = \mathcal{O}(d_\phi \log(d_\phi H N K / \delta))$, this choice satisfies the stated order. This concludes the proof.
\end{proof}

\subsubsection{Proof of \texorpdfstring{\cref{thm:ucb_on_policy}}{}}

\begin{proof}[Proof of \cref{thm:ucb_on_policy}]
    We follow the proof of~\cref{thm:on_policy_sample_efficiency} in~\cref{subsec:proof_on_policy_sample_efficiency}. Since~\cref{lem:regret_decomposition} holds for any critic, we use the same regret decomposition in~\cref{eq:regret_decomposition}.

    \paragraph{Bounding Term (I)} Since the $\UCB$ critic also clips $\widehat{Q}_{k,h}$ to $[0, H - h + 1] \subseteq [0, H]$, the bound on Term (I) of~\cref{eq:regret_decomposition} in~\cref{subsec:proof_on_policy_sample_efficiency} applies verbatim, i.e.,
    \begin{align}
        \text{Term (I)} \leq H^2 \sqrt{2 K \log|\mathcal{A}|} + K \varepsilon_{\proj}.
    \end{align}

    \paragraph{Bounding Terms (II) and (III)} By~\cref{lem:ucb_on_policy_concentration}, the event $\mathcal{E}_V$ holds with probability at least $1 - \delta$. On this event, the optimism guarantee (RHS) in~\cref{lem:ucb_optimism_and_error_bound} gives $\text{Term (II)} \leq 0$. Therefore,
    \begin{align}
        \text{Term (II)} + \text{Term (III)} &\leq \Gamma_{\UCB} \sum_{k=1}^K \sum_{h=1}^H \mathbb{E}_{\pi_{k}} \left[ \left\| \bm{\phi}(s_h, a_h) \right\|_{\bm{\Lambda}_{k,h}^{-1}} \right] \\
        &\leq \Gamma_{\UCB} K H \sqrt{\frac{2}{N}} \left( \sqrt{d_\phi \log \left( 1 + \frac{N}{d_\phi} \right)} + \sqrt{\log \left( \frac{K}{\delta} \right)} \right) \\
        &= \widetilde{\mathcal{O}} \left( H^2 \sqrt{K} \right).
    \end{align}
    The first inequality follows from $\text{Term (II)} \leq 0$ and the error bound (LHS) in~\cref{lem:ucb_optimism_and_error_bound}. The second inequality follows from~\cref{lem:on_policy_bonus}, which holds for any critic with probability at least $1 - \delta$. The equality uses $\Gamma_{\UCB} = 2 \alpha = 6 C^{\UCB, \on}_\delta H \sqrt{d_\phi} = \widetilde{\mathcal{O}}(H d_\phi)$ from~\cref{lem:ucb_on_policy_concentration} and the choice $N = \lceil d_\phi^3 K / \log \lvert \mathcal{A} \rvert \rceil$.

    \paragraph{Putting everything together} Since $\SubOpt(\overline{\pi}_K) = \mathrm{Reg}_K / K$, combining the above bounds with~\cref{eq:regret_decomposition} gives that, with probability at least $1 - 2 \delta$,
    \begin{align}
        \SubOpt(\overline{\pi}_K) = \frac{1}{K} \left( \text{Term (I)} + \text{Term (II)} + \text{Term (III)} \right) \leq \widetilde{\mathcal{O}} \left( \frac{H^2}{\sqrt{K}} + \varepsilon_{\proj} \right).
    \end{align}
    Replacing $\delta$ by $\delta / 2$ proves the stated bound. Finally, with $K = H^4 \log \lvert \mathcal{A} \rvert \epsilon^{-2}$, the bound becomes $\widetilde{\mathcal{O}}(\epsilon + \varepsilon_{\proj})$. This concludes the proof.
\end{proof}

\subsection{Proof of Sample Complexity in the Off-Policy Setting}

\subsubsection{Covering Number}

Similar to~\cref{sec:off_policy_sample_complexity}, we first note that~\cref{lem:bounded_parameters} also holds for the $\UCB$ critic, since its proof only uses the clipping range of $\widehat{Q}_{k,h}$. Hence, the policy class $\Pi_{R_\theta}$ in~\cref{lem:covering_number_of_v} with $R_\theta = \overline{R}_\theta$ is fixed before data collection and contains $\pi_{k,h}$ for every $k \in [K+1]$ and $h \in [H]$. The following lemma bounds the covering number of the value function class induced by the $\UCB$ critic.
\begin{lemma}
    \label{lem:ucb_covering_v}
    For $h \in [H]$, let $\mathcal{V}^{\UCB}_{h} \coloneq \{ \left\langle Q(\cdot, \cdot), \pi(\cdot \mid \cdot) \right\rangle_{\mathcal{A}} \mid Q \in \mathcal{Q}^{\UCB}_{h}, \pi \in \Pi_{R_\theta} \}$, where $\mathcal{Q}^{\UCB}_{h}$ is defined in~\cref{subsec:ucb_on_policy_proof} and $\Pi_{R_\theta}$ is defined in~\cref{lem:covering_number_of_v}. Then, for any $\Delta > 0$, it holds that
    \begin{align}
        \log N_\Delta \left( \mathcal{V}^{\UCB}_{h} \right) \leq d_\phi \log \left( 1 + \frac{16 H \sqrt{d_\phi D_{\max}}}{\Delta} \right) + d_\phi^2 \log \left( 1 + \frac{32 \sqrt{d_\phi} \alpha^2}{\Delta^2} \right) + d_x \log \left( 1 + \frac{8 H R_\theta}{\Delta} \right).
    \end{align}
\end{lemma}
\begin{proof}
    By~\cref{lem:covering_number_v_q_pi}, we have that
    \begin{align}
        \log N_\Delta \left( \mathcal{V}^{\UCB}_{h} \right) \leq \log N_{\Delta / 2} \left( \mathcal{Q}^{\UCB}_{h} \right) + \log N_{\Delta / (2H)} \left( \Pi_{R_\theta} \right).
    \end{align}
    The first term is bounded by $d_\phi \log(1 + 16 H \sqrt{d_\phi D_{\max}} / \Delta) + d_\phi^2 \log(1 + 32 \sqrt{d_\phi} \alpha^2 / \Delta^2)$ using~\cref{lem:ucb_covering} with $\Delta / 2$. The second term is bounded by $d_x \log(1 + 8 H R_\theta / \Delta)$, as shown in the proof of~\cref{lem:covering_number_of_v}. This concludes the proof.
\end{proof}

\subsubsection{Good Events}

Next, we establish the event $\mathcal{E}_V$ for the $\UCB$ critic in the off-policy setting, where the uniform concentration must also hold over the policy class.
\begin{lemma}
    \label{lem:ucb_off_policy_concentration}
    Consider~\cref{alg:optimistic_actor_critic} in the off-policy setting with the $\PNPG$ actor, the $\UCB$ critic from~\cref{alg:ucb}, $\lambda = 1$, and $\alpha = 3 C_\delta H \sqrt{d_\phi}$, where $C_\delta = C^{\UCB, \off}_\delta \geq 1$ satisfies
    \begin{align}
        \MoveEqLeft
        C_\delta^2 \geq 4 \Bigg[ \frac{1}{2} \log(1 + K) + \log \left( 1 + 16 \sqrt{d_\phi} K^{3/2} \right) + d_\phi \log \left( 1 + 288 d_\phi^{3/2} K^2 C_\delta^2 \right) \\
        &+ \frac{d_x}{d_\phi} \log \left( 1 + 8 K \overline{R}_\theta \right) + \frac{\log(H/\delta) + 2}{d_\phi} \Bigg], \label{eq:c_ucb_off}
    \end{align}
    where $\overline{R}_\theta$ is defined in~\cref{eq:theta_bound}. Then, for any $\delta \in (0, 1)$, it holds that $\Pr(\mathcal{E}_V) \geq 1 - \delta$. Moreover, there exists such a choice with $C^{\UCB, \off}_\delta = \widetilde{\mathcal{O}}(\sqrt{d_\phi + d_x / d_\phi})$.
\end{lemma}
\begin{proof}
    We consider a fixed $h \in [H]$. As before, the case $h = H$ is trivial, so it suffices to consider $h < H$. We consider the value function class $\mathcal{V}^{\UCB}_{h+1}$ in~\cref{lem:ucb_covering_v} with $R_\theta = \overline{R}_\theta$. As discussed above, this class is fixed before data collection, and $\widehat{V}_{k,h+1} \in \mathcal{V}^{\UCB}_{h+1}$ for every $k \in [K]$ and every possible history.     Using the same filtration $\{\mathcal{G}_i\}_{i=0}^K$ as in the proof of~\cref{lem:off_policy_self_normalized_concentration}, we can apply~\cref{lem:d4_jin2020provably} with $\Delta = H / K$ and failure probability $\delta / H$. Then, with probability at least $1 - \delta / H$, for all $k \in [K]$,
    \begin{align}
        \MoveEqLeft
        \left\| \sum_{\tau_i \in \mathcal{D}_k} \bm{\phi}_{i,h} \left[ \widehat{V}_{k,h+1}(s_{i,h+1}) - \mathbb{P}_h \widehat{V}_{k,h+1}(s_{i,h}, a_{i,h}) \right] \right\|^2_{\bm{\Lambda}_{k,h}^{-1}} \\
        &\leq 4 H^2 \left[ \frac{d_\phi}{2} \log(1 + K) + \log N_\Delta \left( \mathcal{V}^{\UCB}_{h+1} \right) + \log \left( \frac{H}{\delta} \right) \right] + 8 K^2 \Delta^2 \\
        &\leq 4 H^2 d_\phi \Bigg[ \frac{1}{2} \log(1 + K) + \log \left( 1 + 16 \sqrt{d_\phi} K^{3/2} \right) + d_\phi \log \left( 1 + 288 d_\phi^{3/2} K^2 C_\delta^2 \right) \\
        &\qquad + \frac{d_x}{d_\phi} \log \left( 1 + 8 K \overline{R}_\theta \right) + \frac{\log(H/\delta) + 2}{d_\phi} \Bigg] \\
        &\leq \left( C_\delta H \sqrt{d_\phi} \right)^2.
    \end{align}
    The first inequality applies~\cref{lem:d4_jin2020provably} with $\lambda = 1$. The second inequality uses~\cref{lem:ucb_covering_v} with $R_\theta = \overline{R}_\theta$, $\Delta = H / K$, and $\alpha^2 = 9 C_\delta^2 H^2 d_\phi$. The last inequality follows from~\cref{eq:c_ucb_off}. Applying a union bound over $h \in [H]$ gives $\Pr(\mathcal{E}_V) \geq 1 - \delta$.

    For the second statement, similar to the proof of~\cref{lem:ucb_on_policy_concentration}, the right-hand side of~\cref{eq:c_ucb_off} is at most $4 (a + 2 d_\phi \log C_\delta)$ for $C_\delta \geq 1$, where
    \begin{align}
        a \coloneq \frac{1}{2} \log(1 + K) + \log \left( 1 + 16 \sqrt{d_\phi} K^{3/2} \right) + d_\phi \log \left( 1 + 288 d_\phi^{3/2} K^2 \right) + \frac{d_x}{d_\phi} \log \left( 1 + 8 K \overline{R}_\theta \right) + \frac{\log(H/\delta) + 2}{d_\phi}.
    \end{align}
    By~\cref{lem:fixed_point} with $b = 2 d_\phi$, the choice $C^{\UCB, \off}_\delta = 4 \sqrt{a + 2 d_\phi \log(4 (a + 2 d_\phi))}$ satisfies~\cref{eq:c_ucb_off}. Since $\overline{R}_\theta$ is polynomial in the problem parameters, we have $a = \widetilde{\mathcal{O}}(d_\phi + d_x / d_\phi)$, which gives the stated order. This concludes the proof.
\end{proof}

\subsubsection{Proof of \texorpdfstring{\cref{thm:ucb_off_policy}}{}}

\begin{proof}[Proof of \cref{thm:ucb_off_policy}]
    We use the same regret decomposition in~\cref{eq:regret_decomposition}, and Term (I) of~\cref{eq:regret_decomposition} is bounded in the same way as in the proof of~\cref{thm:ucb_on_policy}. Hence, it suffices to bound Terms (II) and (III) of~\cref{eq:regret_decomposition}.

    \paragraph{Bounding Terms (II) and (III)} By~\cref{lem:ucb_off_policy_concentration}, the event $\mathcal{E}_V$ holds with probability at least $1 - \delta$. On this event, as in the proof of~\cref{thm:ucb_on_policy}, the optimism guarantee (RHS) in~\cref{lem:ucb_optimism_and_error_bound} gives $\text{Term (II)} \leq 0$. Therefore,
    \begin{align}
        \text{Term (II)} + \text{Term (III)} &\leq \Gamma_{\UCB} \sum_{k=1}^K \sum_{h=1}^H \mathbb{E}_{\pi_{k}} \left[ \left\| \bm{\phi}(s_h, a_h) \right\|_{\bm{\Lambda}_{k,h}^{-1}} \right] \\
        &\leq \Gamma_{\UCB} \left( H \sqrt{2 d_\phi K \log \left( 1 + \frac{K}{d_\phi} \right)} + H \sqrt{2 K \log(1 / \delta)} \right) \\
        &= \widetilde{\mathcal{O}} \left( H^2 \sqrt{(d_\phi^3 + d_\phi d_x) K} \right).
    \end{align}
    The first inequality follows from $\text{Term (II)} \leq 0$ and the error bound (LHS) in~\cref{lem:ucb_optimism_and_error_bound}. The second inequality follows from~\cref{lem:off_policy_bonus}, which holds for any critic with probability at least $1 - \delta$. The equality uses $\Gamma_{\UCB} = 2 \alpha = 6 C^{\UCB, \off}_\delta H \sqrt{d_\phi} = \widetilde{\mathcal{O}}(H \sqrt{d_\phi^2 + d_x})$ from~\cref{lem:ucb_off_policy_concentration}.

    \paragraph{Putting everything together} Since $\SubOpt(\overline{\pi}_K) = \mathrm{Reg}_K / K$, combining the above bounds with~\cref{eq:regret_decomposition} gives that, with probability at least $1 - 2 \delta$,
    \begin{align}
        \MoveEqLeft
        \SubOpt(\overline{\pi}_K) = \frac{1}{K} \left( \text{Term (I)} + \text{Term (II)} + \text{Term (III)} \right) \\
        &\leq \widetilde{\mathcal{O}} \left( H^2 \sqrt{\frac{d_\phi^3 + d_\phi d_x}{K}} + \varepsilon_{\proj} \right).
    \end{align}
    Replacing $\delta$ by $\delta / 2$ proves the stated bound. Finally, with $K = (\log \lvert \mathcal{A} \rvert + d_\phi^3 + d_\phi d_x) H^4 \epsilon^{-2}$, the bound becomes $\widetilde{\mathcal{O}}(\epsilon + \varepsilon_{\proj})$. This concludes the proof.
\end{proof}

\section{Instantiating the Actor with SPMA}
\label{sec:spma_actor}

\cref{lem:projection_error} can not only be applied to $\NPG$ but also to other mirror descent-based policy optimization methods such as $\texttt{MDPO}$~\citep{tomar2020mirror} and $\SPMA$~\citep{asad2025fast}. In this section, as an example, we show that the projected variant of $\SPMA$ (projected $\SPMA$) is also compatible with our framework and enjoys similar guarantees as $\PNPG$.

\subsection{Algorithm Design}

For every $(h, s) \in [H] \times \mathcal{S}$, projected $\SPMA$ replaces the $\NPG$ step in~\cref{eq:pnpg} by the $\SPMA$ step:
\begin{align}
    \pi_{k+1/2,h}(a \mid s) = \pi_{k,h}(a \mid s) ( 1 + \eta ( \widehat{Q}_{k,h}(s, a) - \langle \pi_{k,h}(\cdot \mid s), \widehat{Q}_{k,h}(s, \cdot) \rangle ) ), \qquad \pi_{k+1} = \Proj_\Pi (\pi_{k+1/2}). \label{eq:spma}
\end{align}
For $\eta \leq 1/(2H)$ and $\widehat{Q}_{k,h}(s, \cdot) \in [0, H]^{\lvert \mathcal{A} \rvert}$, $\pi_{k+1/2,h}(\cdot \mid s)$ is a valid distribution. If $\pi_{k} = \pi(\bm{\theta}_{k})$ is log-linear, then $\pi_{k+1/2,h}(\cdot \mid s)$ is the softmax policy with logits
\begin{align}
    z^{\SPMA}_{k,h}(s, a) \coloneq \bm{x}(s, a)^\top \bm{\theta}_{k,h} + \log \left( 1 + \eta \left( \widehat{Q}_{k,h}(s, a) - \left\langle \pi_{k,h}(\cdot \mid s), \widehat{Q}_{k,h}(s, \cdot) \right\rangle \right) \right). \label{eq:target_logits_spma}
\end{align}
We can therefore instantiate the actor in~\cref{alg:optimistic_actor_critic} with projected $\SPMA$ by replacing $z_{k,h}$ with $z^{\SPMA}_{k,h}$ in the actor losses in~\cref{eq:global_actor_loss,eq:actor_loss}. The resulting actor is presented in~\cref{alg:actor_spma}. Compared to~\cref{alg:actor_npg}, it only changes the target logits in the coreset actor loss, and hence it enjoys the same memory complexity.

\begin{algorithm2e}[!ht]
\caption{Actor: Projected $\SPMA$}\label{alg:actor_spma}
\begin{algorithmic}[1]
\State \textbf{Input}: estimated $Q$-function $\widehat{Q}_{k}$, actor parameters $\bm{\theta}_{k}$, policy optimization learning rate $\eta$, number of actor updates $T_{k}$, actor learning rate $\eta^{\act}$, coreset $\mathcal{C}$ and distribution $\rho_{\mathcal{C}}$
\For{$h=1, 2, \ldots, H$}
    \State Define the coreset actor loss $\widetilde{\ell}^{\act}_{k,h} (\bm{\theta})$ using~\cref{eq:actor_loss} with $z_{k,h}$ replaced by $z^{\SPMA}_{k,h}$ in~\cref{eq:target_logits_spma}
    \State $\bm{\theta}_{k,h}^{(0)} \gets \bm{\theta}_{k,h}$
    \For{$t=1, \ldots, T_{k}$}
        \State $\bm{\theta}_{k,h}^{(t)} \gets \bm{\theta}_{k,h}^{(t-1)} - \eta^{\act} \nabla_{\bm{\theta}} \widetilde{\ell}^{\act}_{k,h}(\bm{\theta}_{k,h}^{(t-1)})$
    \EndFor
    \State $\bm{\theta}_{k+1,h} \gets \bm{\theta}_{k,h}^{(T_k)}$
\EndFor
\State \textbf{Return}: actor parameters $\bm{\theta}_{k+1} = \{\bm{\theta}_{k+1,h}\}_{h \in [H]}$
\end{algorithmic}
\end{algorithm2e}

\subsection{Projected SPMA Regret}

Since the proof of~\cref{lem:projection_error} only uses the fact that $\pi_{k+1/2,h}(\cdot \mid s)$ is the softmax policy with the target logits, \cref{lem:projection_error} holds verbatim for projected $\SPMA$ when~\cref{asm:bias,asm:opt_error} are stated with $z^{\SPMA}_{k,h}$. The following lemma is the counterpart of~\cref{lem:generalized_omd_regret} for projected $\SPMA$, stated in the same generic form.
\begin{lemma}[Projected SPMA Regret]
\label{lem:generalized_omd_regret_spma}
Let $\{g_k\}_{k=1}^K$ be a sequence of vectors in $[0, L]^n$. Let $p_1 \in \Delta_n$ be the uniform distribution, and for every $k \in [K]$, let
\begin{align}
    p_{k+1/2}(i) &= p_k(i) \left( 1 + \eta \left( g_k(i) - \left\langle p_k, g_k \right\rangle \right) \right) \quad \text{for all } i \in [n], \label{eq:standard_spma} \\
    p_{k+1} &= \Proj_k (p_{k+1/2}),
\end{align}
where $\eta \leq 1/(2L)$ is the step size, and $\Proj_k \colon \Delta_n \to \Delta_n$ is a projection operator whose outputs have full support. Then, with the projection error $\varepsilon_k(u)$ defined in~\cref{lem:generalized_omd_regret}, for every $u \in \Delta_n$, it holds that
\begin{align}
    \sum_{k=1}^K \left\langle u - p_{k}, g_{k} \right\rangle \leq \frac{\log n + \sum_{k=1}^K \varepsilon_{k}(u)}{\eta} + \frac{17 \eta L^2 K}{8}.
\end{align}
\end{lemma}

\begin{proof}[Proof of \cref{lem:generalized_omd_regret_spma}]
We consider a fixed $u \in \Delta_n$, and write $\varepsilon_k \coloneq \varepsilon_k(u)$ for simplicity. Let $c_k \coloneq \left\langle p_k, g_k \right\rangle \in [0, L]$ and $d_{k} \coloneq \log \left( 1 + \eta \left( g_{k} - c_k \bm{1} \right) \right)$, where $\bm{1} \in \mathbb{R}^n$ is the all-one vector. Then, the update in~\cref{eq:standard_spma} can be written as $p_{k+1/2} \propto p_k \exp(d_k)$, where the normalization is automatic since $\left\langle p_k, g_k - c_k \bm{1} \right\rangle = 0$. Equivalently,
\begin{align}
    p_{k+1/2} = \operatorname*{arg min}_{p \in \Delta_n} \left\{ \left\langle p, - d_k \right\rangle + \KL(p \| p_{k}) \right\}.
\end{align}
For all $i \in [n]$, since $g_{k}(i) - c_k \in [-L, L]$ and $\eta \leq \frac{1}{2 L}$, we have $\eta ( g_{k}(i) - c_k ) \geq -\frac{1}{2}$, and therefore
\begin{align}
    d_{k}(i) &\leq \eta \left( g_{k}(i) - c_k \right) \leq \eta L, \\
    d_{k}(i) &\geq \eta \left( g_{k}(i) - c_k \right) - \eta^2 \left( g_{k}(i) - c_k \right)^2 \geq \eta \left( g_{k}(i) - c_k \right) - \eta^2 L^2,
\end{align}
where the first inequality in the first line follows from $\log(1+y) \leq y$ for all $y > -1$, and the first inequality in the second line holds because $\log(1+y) \geq y - y^2$ for all $y \geq -\frac{1}{2}$. In particular, $\left\| d_k \right\|_\infty \leq \eta L + \eta^2 L^2 \leq \frac{3}{2} \eta L$.

Following the same steps as in the first display of the proof of~\cref{lem:generalized_omd_regret} with $\eta g_k$ replaced by $d_k$ (using the optimality condition of the above problem instead), we obtain
\begin{align}
    \left\langle u - p_{k}, d_{k} \right\rangle \leq \KL (u \| p_{k}) - \KL (u \| p_{k+1}) + \varepsilon_{k} + \frac{1}{2} \left\| d_{k} \right\|_\infty^2 \leq \KL (u \| p_{k}) - \KL (u \| p_{k+1}) + \varepsilon_{k} + \frac{9 \eta^2 L^2}{8}.
\end{align}
Moreover, by the above bounds on $d_k(i)$, the vector $\bm{e}_k \coloneq d_k - \eta (g_k - c_k \bm{1})$ satisfies $\bm{e}_k(i) \in [-\eta^2 L^2, 0]$ for all $i \in [n]$. Since $u$ and $p_k$ are distributions, $\left\langle u - p_k, \bm{e}_k \right\rangle \geq -\eta^2 L^2$ and $\left\langle u - p_k, \bm{1} \right\rangle = 0$. Hence,
\begin{align}
    \left\langle u - p_{k}, d_{k} \right\rangle = \left\langle u - p_{k}, \eta g_{k} \right\rangle + \left\langle u - p_k, \bm{e}_k \right\rangle \geq \left\langle u - p_{k}, \eta g_{k} \right\rangle - \eta^2 L^2.
\end{align}
Combining the two displays gives
\begin{align}
    \left\langle u - p_{k}, \eta g_{k} \right\rangle \leq \KL (u \| p_{k}) - \KL (u \| p_{k+1}) + \varepsilon_{k} + \frac{17 \eta^2 L^2}{8}.
\end{align}
Summing over $k \in [K]$, telescoping, and using $\KL(u \| p_1) \leq \log n$ as in the proof of~\cref{lem:generalized_omd_regret}, and then dividing both sides by $\eta$ concludes the proof.
\end{proof}

We now instantiate~\cref{lem:generalized_omd_regret_spma} for projected $\SPMA$. Consider any $(h, s) \in [H] \times \mathcal{S}$, and let $n = \lvert \mathcal{A} \rvert$, $L = H$, $g_k = \widehat{Q}_{k,h}(s, \cdot)$, and $p_k = \pi_{k,h}(\cdot \mid s)$ for every $k$, where $g_k \in [0, H]^{\lvert \mathcal{A} \rvert}$ due to the clipping in the critic. Then, the $\SPMA$ step in~\cref{eq:spma} is exactly~\cref{eq:standard_spma}, $\Proj_k$ maps $p_{k+1/2}$ to $p_{k+1} = \pi_{k+1,h}(\cdot \mid s)$, as determined by $\Proj_\Pi$ in~\cref{eq:spma}, and the projection error $\varepsilon_k(u)$ becomes $\varepsilon^{\proj}_{k,h}(u, s)$ in~\cref{eq:projection_error}. With $\eta = \sqrt{8 \log \lvert \mathcal{A} \rvert / (17 K H^2)}$, which satisfies $\eta \leq 1 / (2H)$ whenever $K \geq 2 \log \lvert \mathcal{A} \rvert$, and with $\varepsilon_{\proj}$ defined in~\cref{eq:eps_proj} for this step size, the same argument as in~\cref{eq:npg-regret-bound} gives
\begin{align}
    \sum_{k=1}^{K} \sum_{h = 1}^H \mathbb{E}_{\pi^\star} \left[ \left\langle \pi_h^\star(\cdot \mid s_h) - \pi_{k,h}(\cdot \mid s_h), \widehat{Q}_{k,h} (s_h, \cdot) \right\rangle \right] \leq H^2 \sqrt{\frac{17 K \log \lvert \mathcal{A} \rvert}{2}} + K \varepsilon_{\proj}.
\end{align}
This bound matches~\cref{eq:npg-regret-bound} up to a constant factor. In the proofs of~\cref{thm:on_policy_sample_efficiency,thm:off_policy_sample_efficiency}, the actor only enters through~\cref{eq:npg-regret-bound}, which bounds Term (I) of~\cref{eq:regret_decomposition}, and, in the off-policy setting, through the bound on the actor parameters in~\cref{lem:bounded_parameters}. The latter also holds for projected $\SPMA$ with $\overline{R}_\theta$ replaced by $2 \overline{R}_\theta$, since $\lvert z^{\SPMA}_{k,h}(s, a) - \bm{x}(s, a)^\top \bm{\theta}_{k,h} \rvert \leq 2 \eta H$ by $\lvert \log(1 + y) \rvert \leq 2 \lvert y \rvert$ for $\lvert y \rvert \leq 1/2$, and this change only affects logarithmic factors. Therefore, by plugging the above regret bound into these proofs, projected $\SPMA$ recovers the same sample complexities as $\PNPG$ in both the on-policy setting (\cref{thm:on_policy_sample_efficiency}) and the off-policy setting (\cref{thm:off_policy_sample_efficiency}), up to constant factors.

\section{Technical Tools}
\label{sec:technical_tools}

In this section, we collect the technical tools used in the analysis.

\begin{lemma}[Three-Point Property of Bregman Divergence]
\label{lem:three_point}
    Suppose $\Omega \subseteq \mathbb{R}^d$ is closed and convex. Consider a strictly convex function $\Phi: \Omega \to \mathbb{R}$. For all $\bm{v}_1 \in \Omega$ and $\bm{v}_2, \bm{v}_3 \in \operatorname{int} \Omega$,
    \begin{align}
        D_{\Phi} (\bm{v}_1, \bm{v}_2) + D_{\Phi} (\bm{v}_2, \bm{v}_3) - D_{\Phi} (\bm{v}_1, \bm{v}_3) = \left\langle \nabla\Phi (\bm{v}_3) - \nabla\Phi (\bm{v}_2), \bm{v}_1 - \bm{v}_2 \right\rangle.
    \end{align}
\end{lemma}

\begin{lemma}[{\citealt[Lemma D.1]{jin2020provably}}]
\label{lem:elliptical_potential}
    Let $\bm{\Lambda} = \lambda I + \sum_{i=1}^n \bm{\phi}_i\bm{\phi}_i^\top$, where $\bm{\phi}_i \in \mathbb{R}^d$ and $\lambda > 0$. Then,
    \begin{align}
        \sum_{i=1}^n \bm{\phi}_i^\top \bm{\Lambda}^{-1} \bm{\phi}_i \leq d.
    \end{align}
\end{lemma}

\begin{lemma}[\citealt{abramowitz1948handbook}]
    \label{lem:gaussian_abramowitz}
    Suppose $X$ is a Gaussian random variable $X \sim \mathcal{N}(\mu,\sigma^2)$, where $\sigma > 0$. For $z \in [0, 1]$, it holds that
    \begin{align}
        \Pr(X > \mu + z \sigma) \geq \frac{e^{-z^2/2}}{\sqrt{8\pi}} \quad {\text{and}} \quad \Pr(X < \mu - z \sigma) \geq \frac{e^{-z^2/2}}{\sqrt{8\pi}}.
    \end{align}
\end{lemma}

\begin{lemma}[{\citealt[Lemma 1]{laurent2000adaptive}}]
    \label{lem:gaussian_norm_concentration}
    Let $\bm{y} \sim \mathcal{N}(\bm{0}, I)$ be a standard Gaussian vector in $\mathbb{R}^d$. For any $\delta \in (0, 1)$, it holds that
    \begin{align}
        \Pr \left( \left\| \bm{y} \right\|_2 \geq \sqrt{d} + \sqrt{2 \log(1 / \delta)} \right) \leq \delta.
    \end{align}
\end{lemma}
\begin{proof}
    \citet[Lemma 1]{laurent2000adaptive} shows that $\Pr(\| \bm{y} \|_2^2 \geq d + 2 \sqrt{d x} + 2 x) \leq e^{-x}$ for any $x > 0$. Since $d + 2 \sqrt{d x} + 2 x \leq d + 2 \sqrt{2 d x} + 2 x = (\sqrt{d} + \sqrt{2 x})^2$, setting $x = \log(1/\delta)$ concludes the proof.
\end{proof}

\begin{lemma}[{\citealt[Corollary 4.2.13]{vershynin2018high}}]
    \label{lem:sphere_net}
    For any $\varepsilon > 0$, there exists a subset $\mathcal{U} \subseteq \mathbb{S}^{d - 1}$ of the unit sphere in $\mathbb{R}^d$ with $\lvert \mathcal{U} \rvert \leq (1 + 2 / \varepsilon)^d$ such that for every $\bm{u} \in \mathbb{S}^{d-1}$, there exists $\bm{u}^\prime \in \mathcal{U}$ with $\| \bm{u} - \bm{u}^\prime \|_2 \leq \varepsilon$.
\end{lemma}

\begin{lemma} \label{lem:l1_norm_and_matrix_norm}
    Let $\bm{A} \in \mathbb{R}^{d\times d}$ be a positive definite matrix where its largest eigenvalue $\lambda_{\max}(\bm{A}) \leq \lambda$. Given that $\bm{v}_1, \ldots, \bm{v}_n$ are $n$ vectors in $\mathbb{R}^d$, it holds that
    \begin{align}
        \left\| \bm{A} \sum_{i=1}^n \bm{v}_i \right\| \leq \sqrt{\lambda n \sum_{i=1}^n \|\bm{v}_i\|_{\bm{A}}^2}.
    \end{align}
\end{lemma}

\begin{lemma}
\label{lem:fixed_point}
    Let $a, b \geq 0$ with $a + b \geq 1$, and let $C^\star \coloneq 4 \sqrt{a + b \log(4 (a + b))}$. Then, $C^\star \geq 4$ and $C^\star \geq 2 \sqrt{a + b \log C^\star}$.
\end{lemma}
\begin{proof}
    We first show that $C^\star \geq 4$. Since $4(a + b) \geq 4$, we have $\log(4 (a + b)) \geq \log 4 \geq 1$, and hence $a + b \log(4(a+b)) \geq a + b \geq 1$. This implies $C^\star \geq 4$.

    We then show the second statement. Using $1 + \log y \leq y$ for all $y > 0$ with $y = 4 (a + b)$, we have that
    \begin{align}
        a + b \log(4(a+b)) \leq (a + b) \left( 1 + \log(4(a+b)) \right) \leq 4 (a + b)^2.
    \end{align}
    Therefore,
    \begin{align}
        \log C^\star = \log 4 + \frac{1}{2} \log \left( a + b \log(4(a+b)) \right) \leq \log 4 + \frac{1}{2} \log \left( 4 (a+b)^2 \right) = \log(8 (a + b)) \leq 2 \log(4 (a+b)),
    \end{align}
    where the last inequality uses $8 (a + b) \leq 16 (a + b)^2$, which holds since $a + b \geq 1/2$. Consequently,
    \begin{align}
        4 \left( a + b \log C^\star \right) \leq 4 a + 8 b \log(4(a+b)) \leq 16 a + 16 b \log(4(a+b)) = \left( C^\star \right)^2.
    \end{align}
    Taking the square root of both sides concludes the proof.
\end{proof}

\begin{lemma}[Value Difference{~\citep{kakade2002approximately}}]
\label{lem:value_diff}
Given any $\pi, \pi^\prime \in \Delta(\mathcal{A} \mid \mathcal{S}, H)$, it holds that
    \begin{align}
        V_1^{\pi^\prime}(s_1) - V_1^\pi(s_1) = \sum_{h=1}^H \mathbb{E}_{\pi} \left[ \left\langle \pi^\prime_h(\cdot \mid s_h) - \pi_h(\cdot \mid s_h), Q_h^{\pi^\prime}(s_h,\cdot) \right\rangle \right].
    \end{align}
\end{lemma}

The following lemma extends~\cref{lem:value_diff} to an arbitrary $Q$-function $\widehat{Q}$, at the cost of an additional Bellman error term. When $\widehat{Q} = Q^{\pi^\prime}$, this term vanishes by the Bellman equation, and the lemma recovers~\cref{lem:value_diff}.
\begin{lemma}[Extended Value Difference]
\label{lem:extended_value_diff}
Given any $\pi, \pi^\prime \in \Delta(\mathcal{A} \mid \mathcal{S}, H)$ and any $Q$-function $\widehat{Q} = \{\widehat{Q}_h\}_{h \in [H]}$ with $\widehat{Q}_h: \mathcal{S} \times \mathcal{A} \to \mathbb{R}$, we define $\widehat{V}_h (s) = \mathbb{E}_{a \sim \pi^\prime_h(\cdot \mid s)} \widehat{Q}_h (s, a)$ for any $h \in [H]$ and $\widehat{V}_{H+1} = 0$. Then,
    \begin{align}
        \MoveEqLeft
        \widehat{V}_1 (s_1) - V_1^\pi(s_1) \\
        &= \sum_{h=1}^H \mathbb{E}_{\pi} \left[ \left\langle \pi^\prime_h(\cdot \mid s_h) - \pi_h(\cdot \mid s_h), \widehat{Q}_h(s_h,\cdot) \right\rangle \right] + \sum_{h=1}^H \mathbb{E}_{\pi} \left[ \widehat{Q}_h(s_h,a_h) - r_h(s_h, a_h) - \mathbb{P}_h \widehat{V}_{h+1}(s_h, a_h) \right].
    \end{align}
\end{lemma}

\begin{lemma}[Azuma–Hoeffding Inequality{~\citep{azuma1967weighted}}]
    \label{lem:azuma}
    Let $\{X_i\}_{i=1}^n$ be a martingale difference sequence with respect to a filtration $\{\mathcal{G}_i\}_{i=0}^n$, i.e., $X_i$ is $\mathcal{G}_i$-measurable and $\mathbb{E}[X_i \mid \mathcal{G}_{i-1}] = 0$, such that $\lvert X_i \rvert \leq c$ almost surely. Then, for any $\delta \in (0, 1)$, with probability at least $1 - \delta$, it holds that $\sum_{i=1}^n X_i \leq c \sqrt{2 n \log(1 / \delta)}$.
\end{lemma}

\begin{lemma}[Determinant-Trace Inequality{~\citep[Lemma 10]{abbasi2011improved}}] \label{lem:determinant_trace_inequality}
    Suppose $\bm{v}_1, \bm{v}_2, \ldots, \bm{v}_n \in \mathbb{R}^d$ and $\left\| \bm{v}_i \right\|_2 \leq L$ for all $i \in [n]$. Let $\bm{\Lambda}_n = \lambda I + \sum_{i=1}^n \bm{v}_i \bm{v}_i^\top$ for some $\lambda > 0$. Then, for all $n$, it holds that
    \begin{align}
        \det(\bm{\Lambda}_n) \leq (\lambda + n L^2 / d)^d.
    \end{align}
\end{lemma}

\begin{lemma}[{\citealt[Lemma 11]{abbasi2011improved}}] \label{lem:sum_induced_norm}
    Suppose $\bm{v}_1, \bm{v}_2, \ldots, \bm{v}_n \in \mathbb{R}^d$ and $\left\| \bm{v}_i \right\|_2 \leq L$ for all $i \in [n]$. Let $\bm{\Lambda}_n = \bm{\Lambda}_0 + \sum_{i=1}^n \bm{v}_i \bm{v}_i^\top$ and $\lambda_{\min}(\bm{\Lambda}_0) \geq \max \left\{ 1, L^2 \right\}$. Then, for all $n$, it holds that
    \begin{align}
        \log \left( \frac{\det(\bm{\Lambda}_n)}{\det(\bm{\Lambda}_0)} \right) \leq \sum_{i=1}^n \left\| \bm{v}_i \right\|^2_{\bm{\Lambda}_{i-1}^{-1}} \leq 2 \log \left( \frac{\det(\bm{\Lambda}_n)}{\det(\bm{\Lambda}_0)} \right).
    \end{align}
\end{lemma}

\begin{lemma}[{\citealp[Lemma B.1]{zhong2023theoretical}}]
    \label{lem:covering_number_v_q_pi}
    Consider the value function class $\mathcal{V} = \{ \left\langle Q(\cdot, \cdot), \widehat{\pi} (\cdot \mid \cdot) \right\rangle_{\mathcal{A}} \mid Q \in \mathcal{Q}, \widehat{\pi} \in \Pi\}$, where every $Q \in \mathcal{Q}$ takes values in $[0, H]$. Then, it holds that
    \begin{align}
        N_{\Delta}(\mathcal{V}) \leq N_{\Delta / 2}(\mathcal{Q}) \cdot N_{\Delta / (2 H)}(\Pi),
    \end{align}
    where the covering numbers of $\mathcal{V}$, $\mathcal{Q}$, and $\Pi$ are with respect to $\sup_{s} \lvert V(s) - V^\prime(s) \rvert$, $\sup_{(s,a)} \lvert Q(s,a) - Q^\prime(s,a) \rvert$, and $\sup_{s} \left\| \widehat{\pi}(\cdot \mid s) - \widehat{\pi}^\prime(\cdot \mid s) \right\|_1$, respectively.
\end{lemma}

\begin{lemma}[Value-Aware Uniform Concentration{~\citep[Lemma D.4]{jin2020provably}}]
    \label{lem:d4_jin2020provably}
    Let $\{s_n\}_{n=1}^\infty$ be a stochastic process on the state space $\mathcal{S}$ with the corresponding filtration $\{\mathcal{F}_n\}_{n=0}^\infty$ such that $s_n$ is $\mathcal{F}_{n}$-measurable. Let $\{\bm{\phi}_n\}_{n=1}^\infty$ be an $\mathbb{R}^d$-valued stochastic process such that $\bm{\phi}_n$ is $\mathcal{F}_{n-1}$-measurable, and $\left\| \bm{\phi}_n \right\| \leq 1$. Let $\bm{\Lambda}_n = \lambda I + \sum_{i=1}^n \bm{\phi}_i \bm{\phi}_i^\top$ for some $\lambda > 0$. Assume $\mathcal{V}$ is a value function class such that $\sup_{s \in \mathcal{S}} |V(s)| \leq H$. Then, for any $\delta > 0$, with probability at least $1 - \delta$, for all $n \geq 0$ and any $V \in \mathcal{V}$, it holds that
    \begin{align}
        \left\| \sum_{i=1}^n \bm{\phi}_i \left\{ V(s_i) - \mathbb{E} \left[ V(s_i) \mid \mathcal{F}_{i-1} \right]\right\} \right\|^2_{\bm{\Lambda}_n^{-1}} \leq 4 H^2 \left[ \frac{d}{2} \log \left( \frac{n + \lambda}{\lambda} \right) + \log \left( \frac{N_\Delta}{\delta} \right) \right] + \frac{8 n^2 \Delta^2}{\lambda},
    \end{align}
    where $N_\Delta$ represents the $\Delta$-covering number of $\mathcal{V}$ with the distance measured by $\text{dist}(V, V^\prime) = \sup_{s \in \mathcal{S}} |V(s) - V^\prime(s)|$.
\end{lemma}

\begin{lemma}[Covering Number of Euclidean Ball]
    \label{lem:covering_number_euclidean_ball}
For any $\Delta > 0$, the $\Delta$-covering number, $N_\Delta$, of the Euclidean ball of radius $R > 0$ in $\mathbb{R}^d$ satisfies that
\begin{align}
    N_\Delta \leq \left( 1+ \frac{2R}{\Delta} \right)^d.
\end{align}
\end{lemma}

\begin{lemma}[Lipschitz Continuity of Softmax]
    \label{lem:softmax_policy_distance}
    For any logit vectors $\bm{v}, \bm{v}^\prime \in \mathbb{R}^{\lvert \mathcal{A} \rvert}$, let $\pi, \pi^\prime \in \Delta(\mathcal{A})$ be the softmax distributions $\pi(\cdot) \propto \exp(\bm{v}(\cdot))$ and $\pi^\prime(\cdot) \propto \exp(\bm{v}^\prime(\cdot))$. Then, it holds that
    \begin{align}
        \left\| \pi - \pi^\prime \right\|_1 \leq 2 \left\| \bm{v} - \bm{v}^\prime \right\|_\infty.
    \end{align}
\end{lemma}
\begin{proof}
    For $t \in [0, 1]$, let $\pi_t$ be the softmax distribution with logits $\bm{v}^\prime + t \bm{\delta}$, where $\bm{\delta} \coloneq \bm{v} - \bm{v}^\prime$. Differentiating the softmax gives $\frac{\mathrm{d} \pi_t(a)}{\mathrm{d} t} = \pi_t(a) \left( \bm{\delta}(a) - \sum_{b \in \mathcal{A}} \pi_t(b) \bm{\delta}(b) \right)$, and hence $\sum_{a \in \mathcal{A}} \left\lvert \frac{\mathrm{d} \pi_t(a)}{\mathrm{d} t} \right\rvert \leq 2 \left\| \bm{\delta} \right\|_\infty$. Integrating over $t \in [0, 1]$ gives $\left\| \pi - \pi^\prime \right\|_1 = \left\| \pi_1 - \pi_0 \right\|_1 \leq 2 \left\| \bm{\delta} \right\|_\infty$.
\end{proof}

\clearpage
\addtocontents{toc}{\protect\setcounter{tocdepth}{2}}
\section{Experiments}
\label{sec:experiments}

In this section, we evaluate the performance of our proposed algorithm with other methods on various benchmarks. In~\cref{subsec:exp_linear_mdp}, we test our proposed algorithm in two specific environments of linear MDPs. In~\cref{subsec:ablation_studies}, we further conduct some ablation studies in the Deep Sea environment. In~\cref{subsec:atari_experiments}, we test our proposed algorithm in large-scale deep RL applications (Atari~\citep{mnih2013playing}) and compare its performance to two commonly used deep RL algorithms, \texttt{PPO}~\citep{schulman2017proximal} in the on-policy setting and \texttt{DQN}~\citep{mnih2015human} in the off-policy setting.

\subsection{Experiments in Linear MDPs}
\label{subsec:exp_linear_mdp}

First, we test our proposed algorithm in the randomly generated linear MDPs (Random MDP) and the linear MDP version of the Deep Sea~\citep{osband2019behaviour}.

\subsubsection{Environment Setup}
\label{subsubsec:environment_setup}

Our experimental setup is an extension of~\citet{ishfaq2024provable}. In particular, we test our proposed algorithm in the off-policy setting on the linear MDP version of the Deep Sea~\citep{osband2019behaviour} and the Random MDP. In both experiments, we use the linear MDP features as the policy features (i.e., $\bm{\phi} = \bm{x}$), and we set $d \coloneq d_\phi = d_x$ to represent the feature dimension for both the actor and the critic parameters.

\begin{figure}[!ht]
    \centering
    \includegraphics[width=0.33\linewidth]{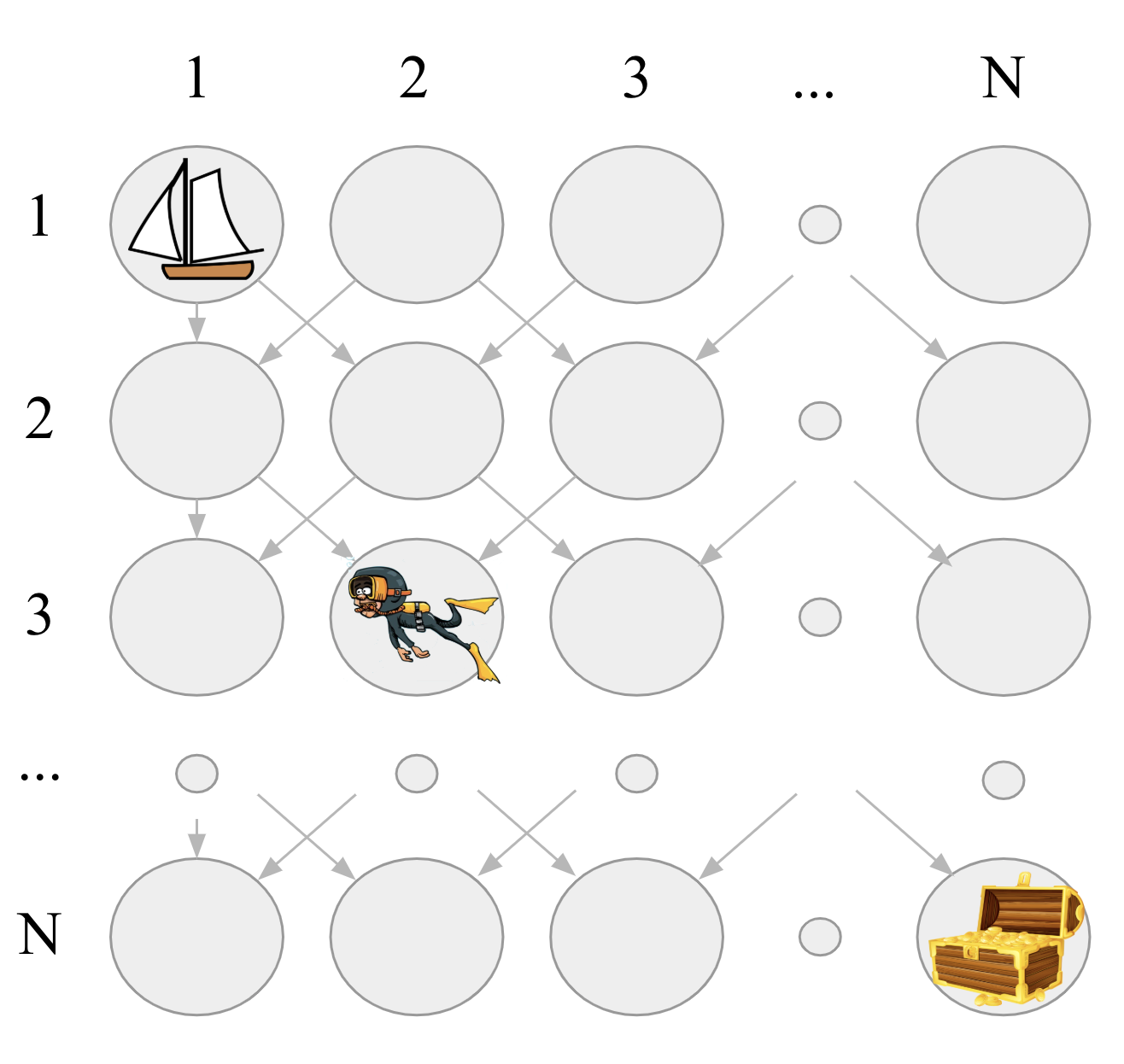}
    \caption{Example of the Deep Sea environment from~\citet{osband2019behaviour}.}
    \label{fig:placeholder}
\end{figure}

For the Random MDP environment, we consider $15$ states and $5$ actions. For each state $s \in \mathcal{S}$, we generate $\bm{\psi}_h(s) \in \mathbb{R}^d$ uniformly at random in $[0, 1]$ and construct tile coded features. The agent always starts from state $0$, receiving a small reward of $0.1$ upon taking action $0$, and obtains the maximum reward when reaching the final state and taking action $1$. All other state-action pairs yield zero reward. Given this reward function and a randomly generated transition kernel, we solve for $\bm{\psi}_h$ and $\bm{\upsilon}_h$ via minimizing least square errors following~\cref{def:linear_mdp} and use the corresponding $\mathbb{P}_h$ and $r_h$ to set up the environment.

For the Deep Sea environment, we use a $N\times N$ grid with $N=10$ where the agent always starts at $(1, 1)$ and can move either bottom-right or bottom-left, receiving rewards of $-0.01/N$ and $0$, respectively. Reaching the bottom-right corner yields a reward of $1$. Furthermore, we generate the actor and critic features by projecting each state-action pair uniformly between $[0, d-1]$. Given the true transition probabilities and rewards, we solve for $\bm{\psi}_h$ and $\bm{\upsilon}_h$ via minimizing least square errors following~\cref{def:linear_mdp} and use the corresponding $\mathbb{P}_h$ and $r_h$ to set up the environment. In both environments, we normalize the resulting $\mathbb{P}_h$ so that it is a valid transition kernel. Since the normalization may perturb the fitted kernel, the resulting environments are only approximately linear in the features. As in the original benchmark, the rewards in Deep Sea are not restricted to $[0, 1]$.

\subsubsection{Coreset Construction}
To implement our proposed algorithm, we need to conduct the experimental design to obtain $\mathcal{C}$ and $\rho_{\mathcal{C}}$. For this, we use the greedy heuristic of~\citet{jain2022towards} in~\cref{alg:g_experimental_design}, which is inspired by the G-optimal design. Starting from $\widetilde{\bm{G}} = \lambda_{\mathcal{C}} I$, in each iteration, it goes through all the state-action pairs outside the coreset and computes the marginal gain $g(s, a) = \| \bm{x}(s, a) \|_{\widetilde{\bm{G}}^{-1}}$. If the largest marginal gain exceeds a threshold $\varepsilon_G$, it adds the corresponding state-action pair to the coreset and updates $\widetilde{\bm{G}}^{-1}$ using the Sherman--Morrison formula. Otherwise, it terminates. Following~\citet{jain2022towards}, we take $\rho_{\mathcal{C}}$ to be the uniform distribution over $\mathcal{C}$. Since $\widetilde{\bm{G}}$ is regularized and unweighted, it differs from the Gram matrix $\bm{G}$ in~\cref{lem:projection_error}, and hence this stopping rule does not certify a bound on $\chi_G$. Therefore, \cref{alg:g_experimental_design} is only a practical way to construct the coreset. A more theoretically principled way is to solve the G-optimal design problem with the Frank--Wolfe algorithm (\cref{subsec:kiefer_wolfowitz}), which returns a coreset with $\chi_G \leq \sqrt{2 d_x}$.

\begin{algorithm2e}[!ht]
\caption{Greedy Coreset Construction}\label{alg:g_experimental_design}
\begin{algorithmic}[1]
\State \textbf{Input}: features $\bm{x}: \mathcal{S} \times \mathcal{A} \mapsto \mathbb{R}^{d_x}$, threshold $\varepsilon_G > 0$, regularization $\lambda_{\mathcal{C}} = 1$ (default)
\State \textbf{Initialize}: $\widetilde{\bm{G}}^{-1} = \frac{1}{\lambda_{\mathcal{C}}} I_{d_x \times d_x}$, $\mathcal{C} = \emptyset$, $g_{\max} = \infty$
\While{$g_{\max} > \varepsilon_G$}
    \State $g_{\max} = 0$
    \For{$(s, a) \in (\mathcal{S} \times \mathcal{A}) \setminus \mathcal{C}$}
        \State $g(s, a) = \sqrt{\bm{x}(s, a)^\top \widetilde{\bm{G}}^{-1} \bm{x}(s, a)}$ \Comment{Compute the marginal gain}
        \If{$g_{\max} < g(s, a)$}
            \State $(s^\star, a^\star) = (s, a)$
            \State $g_{\max} = g(s, a)$
        \EndIf
    \EndFor
    \If{$g_{\max} > \varepsilon_G$}
        \State $\mathcal{C} = \mathcal{C} \cup \{ (s^\star, a^\star) \}$
        \State $\widetilde{\bm{G}}^{-1} = \widetilde{\bm{G}}^{-1} - \frac{\widetilde{\bm{G}}^{-1} \bm{x}(s^\star, a^\star) \bm{x}(s^\star, a^\star)^\top \widetilde{\bm{G}}^{-1}}{1 + \bm{x}(s^\star, a^\star)^\top \widetilde{\bm{G}}^{-1} \bm{x}(s^\star, a^\star)}$ \Comment{Sherman--Morrison update}
    \EndIf
\EndWhile
\State \textbf{Return}: coreset $\mathcal{C}$
\end{algorithmic}
\end{algorithm2e}

\subsubsection{Algorithms and Hyperparameters}
We denote by \texttt{LMC-NPG-EXP} our proposed algorithm with an explicit log-linear policy parameterization that uses $\LMC$ for policy evaluation and $\PNPG$ for policy optimization over the obtained coreset. We denote by \texttt{LMC-NPG-IMP} an idealized variant of $\NPG$ that does not have an explicit policy parameterization and maintains an implicit policy by storing all parameterized $Q$ functions (and hence requires significantly more memory). As a baseline, we also consider the value-based algorithm $\LMC$~\citep{ishfaq2024provable}.

In~\cref{tab:hyperparams}, we list the hyperparameters used across all experiments. For log-linear policies, the actor loss in~\cref{eq:actor_loss} admits a closed-form solution, allowing us to avoid tuning of the actor learning rate $\eta^{\act}$ and the number of actor updates $T_{k}$, by minimizing the objective exactly. In general, for non-linear models, inexact optimization (e.g., stochastic gradient descent) is usually required to optimize the actor loss.

\begin{table}[ht]
\centering
\caption{Hyperparameter search space for our experiments in linear MDPs.}
\label{tab:hyperparams}
\begin{tabular}{lccc}
\toprule[1.2pt]
\textbf{Hyperparameter} & \textbf{LMC} & \textbf{LMC-NPG-IMP} & \textbf{LMC-NPG-EXP} \\
\midrule
Policy optimization learning rate ($\eta$) & \xmark & $[0.1, 1, 10, 100]$ & $[0.1, 1, 10, 100]$ \\
Temperature ($\beta^{-1}$) & \multicolumn{3}{c}{$[10^{-2}, 10^{-3}, 10^{-4}, 10^{-5}]$} \\
Number of critic updates ($J_k$) & \multicolumn{3}{c}{$100$} \\
Number of critic samples ($M$) & \multicolumn{3}{c}{$1$} \\
Critic learning rate ($\eta^{\cri}_{k,h}$) & \multicolumn{3}{c}{$[10^{-2}, 10^{-3}, 10^{-4}, 10^{-5}]$} \\
Number of episodes ($K$) & \multicolumn{3}{c}{$600$} \\
Horizon length ($H$) & \multicolumn{3}{c}{$100$} \\
\bottomrule[1.2pt]
\end{tabular}
\end{table}

\subsubsection{Experimental Results}
Following the protocol of~\citet{ishfaq2024provable}, each algorithm is run with $20$ random seeds. We sweep the hyperparameters as shown in~\cref{tab:hyperparams} and report the best performance with $95\%$ confidence intervals.

\begin{figure}[!ht]
\centering
\begin{minipage}[t]{0.48\linewidth}
\centering
\hbadness=10000
\includegraphics[width=0.88\linewidth]{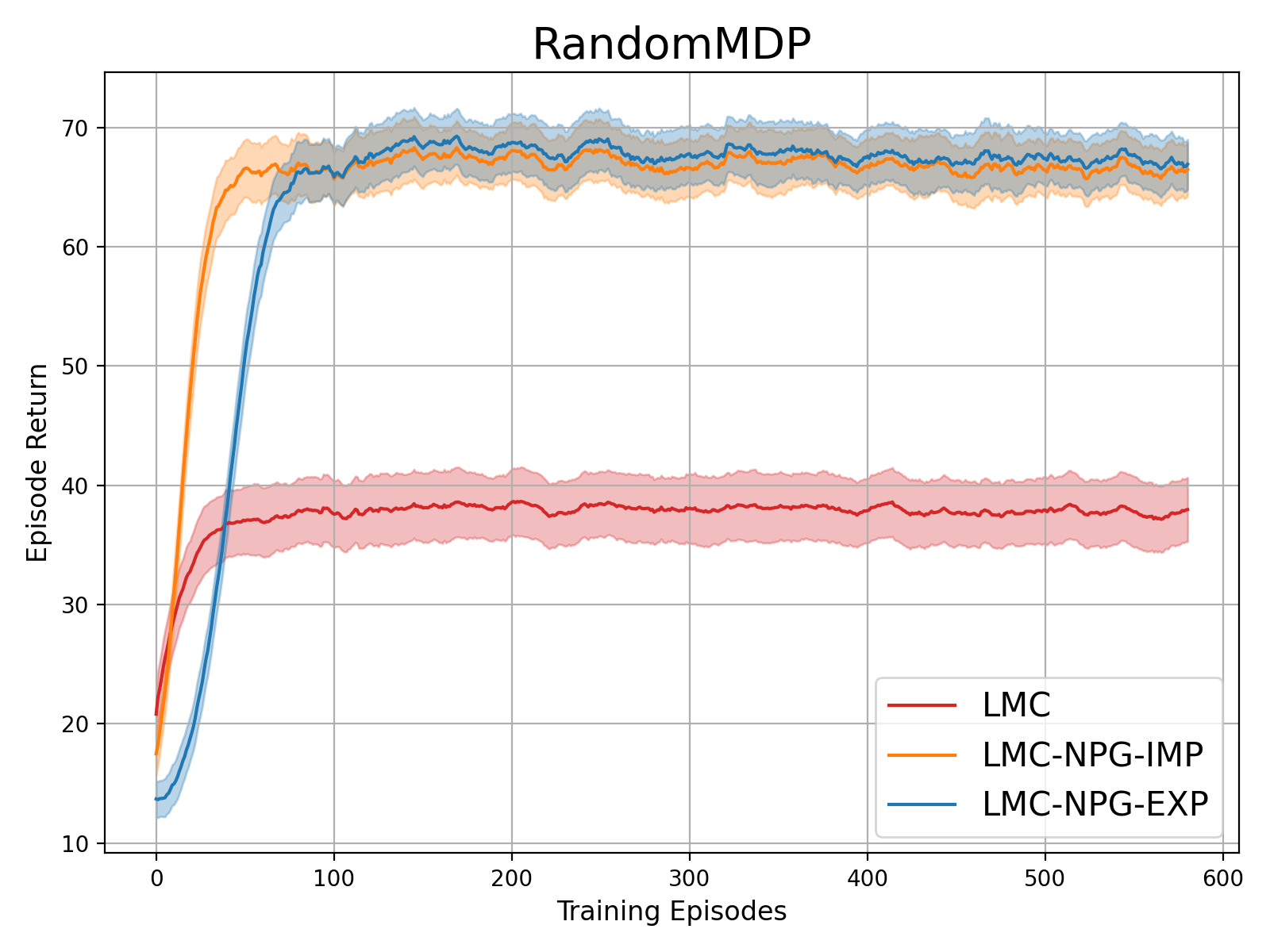}
\caption{Comparison of \texttt{LMC-NPG-EXP} (our method), \texttt{LMC-NPG-IMP} (memory-intensive variant), and $\LMC$ (value-based baseline) in the Random MDP.}
\label{fig:random_mdp}
\end{minipage}\hfill
\begin{minipage}[t]{0.48\linewidth}
\centering
\hbadness=10000
\includegraphics[width=0.88\linewidth]{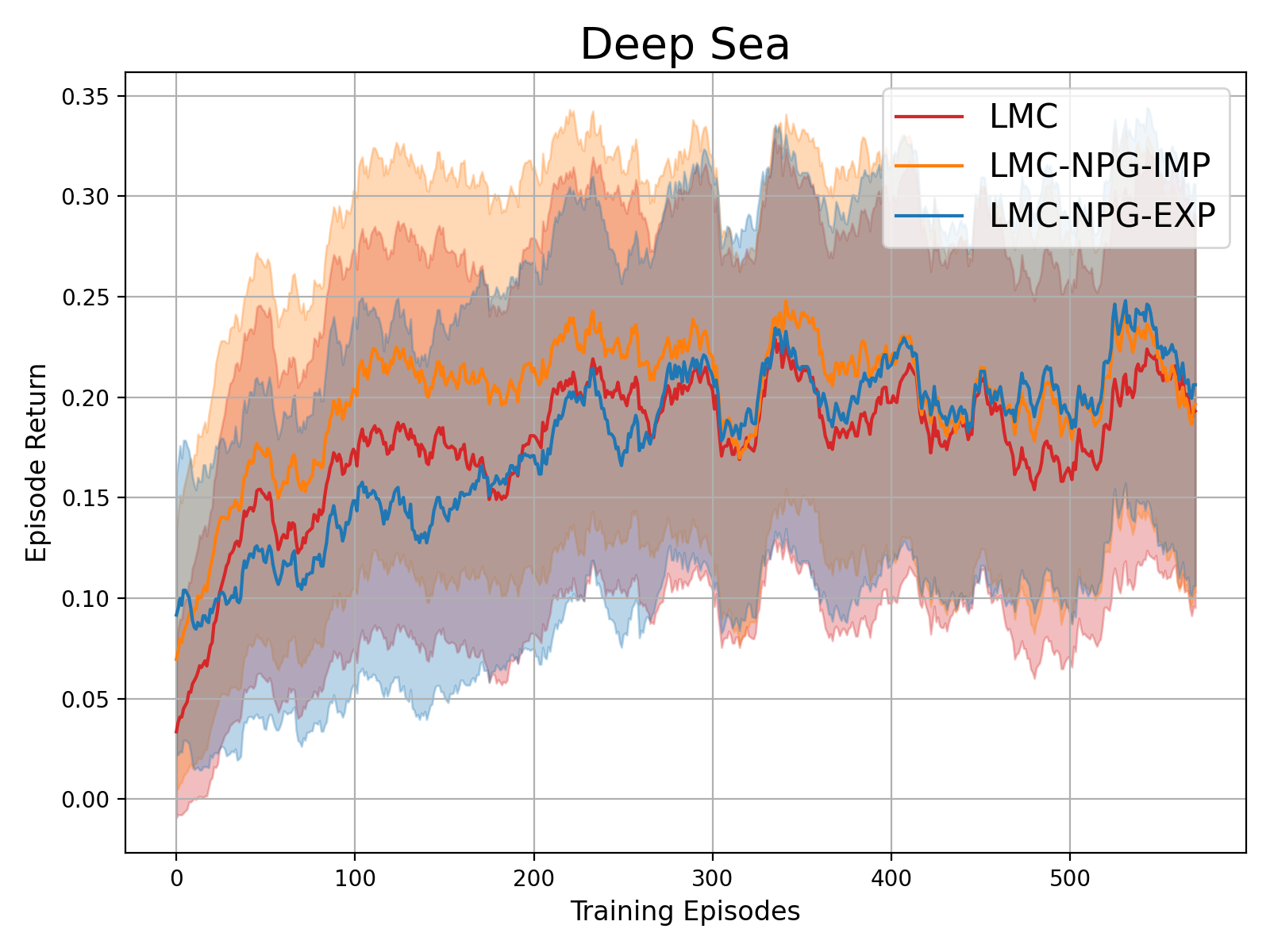}
\caption{Comparison of \texttt{LMC-NPG-EXP} (our method), \texttt{LMC-NPG-IMP} (memory-intensive variant), and $\LMC$ (value-based baseline) in the Deep Sea environment.}
\label{fig:deep_sea}
\end{minipage}
\end{figure}

For the Random MDP,~\cref{fig:random_mdp} indicates that \texttt{LMC-NPG-EXP} closely matches \texttt{LMC-NPG-IMP} while outperforming the value-based baseline, $\LMC$. For the linear MDP version of Deep Sea,~\cref{fig:deep_sea} showcases that \texttt{LMC-NPG-EXP} can achieve comparable performance with \texttt{LMC-NPG-IMP} and $\LMC$.

In the Random MDP, \texttt{LMC-NPG-IMP} improves slightly faster in the first few episodes, since it represents the policy implicitly and does not incur any projection error. \texttt{LMC-NPG-EXP} catches up within around $100$ episodes and attains the same final return, which suggests that the projection did not materially reduce the attained return in this environment. In contrast, $\LMC$ plateaus at a much lower return. In Deep Sea, the three algorithms achieve similar returns, and their confidence intervals largely overlap. Overall, these results suggest that explicitly parameterizing the policy does not sacrifice performance, while it avoids storing all the previous $Q$-functions as required by \texttt{LMC-NPG-IMP}.

\subsection{Ablation Studies}
\label{subsec:ablation_studies}

\subsubsection{Ablation on Exploration}
\label{subsubsec:ablation_exploration}
To study the impact of the exploration mechanism, $\texttt{LMC}$, in our proposed algorithm, we perform an ablation in the linear MDP variant of Deep Sea with $N = 15$. For the baseline without exploration, we consider the same algorithm design as \texttt{LMC-NPG-EXP} and simply do not inject any noise into the $\LMC$ update. The results in~\cref{fig:lmc-npg-exp-exploration-ablation} indicate that the exploration mechanism improves the average returns at the smaller tested feature dimensions.

\begin{figure*}[!ht]
\centering
\includegraphics[width=\linewidth]{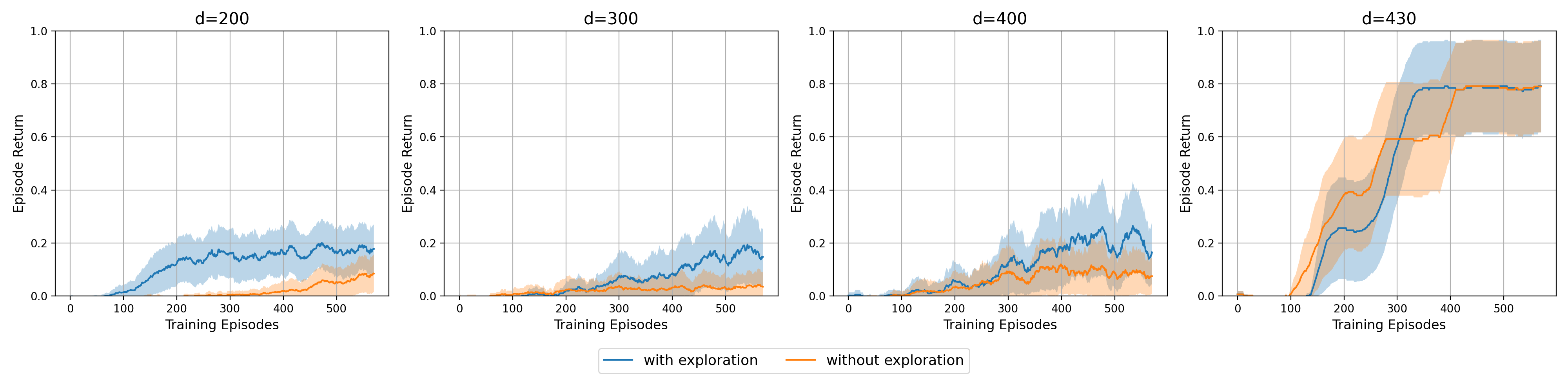}
\caption{Ablation of the exploration mechanism for \texttt{LMC-NPG-EXP}.}
\label{fig:lmc-npg-exp-exploration-ablation}
\end{figure*}

\subsubsection{Sensitivity to the Temperature}
\label{subsubsec:ablation_inverse_temp}
We study the sensitivity of our algorithm to the temperature $\beta^{-1}$ (i.e., the reciprocal of the inverse temperature $\beta$ in~\cref{alg:lmc}) in the Deep Sea environment. As shown in~\cref{fig:deep_sea_inv_temp_ablation}, several temperature choices achieve similar returns, although the convergence speed and final performance vary with the feature dimension.
\begin{figure*}[!ht]
\centering
\includegraphics[width=\linewidth]{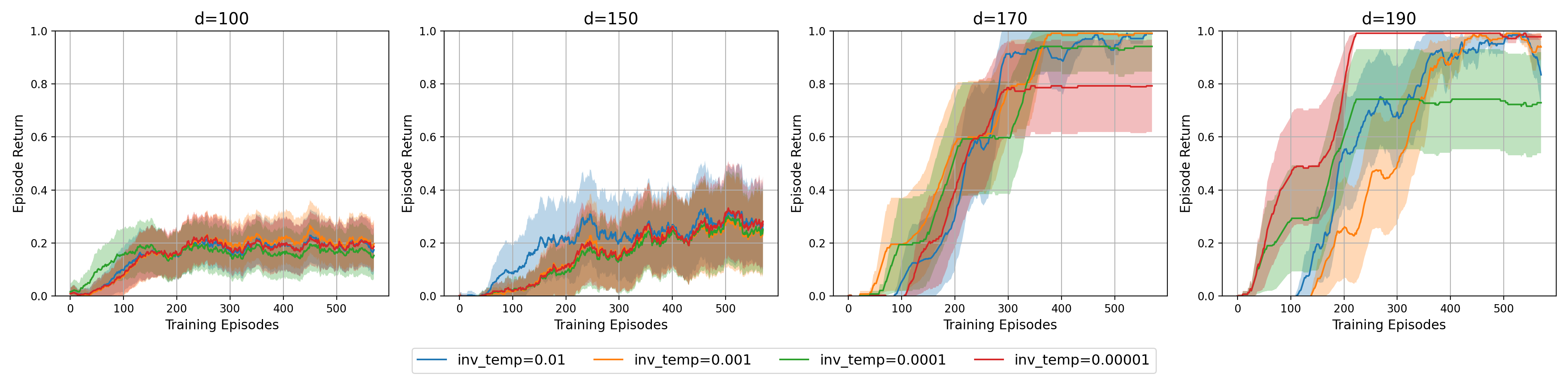}
\caption{Sensitivity of \texttt{LMC-NPG-EXP} to the temperature $\beta^{-1}$ (\texttt{inv\_temp} in the legend).}
\label{fig:deep_sea_inv_temp_ablation}
\end{figure*}
\vspace{-3ex}

\subsubsection{Sensitivity to the Number of Critic Samples}
\label{subsubsec:ablation_num_critic_samples}
We further study the effect of the number of critic samples, $M$, in the Deep Sea environment. We vary $M$ across a range of different feature dimensions. Similarly, as shown in~\cref{fig:deep_sea_M_ablation}, we find that the performance of our proposed algorithm remains relatively stable across the tested values of $M$ for each feature dimension.
\begin{figure*}[!ht]
\centering
\includegraphics[width=\linewidth]{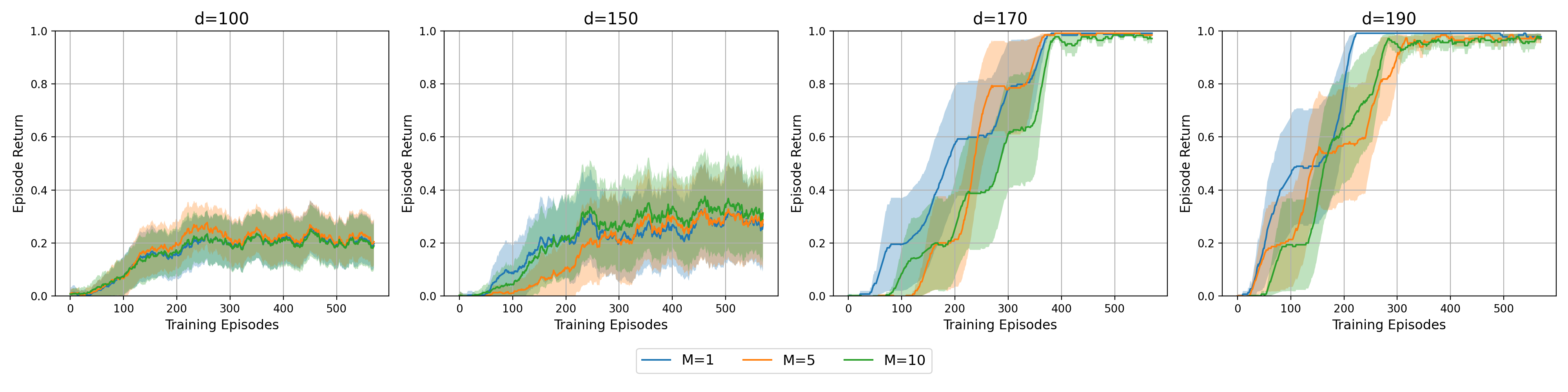}
\caption{Sensitivity of \texttt{LMC-NPG-EXP} to the number of critic samples $M$.}
\label{fig:deep_sea_M_ablation}
\end{figure*}
\subsection{Experiments Beyond Linear MDPs: Atari}
\label{subsec:atari_experiments}

\subsubsection{Extension to Deep RL Applications}

Compared to our theoretical results in the finite-horizon linear MDP setting, the Atari benchmark requires handling discounted problems with complex nonlinear function approximation. Consequently, starting from our theoretically principled algorithm, we make three practical modifications and follow the same protocol as the practical version of \texttt{LMC} in~\citet{ishfaq2024provable}. Since we are not in the finite-horizon setting, we compute the value functions and $Q$ functions in a forward fashion ($h = 1, \ldots$) rather than backwards ($h = H, \ldots, 1$), which changes how the critic loss is formed in~\cref{alg:lmc}. Moreover, instead of using stochastic gradient descent to optimize the \texttt{LMC} critic loss, we follow~\citet{ishfaq2024provable} and use the practical version of \texttt{LMC} that integrates Adam. Finally, rather than relying on the experimental design, at iteration $k$, we roll out the current policy to store state–action pairs in a buffer $\mathcal{D}_{k}$ and minimize the actor objective over (a subset of) it. Following such changes, we can extend our proposed algorithm beyond the linear function approximation and empirically test its performance against other commonly used algorithms.

\subsubsection{Environment Setups and Hyperparameters}
\label{subsec:atari_hyperparams}
In the Atari experiments, we use the following setups. In the on-policy setting, we adopt the recommended hyperparameters from~\citet{rl-zoo3} for $\texttt{PPO}$ and our algorithm. In the off-policy setting, following prior work~\citep{tomar2020mirror}, we use the default hyperparameters from stable baselines~\citep{stable-baselines3} for $\texttt{DQN}$ and our algorithm. These setups are motivated by two considerations. First, we aim to evaluate the effect of different objectives without performing an extensive hyperparameter search. Second, the CNN-based actor and the critic architectures make large grid searches over multiple hyperparameters (e.g., framestack, $\lambda$ in GAE, horizon length, and discount factor) computationally prohibitive. A complete list of hyperparameters used in the on-policy and off-policy Atari experiments is provided in~\cref{tab:atari-off-policy-hyperparams,tab:atari-on-policy-hyperparams}. Additionally, for the policy optimization learning rate $\eta$ in our algorithm, we perform a grid search over $\{0.01, 0.1, 1.0\}$. For the $\LMC$ critic, we use the Adam-based $\LMC$ update in~\citet[Algorithm 2]{ishfaq2024provable} with a single critic sample ($M = 1$), $J_k = 1$, and their default smoothing factors. Following their grid search, we sweep the bias factor over $\{0.01, 0.1, 1\}$ and the inverse temperature $\beta$ over $\{10^{12}, 10^{14}, 10^{16}\}$.

\subsubsection{Experimental Results}
As illustrated in~\cref{fig:atari_lmc_vs_ppo}, our algorithm can achieve comparable or even better performance than \texttt{PPO} in the on-policy setting. Similarly, in the off-policy setting, our algorithm's performance is comparable to or exceeds \texttt{DQN} in most considered games, as shown in~\cref{fig:atari_lmc_vs_dqn}. These results illustrate the empirical promise of the practical extension on the evaluated Atari games.

\clearpage

\begin{table}[!ht]
\centering
\caption{Hyperparameters for the Atari experiments in the on-policy setting.}
\label{tab:atari-on-policy-hyperparams}
\small
\begin{tabular}{lcc}
\toprule[1.2pt]
\textbf{Hyperparameter} & \textbf{$\texttt{LMC-NPG-EXP}$} & \textbf{$\texttt{PPO}$} \\
\midrule
Reward normalization & \xmark & \xmark \\
Observation normalization & \xmark & \xmark \\
Orthogonal weight initialization & \cmark & \cmark \\
Value function clipping & \xmark & \xmark \\
Gradient clipping & \xmark & \cmark  \\
Probability ratio clipping & \xmark & \cmark \\
Clip range & \xmark & $0.1$ \\
Entropy coefficient & $0$ & $ 0.01$\\
Number of inner loop updates & $5$ & $4$ \\
Adam step-size & $3\times10^{-4}$ & $2.5\times10^{-4}$ \\
Value function coefficient & \xmark & $0.5$ \\
\midrule
Minibatch size & \multicolumn{2}{c}{256} \\
Framestack & \multicolumn{2}{c}{4} \\
Number of environment copies & \multicolumn{2}{c}{8} \\
GAE ($\lambda$) & \multicolumn{2}{c}{0.95} \\
Rollout length & \multicolumn{2}{c}{128} \\
Discount factor & \multicolumn{2}{c}{0.99} \\
Total number of timesteps & \multicolumn{2}{c}{$10^7$} \\
Number of runs for plot averages & \multicolumn{2}{c}{5} \\
Confidence interval for plot runs & \multicolumn{2}{c}{$\sim 95\%$} \\
\bottomrule[1.2pt]
\end{tabular}
\end{table}

\begin{table}[!ht]
\centering
\caption{Hyperparameters for the Atari experiments in the off-policy setting.}
\label{tab:atari-off-policy-hyperparams}
\small
\begin{tabular}{lcc}
\toprule[1.2pt]
\textbf{Hyperparameter} & \textbf{$\texttt{LMC-NPG-EXP}$} & \textbf{$\texttt{DQN}$} \\
\midrule
Reward normalization & \xmark & \xmark \\
Observation normalization & \xmark & \xmark \\
Orthogonal weight initialization & \xmark & \xmark \\
Value function clipping & \xmark & \xmark \\
Gradient clipping & \xmark & \xmark  \\
Probability ratio clipping & \xmark & \xmark \\
Exploration & \texttt{LMC} & $\epsilon$-greedy\\
\midrule
Adam step-size & \multicolumn{2}{c}{$3\times10^{-4}$} \\
Buffer size & \multicolumn{2}{c}{$10^6$} \\
Minibatch size & \multicolumn{2}{c}{256} \\
Framestack & \multicolumn{2}{c}{4} \\
Number of environment copies & \multicolumn{2}{c}{8} \\
Discount factor & \multicolumn{2}{c}{0.99} \\
Total number of timesteps & \multicolumn{2}{c}{$10^7$} \\
Number of runs for plot averages & \multicolumn{2}{c}{5} \\
Confidence interval for plot runs & \multicolumn{2}{c}{$\sim 95\%$} \\
\bottomrule[1.2pt]
\end{tabular}
\end{table}

\begin{figure*}[!ht]
\centering
\includegraphics[width=\linewidth]{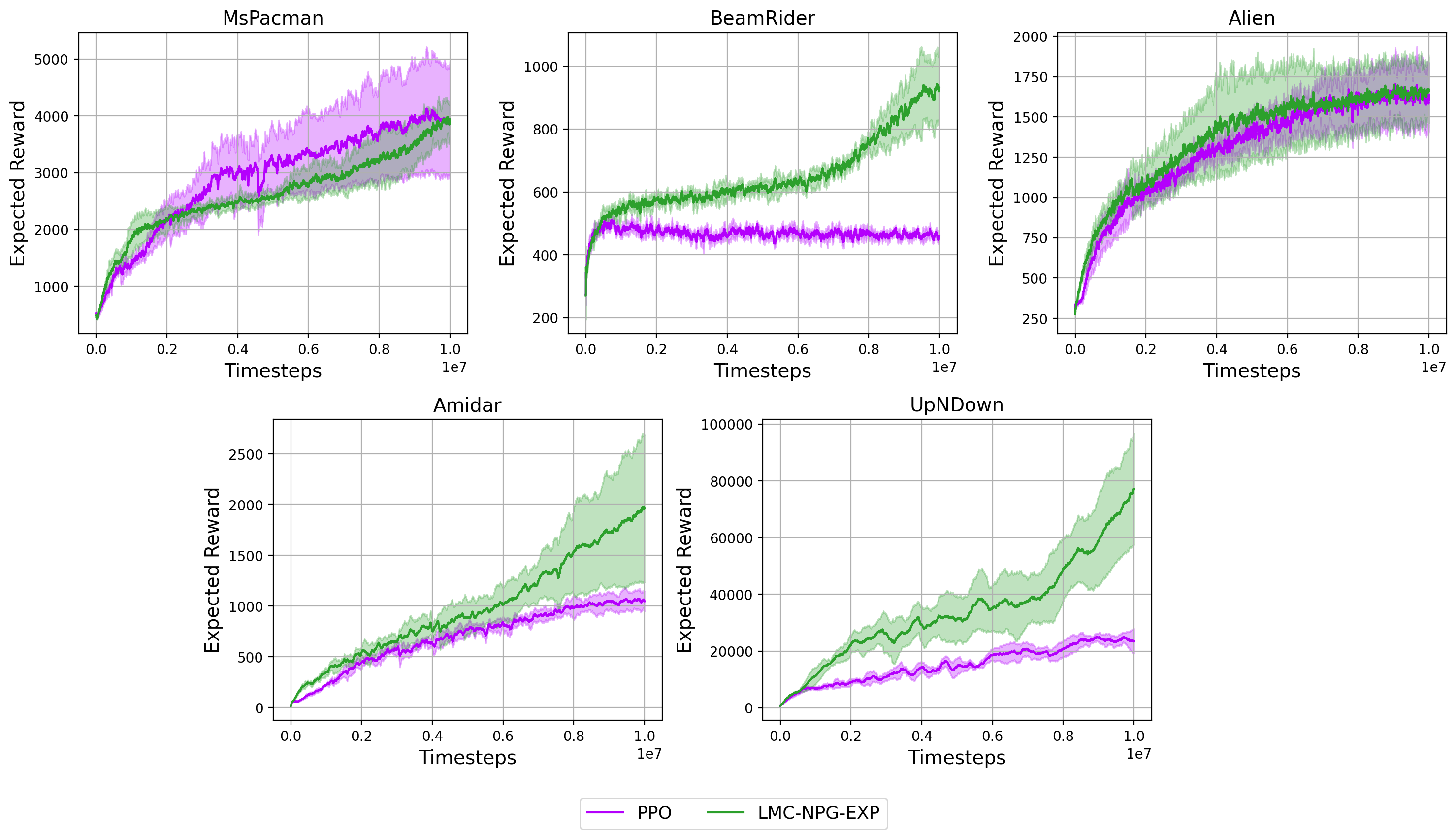}
\caption{In the on-policy setting, $\texttt{LMC-NPG-EXP}$ achieves comparable or better performance compared to \texttt{PPO}.}
\label{fig:atari_lmc_vs_ppo}
\end{figure*}

\begin{figure*}[!ht]
\centering
\includegraphics[width=\linewidth]{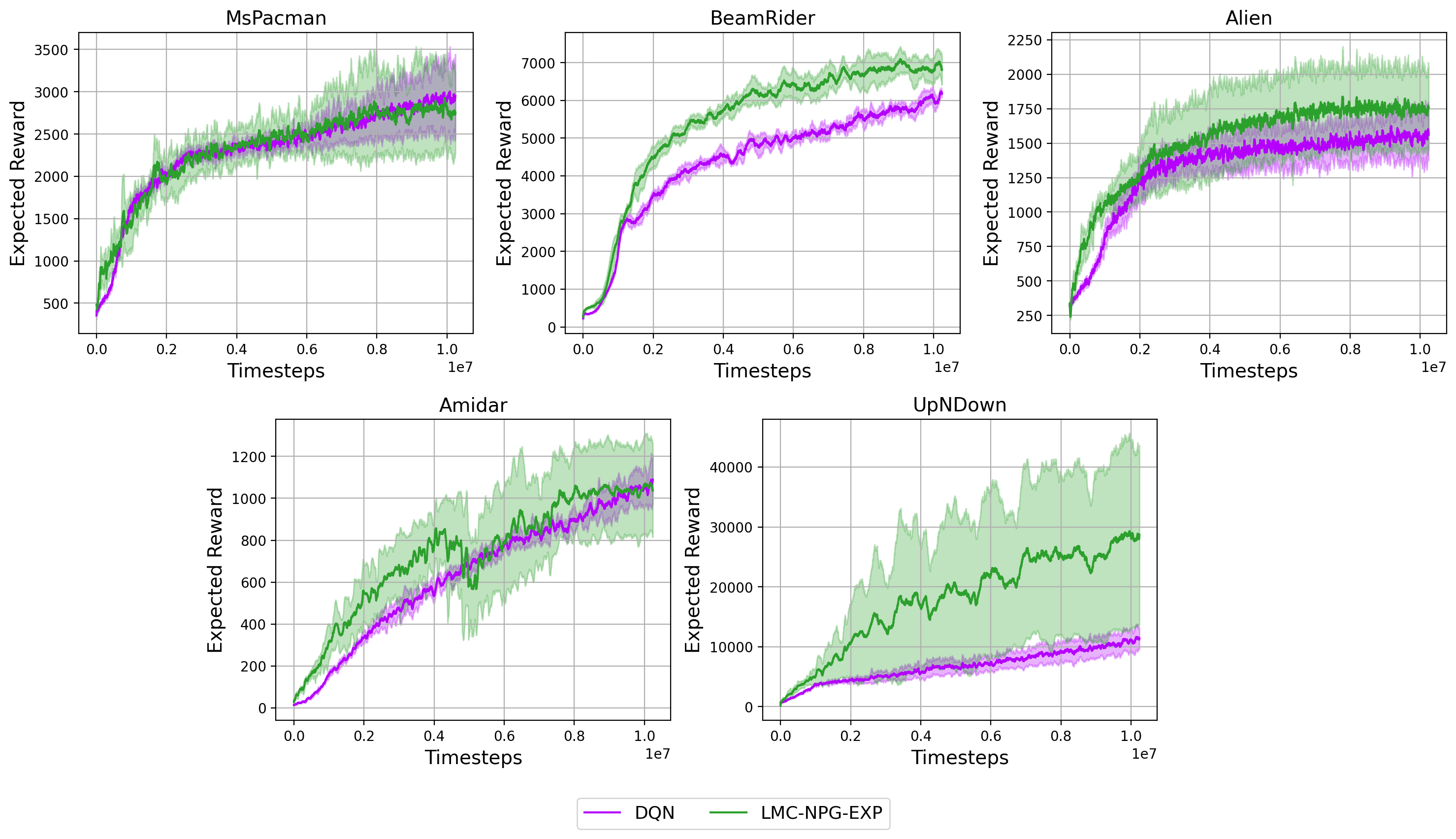}
\caption{In the off-policy setting, $\texttt{LMC-NPG-EXP}$ achieves comparable or better performance compared to \texttt{DQN}.}
\label{fig:atari_lmc_vs_dqn}
\end{figure*}

\end{document}

%% file: preamble.tex
\usepackage{silence}
\usepackage[letterpaper, left=1in, right=1in, top=1in, bottom=1in]{geometry}
\usepackage{parskip}
\usepackage[T1]{fontenc}
\usepackage{times}
\usepackage{inconsolata}
\AddToHook{ttfamily}{\fontshape{\updefault}}
\usepackage[final,nopatch=footnote]{microtype}
\usepackage{amsmath}
\usepackage{amssymb}
\usepackage{mathtools}
\usepackage{amsthm}
\usepackage{bm}
\allowdisplaybreaks

\usepackage{graphicx}
\usepackage{booktabs}
\usepackage{xcolor}
\usepackage{pifont}
\AddToHook{env/table/begin}{\setlength{\abovecaptionskip}{0pt}\setlength{\belowcaptionskip}{6pt}}

\usepackage[algo2e,ruled]{algorithm2e}
\makeatletter
\let\For\@undefined
\let\If\@undefined
\let\Else\@undefined
\let\Return\@undefined
\let\While\@undefined
\let\Repeat\@undefined
\let\Until\@undefined
\makeatother
\usepackage[noend]{algpseudocode}
\algrenewcommand{\algorithmicindent}{1.25em}
\algrenewcommand{\alglinenumber}[1]{\scriptsize\color{black}#1}
\algrenewcommand{\algorithmiccomment}[1]{\hfill{\footnotesize\color{black}$\triangleright$~\textit{#1}}}
\SetAlFnt{\small}
\SetAlCapFnt{\small}
\SetAlCapNameFnt{\small\bfseries}
\makeatletter
\AddToHook{env/algorithmic/begin}{\def\@currentcounter{ALG@line}}
\makeatother

\usepackage[round]{natbib}
\usepackage{hyperref}
\definecolor{NavyBlue}{cmyk}{0.94,0.54,0,0}
\hypersetup{
  colorlinks=true,
  allcolors=NavyBlue,
  breaklinks=true,
  hypertexnames=false}
\usepackage[nameinlink,capitalize]{cleveref}
\usepackage{autonum}
\crefname{equation}{Eq.}{Eqs.}
\Crefname{equation}{Eq.}{Eqs.}
\crefformat{equation}{#2Eq.\,(#1)#3}
\Crefformat{equation}{#2Eq.\,(#1)#3}
\crefname{assumption}{Assumption}{Assumptions}
\Crefname{assumption}{Assumption}{Assumptions}
\crefname{algocf}{Algorithm}{Algorithms}
\Crefname{algocf}{Algorithm}{Algorithms}
\crefalias{ALG@line}{line}
\AddToHook{cmd/appendix/after}{%
  \crefalias{section}{appendix}%
  \crefalias{subsection}{appendix}%
  \crefalias{subsubsection}{appendix}}

\theoremstyle{plain}
\newtheorem{theorem}{Theorem}[section]
\newtheorem{lemma}{Lemma}[section]
\newtheorem{corollary}{Corollary}[section]
\newtheorem{definition}{Definition}[section]
\newtheorem{assumption}{Assumption}[section]
\theoremstyle{definition}
\newtheorem{remark}{Remark}[section]

\renewenvironment{proof}[1][\proofname]{%
  \par\addvspace{6pt}\noindent\textit{#1.}\enspace\ignorespaces}{%
  \ifhmode\unskip\nobreak\hfill\fi\ensuremath{\square}\par\addvspace{6pt}}
\let\originalparagraph\paragraph
\renewcommand{\paragraph}[1]{\originalparagraph{#1.}}
\newcommand{\KL}{D_{\mathrm{KL}}}
\newcommand{\Proj}{\operatorname{Proj}}
\newcommand{\Clip}{\operatorname{Clip}}
\newcommand{\SubOpt}{\operatorname{Sub-Opt}}
\newcommand{\given}{\mathrel{}\middle|\mathrel{}}
\newcommand{\bias}{\mathrm{bias}}
\newcommand{\opt}{\mathrm{opt}}
\newcommand{\proj}{\mathrm{proj}}
\newcommand{\act}{\mathrm{act}}
\newcommand{\cri}{\mathrm{cri}}
\newcommand{\on}{\mathrm{on}}
\newcommand{\off}{\mathrm{off}}
\newcommand{\NPG}{\texttt{NPG}}
\newcommand{\PNPG}{\texttt{PNPG}}
\newcommand{\SPMA}{\texttt{SPMA}}
\newcommand{\LMC}{\texttt{LMC}}
\newcommand{\UCB}{\texttt{UCB}}
\newcommand{\cmark}{\text{\ding{51}}}
\newcommand{\xmark}{\text{\ding{55}}}